\documentclass[12pt]{article}
\usepackage[utf8]{inputenc}
\usepackage{times}
\usepackage{amsmath, amsfonts,amsthm,amssymb,mathrsfs, bm}
\usepackage[parfill]{parskip}
\usepackage[shortlabels]{enumitem}
\usepackage{tikz}
\usepackage{hyperref}
\usepackage{graphicx,subcaption}
\usepackage{wrapfig}
\usepackage[open=true]{bookmark}
\usepackage[letterpaper, margin=1in]{geometry}
\usepackage[square, numbers]{natbib}
\newcommand{\cD}{\mathcal{D}}
\newcommand{\cF}{\mathcal{F}}
\newcommand{\cL}{\mathcal{L}}
\newcommand{\cR}{\mathcal{R}}
\newcommand{\cN}{\mathcal{N}}
\newcommand{\cS}{\mathcal{S}}
\newcommand{\cW}{\mathcal{W}}
\newcommand{\cX}{\mathcal{X}}
\newcommand{\cY}{\mathcal{Y}}

\newcommand{\ba}{\mathbf{a}}
\newcommand{\be}{\mathbf{e}}
\newcommand{\bx}{\mathbf{x}}
\newcommand{\by}{\mathbf{y}}
\newcommand{\bz}{\mathbf{z}}
\newcommand{\bu}{\mathbf{u}}
\newcommand{\bv}{\mathbf{v}}
\newcommand{\bw}{\mathbf{w}}

\newcommand{\bA}{\mathbf{A}}
\newcommand{\bW}{\mathbf{W}}
\newcommand{\bI}{\mathbf{I}}

\newcommand{\bP}{\mathbf{P}}
\newcommand{\bS}{\mathbf{S}}
\newcommand{\bK}{\mathbf{K}}
\newcommand{\bU}{\mathbf{U}}
\newcommand{\bV}{\mathbf{V}}
\newcommand{\bQ}{\mathbf{Q}}
\newcommand{\bO}{\mathbf{O}}

\newcommand{\bSigma}{\mathbf{\Sigma}}
\newcommand{\bDelta}{\mathbf{\Delta}}
\newcommand{\bPhi}{\mathbf{\Phi}}
\newcommand{\bPsi}{\mathbf{\Psi}}
\newcommand{\bPi}{\mathbf{\Pi}}
\newcommand{\bXi}{\mathbf{\Xi}}
\newcommand{\bLambda}{\mathbf{\Lambda}}
\newcommand{\blambda}{\mathbf{\lambda}}
\newcommand{\btheta}{\bm{\theta}}
\newcommand{\bbeta}{\mathbf{\beta}}
\newcommand{\bZero}{\mathbf{0}}

\newcommand{\RR}{\mathbb{R}}
\newcommand{\ZZ}{\mathbb{Z}}
\newcommand{\EE}{\mathbb{E}}

\newcommand{\tvec}[1]{\text{vec}(#1)}

\newtheorem{lemma}{Lemma}
\newtheorem{corollary}{Corollary}
\newtheorem{theorem}{Theorem}
\newtheorem{definition}{Definition}
\newtheorem{assumption}{Assumption}
\graphicspath{{Figures/}}

\title{Identifying good directions to escape the NTK regime and efficiently learn low-degree plus sparse polynomials}
\author{Eshaan Nichani\thanks{Princeton University.}\\
\and Yu Bai\thanks{Salesforce Research.}\\
\and Jason D. Lee\footnotemark[1]}

 \ifdefined\usebigfont

 \usepackage{times}
 \usepackage[fontsize=13pt]{scrextend}
 \AtBeginDocument{\newgeometry{letterpaper,left=1.56in,right=1.56in,top=1.71in,bottom=1.77in}}

 \else
 \fi

\makeatletter
\def\thm@space@setup{%
  \thm@preskip=\parskip \thm@postskip=0pt
}
\makeatother
\begin{document}

\maketitle

\begin{abstract}
A recent goal in the theory of deep learning is to identify how neural networks can escape the “lazy training,” or Neural Tangent Kernel (NTK) regime, where the network is coupled with its first order Taylor expansion at initialization. While the NTK is minimax optimal for learning dense polynomials \citep{montanari2021}, it cannot learn features, and hence has poor sample complexity for learning many classes of functions including sparse polynomials. Recent works have thus aimed to identify settings where gradient based algorithms provably generalize better than the NTK. One such example is the “QuadNTK” approach of \citet{bai2020}, which analyzes the second-order term in the Taylor expansion. \citet{bai2020} show that the second-order term can learn sparse polynomials efficiently; however, it sacrifices the ability to learn general dense polynomials.

In this paper, we analyze how gradient descent on a two-layer neural network can escape the NTK regime by utilizing a spectral characterization of the NTK \citep{montanari2020} and building on the QuadNTK approach. We first expand upon the spectral analysis to identify “good” directions in parameter space in which we can move without harming generalization. Next, we show that a wide two-layer neural network can jointly use the NTK and QuadNTK to fit target functions consisting of a dense low-degree term and a sparse high-degree term -- something neither the NTK nor the QuadNTK can do on their own. Finally, we construct a regularizer which encourages the parameter vector to move in the “good" directions, and show that gradient descent on the regularized loss will converge to a global minimizer, which also has low test error. This yields an end to end convergence and generalization guarantee with provable sample complexity improvement over both the NTK and QuadNTK on their own.
\end{abstract}

\section{Introduction}
In recent years, deep learning has acheived a number of practical successes, in domains spanning computer vision, natural language processing, reinforcement learning, and the sciences. Despite these impressive empirical results, the theory underlying deep learning is far from complete. In fact, the dual questions of optimization -- the mechanism by which neural networks trained with gradient descent are able to interpolate training data despite the nonconvexity of the loss landscape -- and generalization -- why these solutions found by gradient descent require relatively few samples to generalize -- are still not well understood.

One successful approach for understanding optimization has been the Neural Tangent Kernel (NTK) theory~\citep{soltanolkotabi2018, jacot2018, chizat2018, du2019}. The NTK approach couples the gradient descent dynamics of a wide neural network under a specific initialization to the gradient descent dynamics of a particular kernel regression problem, with a random, initialization dependent kernel. In the limit of infinite width, this kernel converges almost surely to a deterministic kernel, also referred to as the NTK, and properties of this kernel and its corresponding Reproducing Kernel Hilbert Space can be studied.

However, recent work has shown that the NTK theory fails to explain the generalization capabilities of neural networks. While the equivalence between neural networks in the NTK regime and kernel methods implies that such models perform no better than kernels, in practice, neural networks have been shown to outperform kernel methods on a number of tasks~\citep{arora2019CNTK, leefinite2020}. Theoretically, a recent line of work~\citep{montanari2021, montanari2020, mei2022} has provided a precise statistical analysis of the generalization properties of rotationally invariant kernels on the unit sphere, which includes the NTK. \citep{montanari2021} proves a sample complexity lower bound for such kernels, showing that $d^k$ samples are needed to learn any degree $k$ polynomial in $d$ dimensions. As a result, the NTK is no better than a polynomial kernel, and cannot adapt to low-dimensional structure.

The limitations of the NTK can be further understood from the linearization perspective. Consider a two-layer neural network $f(\bx; \bW)$ with input $\bx$, width $m$, first layer weights initialized as $\bW_0 \in \RR^{d \times m}$, second layer weights $\ba \in \RR^{m}$, activation function $\sigma$, and displacement from initialization $\bW \in \RR^{d \times m}$:
\begin{equation}\label{eq: two layer nn}
	f(\bx; \bW) = \frac{1}{\sqrt{m}}\sum_{r=1}^m a_r\sigma(\bw_{0, r}^T\bx + \bw_r^T\bx).
\end{equation}
For simplicity we assume the second layer weights are held fixed, and so $\bW$ are the trainable parameters. The NTK theory states that when $\bW$ has small norm, the gradient descent dynamics can be well approximated by replacing the model with its first-order Taylor expansion about the initialization:
\begin{equation}
	f(\bx; \bW) \approx f(\bx; \bZero) + \frac{1}{\sqrt{m}}\sum_{r=1}^m a_r \sigma'(\bw_{0, r}^T\bx)\bx^T\bw_r = f(\bx; \bZero) + \tvec{\bW}^T\varphi(\bx).
\end{equation}
Here, $\{\bw_r\}_{r \in [m]}$ and $\{\bw_{0, r}\}_{r \in [m]}$ are the columns of $\bW$ and $\bW_0$ respectively, $\varphi(\bx) := \tvec{\nabla_\bW f(\bx; \bW)|_{\bW = \bZero}}$ is a random feature vector with norm independent of $m$, and 
\begin{equation}
f_L(\bx; \bW) := \tvec{\bW}^T\varphi(\bx)
\end{equation}
is hereafter referred to as the \emph{linear term}. Ignoring $d$ dependence, there exists a global minimizer with $\|\bW\|_F \simeq 1$ and $\|\bw_r\|_2 \simeq m^{-1/2}$, and thus due to local convexity gradient descent will stay in this small norm ball around the initialization while interpolating the training data. This small movement of each individual neuron gives rise to the name \emph{lazy training} for networks in the NTK regime~\citep{chizat2018}. The equivalence to kernel methods and the poor generalization of neural networks in NTK regime, along with their failure to describe the dynamics of neural networks in practice, motivate our goal to understand how neural networks can escape the NTK regime. We concretely ask the following question:

\begin{center}
\textbf{Q: How can we encourage each neuron to move $\gg m^{-1/2}$, thus escaping the NTK regime? And does this allow us to break the NTK sample complexity lower bounds?}
\end{center}

\subsection{Motivation}

\paragraph{Escaping the NTK Regime.} To answer this, we invoke the statistical characterization of the NTK developed in~\citep{montanari2021, montanari2020, mei2022} to understand the mechansism by which it overfits to the training data. For $d^k \ll n \ll d^{k+1}$, consider a dataset of $n$ training samples $(\bx_i, y_i)$, where the $\bx_i$ are sampled i.i.d from $\cS^{d-1}(\sqrt{d})$ (the $d$-dimensional sphere of radius $\sqrt{d}$) and $y_i = f^*(\bx)$ for an unknown function $f^*$. \citet{montanari2020} decompose the SVD of the empirical feature matrix $\bPhi \in \RR^{n \times md}$ into the block matrix form
\begin{align}
\bPhi = \begin{bmatrix}\varphi(\bx_1)^T \\ \cdots \\ \varphi(\bx_n)^T \end{bmatrix}= \begin{bmatrix} \bU_1 & \bU_2 \end{bmatrix} \begin{bmatrix} \tilde\bLambda_1 & \bZero \\ \bZero & \tilde\bLambda_2 \end{bmatrix} \begin{bmatrix} \bV_1^T \\ \bV_2^T \end{bmatrix},
\end{align}
where $\bU_1, \tilde\bLambda_1, \bV_1$ are the top $r$ singular values/vectors and $\bU_2, \tilde\bLambda_2, \bV_2$ are the bottom $n-r$. Here, $r = O(d^k)$ is chosen specifically so that $\bV_1$ is the ``high-variance" subspace which can express polynomials of degree $\le k$ and generalizes well~\citep{bartlett2021}, while $\bV_2$ is used to interpolate the training data while not affecting generalization.

\citep{montanari2021, montanari2020} show that the NTK will learn the projection of $f^*$ onto degree $k$ polynomials. Furthermore, \citep{bartlett2021, ghosh2022} show that gradient descent will first move in the subspace spanned by $\bV_1$ to learn this projection, then move in the $\bV_2$ directions to interpolate the training data, while not affecting the test predictions. This first stage is still desirable for our goal -- if the network can fit and generalize from part of the signal using the NTK, then it should, and previous work~\citep{hu2020} has shown that for general networks gradient descent will learn the optimal degree-1 polynomial in the early stages of training. The $\bV_2$ directions, however, are ``bad" directions the parameters should avoid moving in, as they are only used for the NTK to overfit. Instead, $\bW$ should move in the null space of $\bPhi$, where $\|\bW\|_F$ can be $\Omega(1)$ while keeping the evaluation of $f_L(\cdot; \bW)$ on the training data bounded by $O(1)$. This heuristic argument yields our first criterion for escaping the NTK regime.

\begin{center}
\textbf{Goal \#1: Move minimally in the $\bV_2$ directions.}
\end{center}

\paragraph{Generalizing to Test Data.} While the previous criterion prevents the network from overfitting with the NTK, we must also prevent a movement of $\|\bW\|_F \gg 1$ from causing the test predictions to explode. Since $\|\varphi(\bx)\|_2 = \Theta_m(1)$, it is a priori possible for $f_L(\bx; \bW) = \tvec{\bW}^T\varphi(\bx)$ to be $\gg 1$ on the population, which would necessarily cause a large test loss.

To identify a set of good directions, define the \emph{feature covariance matrix} $\bSigma \in \RR^{md \times md}$ by
\begin{equation}
\bSigma = \EE_{\bx \sim \mathcal{S}^{d-1}(\sqrt{d})}\left[\varphi(\bx)\varphi(\bx)^T\right].
\end{equation}

Our first technical contribution is a characterization of the eigendecomposition of $\bSigma$. Let the top $r$ eigenvectors of $\bSigma$ be $\bQ_1$, the next $s$ eigenvectors be $\bQ_2$, and the bottom $md - r - s$ eigenvectors be $\bQ_3$, where $\bLambda_1, \bLambda_2, \bLambda_3$ are the corresponding diagonal matrices of eigenvalues, and $r = \Theta(d^k), s = d^{\Theta(k)}$. We use the results of~\citep{montanari2020, mei2022} to show that $\bSigma$ has an eigenvalue gap in that $\lambda_{min}(\bLambda_1) \gg \lambda_{max}(\bLambda_2)$ and $\lambda_{min}(\bLambda_2) \gg \lambda_{max}(\bLambda_3)$, and furthermore that $\bQ_1$ can fit arbitrary degree $\le k$ polynomials. This partitioning of the eigenvectors tells us that $\bQ_1$ are informative, large eigenvalue directions which help the NTK to learn a low degree signal, $\bQ_2$ are the medium directions which will cause the test predictions to grow too large if $\bW$ moves too far from initialization, and $\bQ_3$ are the ``good" directions which $\bW$ is free to move a distance of $\gg 1$ in. This yields the following criterion for generalizing well:

\begin{center}
\textbf{Goal \#2: Move in the $\bQ_3$ directions, but minimally in the $\bQ_2$ directions.}
\end{center}

One challenge is that we cannot distinguish the $\bQ_2$ directions from the $\bQ_3$ directions with $d^k$ samples. Nevertheless, the existence of an eigenvalue gap will allow us to constrain movement in $\bQ_2$.

\paragraph{The Quadratic NTK.} Once we have moved $\gg 1$ from the initialization, we can no longer couple to the network's linearization. The network is still, however, in a local regime, and we instead can couple the training dynamics to the second-order Taylor expansion of our model, where the second-order term is denoted the \emph{Quadratic-NTK}~\citep{bai2020} $f_Q(\bx; \bW)$:
\begin{align}
f(\bx; \bW) &\approx f(\bx; \bZero) + \tvec{\bW}^T\varphi(\bx) + \frac{1}{2\sqrt{m}}\sum_{r=1}^m\sigma''(\bw_{0, r}^T\bx)(\bx^T\bw_r)^2\\
&= f(\bx; \bZero) + f_L(\bx; \bW) + f_Q(\bx; \bW).
\end{align}
\citet{bai2020} showed that $f_Q(\bx; \bW)$ can effectively learn low-rank polynomials with better sample complexity than the NTK. In particular, they show that $d^k$ samples are needed to fit a target function of the form $f^*(\bx) = (\beta^T\bx)^{k+1}$, an improvement over the $d^{k+1}$ samples needed by the NTK. In doing so, however, the QuadNTK sacrifices its ability to learn general dense polynomials; furthermore, \citet{bai2020} require a randomization trick to artificially delete the $f_L$ term. A later followup work~\citep{bai2020taylorized} showed that in a number of standard experimental settings, the second-order Taylor expansion of the network better tracks the true gradient descent dynamics and acheives lower test loss than the network's linearization (i.e NTK) does. However, there is no existing result which shows that both the linear term and the quadratic term can provably learn a component of the signal.

Based on the preceeding motivation, we thus aim to show that we can jointly utilize both the NTK and the QuadNTK to learn a larger class of functions than either the NTK or QuadNTK can learn on their own.

\subsection{Our Contributions}

With the previous intuition in hand, we outline the main contributions of our work. We first prove a technical result on the eigendecomposition of $\bSigma$, and show that the eigenvectors can indeed be partitioned into 3 categories corresponding to large (degree $\le k$), medium (``bad'') and small (``good'') eigenvalues. We then construct a regularizer, depending only on the covariate distribution and initialization, that enforces goals 1 and 2 by preventing the parameters from moving in either of the bad sets of directions ($\bV_2$ and $\bQ_2)$. Furthermore, we show how to jointly use the NTK and QuadNTK to fit a target signal $f^*$ consisting of an arbitrary degree $\le k$ component and sparse degree $k + 1$ component. The key technical challenge is to construct a solution with large enough movement so that the QuadNTK can fit the high degree term and hence improve generalization, while simultaneously preventing this large movement from interfering with the NTK training predictions (Goal 1) or greatly increasing test loss (Goal 2). Our main result, Theorem~\ref{thm: main thm}, is that gradient descent on a polynomially wide two-layer neural network converges to an approximate global minimizer of the regularized loss function, which generalizes well to the test distribution. As a result, we show $d^k$ samples are needed to learn $f^*$ up to vanishingly small test loss. This ultimately gives us the ``best of both worlds'', as we leverage both the linear and quadratic term to learn the target $f^*$ with sample complexity better than either the NTK or QuadNTK alone. Overall, our work identifies which directions weights can move further from initialization in and provably generalize better than the NTK.

The outline of our paper is as follows. In Section 2 we formally define the problem setup. In Section 3 we define our regularizers, and present Theorem~\ref{thm: main thm}. Section 4 is an outline of the proof of Theorem~\ref{thm: main thm}, which we split into four components -- expressing $f^*$ with the linear-plus-quad model, showing the optimization landscape has favorable geometry, a gradient descent convergence result, and a generalization bound. We conclude with experiments supporting our main theorem and demonstrating the relevance of the low-degree plus sparse task to standard neural networks.

\subsection{Related Work}

The NTK approach~\citep{soltanolkotabi2018, jacot2018, chizat2018, leewide2019, du2019}, which couples a neural network to its linearization at initialization, has been utilized to show global convergence of gradient descent on neural networks~\citep{du2018b, li2018, zou2018, zhu2019c}. The equivalence to kernel methods has also been used to prove generalization bounds, based on generalization bounds for kernels~\citep{arora2019, cao2019, zhu2019b}. However, neural networks have been shown to perform far better than their NTK in practice~\citep{arora2019CNTK, leefinite2020}. Further, \citep{montanari2021} shows that kernels cannot adapt to low-dimensional structure, and proves a sample complexity lower bound of $d^k$ samples needed to learn a degree $k$ polynomial in $d$ dimensions. A number of recent works~\citep{regmatters2018, du2018quad, yehudai2019, zhu2019, ghorbani2019b, zhu2020, ghorbani2020,daniely2020, woodworth2020, li2020, chen2020generalized, malach2021} have thus aimed to provide examples of learning problems where a neural network trained with a gradient-based algorithm has a provable sample complexity improvement over any kernel method.

One such approach has been to understand higher-order approximations of the training dynamics~\citep{bai2020, bai2020taylorized, chen2020, huang2020}. Here, the network is no longer coupled to its linearization, but rather higher order terms in the Taylor expansion. On the empirical side,~\citep{bai2020taylorized} shows that these higher order Taylor expansions better track the optimization dynamics and can obtain lower test loss. Theoretically, \citep{bai2020, chen2020} prove that the second-order term, the QuadNTK, can be used to obtain sample complexity improvements. However, the QuadNTK has poor sample complexity for learning dense polynomials, and \citep{bai2020, chen2020} do not consider training on the original network, but rather only the second-order term after the linear term has been deleted. This paper, on the other hand, provides an end-to-end convergence and generalization result for training on the full two-layer neural network. We leverage the NTK to efficiently learn polynomials with both a dense and sparse component, and are thus the first work showing that both the linear and quadratic term can learn part of the signal.

The technical results in our paper rely on the statistical characterization of the NTK developed in the series of works~\citep{montanari2021, montanari2020, mei2022}. Furthermore, our optimization results rely on a line of work showing that quadratically parameterized models have nice landscape properties such as all second-order saddle points are global minima~\citep{ge2016, ge2017, soltanolkotabi2018, du2018quad}; the fact that gradient descent avoids saddle points~\citep{ge2016, lee2016, jin2017, jin2019} can then be used to show convergence.

\section{Preliminaries}
\subsection{Problem Setup}

Our problem setup is the standard supervised learning setting. Our dataset $\cD_n = \{(\bx_i, y_i)\}_{i \in [n]}$, has $n$ samples, where $(\bx_i, y_i) \in \cX \times \cY$ are sampled i.i.d from a distribution $\mu$ on $\cX \times \cY$. $\mu$ is defined so that $(\bx, y) \sim \mu$ satisfies $x \sim \text{Unif}(\cS^{d-1}(\sqrt{d}))$, the uniform distribution on the $d$-dimensional sphere of radius $\sqrt{d}$, and $y = f^*(\bx)$ for some deterministic, unknown function $f^* : \cS^{d-1}(\sqrt{d}) \rightarrow \mathbb{R}$.

We assume that $d^k \ll n \ll d^{k+1}$ for some integer $k$, and that the target $f^*$ has the following low-degree plus sparse structure:

\begin{assumption}[Low-degree plus sparse signal]
Let $f^*(\bx) = f_k(\bx) + f_{sp}(\bx)$, where
\begin{itemize}
	\item $f_k(\bx)$ is an arbitrary degree $\le k$ polynomial with $\EE_{\bx \sim \mu}[f_k(\bx)^2] = 1$.~(Low Degree)
	\item $f_{sp}(\bx) = \sum_{i=1}^R \alpha_i (\bbeta_i^T\bx)^{k+1}$ where $|\alpha_i|\le 1, \|\bbeta_i\|_2 = 1$.~(Sparse)
\end{itemize}
\end{assumption}

We aim to fit $f^*$ with $f(\bx; \bW)$, a two-layer neural network as defined in~\eqref{eq: two layer nn}. Here, $\bW = [\bw_1, \dots, \bw_r] \in \mathbb{R}^{d \times m}$ is the first layer weight's distance from initialization and is the trainable parameter. $\bW_0 = [\bw_{0, 1}, \dots, \bw_{0, r}]$ denotes the first layer weight at initialization, and and $\ba = [a_1, \dots, a_m]^T \in \mathbb{R}^m$ is the second layer weight, which is held fixed throughout training.

We consider the following \emph{symmetric initialization} of $\ba, \bW_0$, which ensures that $f(\cdot; \mathbf{0}) = 0$ identically.
\begin{align}
	a_1 = \cdots = a_{m/2} = 1 &\qquad a_{m/2 + 1} = \cdots = a_m = -1\\
	\{\bw_{0, r}\}_{r \le m/2} \sim_{i.i.d} \cS^{(d-1)}(1) &\qquad \bw_{0, m/2 + r} = \bw_{0, r}
\end{align}

$\sigma \in C^2(\RR)$ is our nonlinear activation function. We make the following assumption on $\sigma$:
\begin{assumption}\label{assume:sigma_bound}
The activation $\sigma$ satisfies $\|\sigma\|_{\infty}, \|\sigma^{\prime}\|_{\infty}, \|\sigma^{\prime\prime}\|_{\infty} < 1$.
\end{assumption}

We also require $\sigma', \sigma''$ to satisfy Assumption~\ref{assume: hermite coeffs}, a particular technical condition on their harmonic expansions. These two assumptions are satisfied by commonly used activations, such as the sigmoid with generic shift $b$: $\sigma(z) = \frac{1}{1 + \exp(b - z)}$ .

We assume the loss function $\ell: \RR \times \RR \rightarrow \RR^{\ge 0}$ satisfies $\ell(y, z) \le 1$, $\ell(y, y) = 0$, $\ell(y, z)$ convex in $z$, and $\|\frac{\partial}{\partial z}\ell\|_\infty, \|\frac{\partial^2}{\partial z^2}\ell\|_\infty, \|\frac{\partial^3}{\partial z^3}\ell\|_\infty \le 1$. The empirical loss $\hat L$ and population loss $L$ are defined as
\begin{align}
	\hat{L}(\bW) =  \EE_n\left[\ell(y, f(\bx; \bW))\right] \qquad L(\bW) = \EE_\mu\left[\ell(y, f(\bx; \bW))\right],
\end{align}

where for a function $g(\bx, y)$, $\EE_n[g(\bx, y)] := \frac{1}{n}\sum_{i=1}^n g(\bx_i, y_i)$ denotes the empirical expectation, while $\EE_\mu[g(\bx, y)]$ denotes the population expectation over $(\bx, y) \sim \mu$. 

\paragraph{Notation.} For $f \in L^2(\mathcal{S}^d(\sqrt{d}), \mu)$, define $\|f\|_{L^2} := \|f\|_{L^2(\mathcal{S}^d(\sqrt{d}), \mu)} = \left(\EE_{x \sim \mu}[(f(\bx))^2]\right)^{1/2}$. We use big $O$ notation to ignore absolute constants that do not depend on $n, d, m$, as well as polynomial dependencies on the rank $R$. We write $a_d \lesssim b_d$ if $a_d = O(b_d)$, $a_d \ll b_d$ if $\lim_{d\rightarrow \infty} a_d/b_d = 0$. We also use $\tilde O$ notation to ignore terms that depend logarithmically on $d$. We also treat $k = O(1)$. Finally, all our results hold for $d > C$, where $C$ is a universal constant. For a matrix $\bA$, we let $\|\bA\|_F$ be its Frobenius norm, $\|\bA\| = \|\bA\|_{op}$ be the operator norm, and $\|\bA\|_{2, p} := (\sum_i\|\ba_i\|_2^p)^{1/p}$ be the $2, p$ norm.

\subsection{Linear and Quadratic Expansion}
For $\bW$ small, $f(\cdot; \bW)$ can be approximated by its second order Taylor expansion about $\bW_0$:
\begin{align}
f(\bx; \bW) \approx \frac{1}{\sqrt{m}}\sum_{r=1}^m a_r\sigma'(\bw_{0, r}^T\bx)\bx^T\bw_r + \frac12a_r\sigma''(\bw_{0, r}^T\bx)(\bx^T\bw_r)^2.
\end{align}
We define $f_L(\bx; \bW), f_Q(\bx; \bW)$ to be the linear and quadratic terms of the network:
\begin{align}
f_L(\bx; \bW) = \frac{1}{\sqrt{m}}\sum_{r=1}^m a_r\sigma'(\bw_{0, r}^T\bx)\bx^T\bw_r, \qquad f_Q(\bx; \bW) = \frac{1}{\sqrt{m}}\sum_{r=1}^m\frac12a_r\sigma''(\bw_{0, r}^T\bx)(\bx^T\bw_r)^2
\end{align}

\section{Main Theorem}
Define the \emph{NTK featurization map} $\varphi: \cS^{d-1}(\sqrt{d}) \rightarrow \RR^{md}$ as
\begin{equation}
	\varphi := \text{vec}(\nabla_\bW f(\bx; \bW)|_{\bW = \bZero}).
\end{equation}
and the \emph{feature covariance matrix} $\bSigma \in \RR^{md \times md}$ as
\begin{equation}
\bSigma := \EE_{\bx \sim \mu}\left[\varphi(\bx)\varphi(\bx)^T\right].
\end{equation}
Note that $\bSigma$ depends only on the network at initialization and the input distribution, and not on the target function $f^*$. In practice, $\bSigma$ can be approximated to arbitrary precision by using a large dataset of unlabeled data, or by computing the harmonic expansion of $\sigma'$, as detailed in Appendix~\ref{app: spherical harmonics}.

Let $\bSigma$ admit the eigendecomposition $\bSigma = \sum_{i=1}^{md}\lambda_i(\bSigma)\bv_i\bv_i^T$, where the $\lambda_i(\bSigma)$ are nonnegative and nonincreasing. For $r \in [md]$, we let $\bPi_{\le r}$ be the projection operator onto $\text{span}(\bv_1, \dots, \bv_r)$, and let $\bPi_{>r} = \bI_{md} - \bPi_{\le r}$. Furthermore, define
\begin{equation}
	\bSigma_{\le r} := \sum_{i=1}^{r}\lambda_i(\bSigma)\bv_i\bv_i^T, \qquad \bSigma_{> r} = \bSigma - \bSigma_{\ge r}.
\end{equation}

We define our regularizers as follows
\begin{align}
\cR_1(\bW; r) &:= \tvec{\bW}^T\bSigma_{> r}\tvec{\bW}\\
\cR_2(\bW; r) &:= \tvec{\bW}^T\bSigma_{\le r}\tvec{\bW}\\
\cR_3(\bW; r) &:= \EE_n\left[(f_L(\bx; \bPi_{> r}\bW))^2\right]\\
\cR_4(\bW) &:= \|\bW\|_{2, 4}^8.
\end{align}
Intuitively, $\cR_3$ constrains movement in the $\bV_2$ directions to enforce Goal \#1, $\cR_1$ constrains movement in the $\bQ_2$ directions to enforce Goal \#2, and $\cR_2, \cR_4$ are weight-decay like terms necessary for generalization. Although $\cR_1$ does not know the directions $\bQ_2$, the eigenvalue gap between $\bQ_2$ and $\bQ_3$ ensures that whenever $\cR_1$ is small, movement in $\bQ_2$ must be small as well.

Given regularization parameters $\lambda = (\lambda_1, \lambda_2, \lambda_3, \lambda_4)$, define the regularized loss $L_\lambda(\bW)$ as
\begin{equation}
L_\lambda(\bW) = \hat{L}(\bW) + \lambda_1\cR_1(\bW; r) + \lambda_2\cR_2(\bW; r) + \lambda_3\cR_3(\bW; r) + \lambda_4\cR_4(\bW).
\end{equation}

Finally, we train our model trained via perturbed gradient descent~\citep{jin2019} with learning rate $\eta$ and noise level $\sigma^2$. That is, if $\bW^{t}$ denotes the weights at time step $t$, the update is given by
\begin{equation}\label{eq: gd update}
\bW^{t+1} = \bW^t - \eta\left(\nabla_\bW L_\lambda(\bW^{t}) + \bXi_t\right),
\end{equation}
where $\bXi_t \in \RR^{d \times m}$ are i.i.d random matrices with each entry i.i.d $\cN(\frac{\sigma^2}{md})$.

Given these definitions, we now present our main theorem:

\begin{theorem}\label{thm: main thm}
Let $\varepsilon > 0$ be a target test accuracy, the number of samples be $n \gtrsim d^k\cdot\text{poly}(R)\cdot \max(\varepsilon^{-2}, \log d)$, and the width be $m = \text{poly}(n, d, R, \varepsilon^{-1})$. Let the sequence of iterates $\{\bW^t\}_{t\ge 0}$ follow the update in~\eqref{eq: gd update} with initialization $\bW^{0} = \bZero$. Then, there exists a choice of parameters $(\lambda_1, \lambda_2, \lambda_3, \lambda_4, r, \sigma^2, \eta)$ such that with high probability over $\bW_0, \cD_n$, and $\{\bXi_t\}_{t \ge 0}$, there exists a $\mathscr{T} = \text{poly}(m)$ such that the predictor $\hat \bW := \bW^\mathscr{T}$ satisfies $L(\hat\bW) \le \varepsilon$.
\end{theorem}

\textbf{Remark 1.} Theorem~\ref{thm: main thm} tells us $n = \tilde \Theta(d^k)$ samples are needed to learn the low-degree plus sparse function $f^*$. This is an improvement over the sample complexity needed to learn $f^*$ via the NTK, which is $\Omega(d^{k+1})$ samples as $f^*$ is a degree $(k+1)$-polynomial \citep{montanari2021, montanari2020}. This also improves over the upper bound for the sample complexity of the quadratic NTK given in \citep{bai2020}, in which $\Omega(K^2d^{k})$ samples are needed to learn a rank $K$ polynomial of degree $k+1$. This is polynomially worse than our bound since the dense, low-degree term $f_k$ can have rank $\gg d$.

\section{Proof Sketch}
The proof of Theorem~\ref{thm: main thm} follows similar high-level steps to~\citep{bai2020}. We first construct a $\bW^* \in \RR^{d \times m}$ which fits $f^*$ and is small on the regularizers. Next, we show that the optimization landscape of the regularized loss has a favorable geometry, and as a result that gradient descent converges to a global minimum. We conclude with a generalization bound to show that global minima have low test loss. Throughout the proof sketch, we emphasize how the regularizers encourage us to escape the NTK regime, and discuss the challenges posed by the existence of the $f_L$ term in the dynamics.

\subsection{Expressivity}

We begin by showing that the $f_L$ and $f_Q$ terms can fit the low degree and sparse components, respectively, of the signal. As in~\citep{montanari2021, montanari2020}, the derivations in this section rely on spherical harmonics; an overview of the technical results used are presented in Appendix~\ref{app: spherical harmonics}.

The following lemma shows that $f_Q(\bx; \bW)$ can fit the sparse, high degree term in $f^*$.

\begin{lemma}[QuadNTK can fit high degree component]\label{lemma: quad-ntk expressivity}
Let $m \ge d^{2k}$. With high probability over the initialization, there exists $\bW_Q$ such that 
\begin{equation}
\max_{i \in [n]}\left|f_Q(\bx_i; \bW_Q) - f_{sp}(\bx_i) \right| \lesssim \frac{d^k}{\sqrt{m}} \quad\text{and}\quad \|\bW_Q\|_{2, 4}^4 \lesssim d^{k-1}
\end{equation}
\end{lemma}

This generalizes the corresponding result in~\citep{bai2020} to more activations. The proof of this lemma is presented in Appendix~\ref{sec: quad-ntk expressivity proof}.

Next, we show that $f_L(\bx; \bW)$ can fit the low degree term. Here, we choose $r = n_k = \Theta(d^k)$, where $n_k$ is defined in Appendix~\ref{app: spherical harmonics}, to be the dimension of the subspace which can express degree $\le k$ polynomials. We define $\bP_{\le k} = \bPi_{\le r}$ to be the projection on the top $n_k$ eigenvectors of $\bSigma$.

\begin{lemma}[NTK can fit low degree component]\label{lemma: NTK fit low degree} Let $m \ge d^{10k}$. With high probability, there exists $\bW_L$ with $\text{vec}(\bW_L) \in \text{span}(\bP_{\le k})$ such that
\begin{equation}
	\EE_n[(f_L(\bx; \bW_L) - f_k(\bx))^2] \lesssim \frac{d^k}{n} \quad\text{and}\quad\|\bW_L\|_F^2 \lesssim d^{k-1}
\end{equation}
\end{lemma}
The proof of this Lemma is presented in Appendix~\ref{sec: prove NTK fit lemma}, and relies on key lemmas from~\citep{montanari2020} relating the spherical harmonics of degree $\le k$ to the eigenstructure of the kernel. A key intermediate result is Lemma~\ref{lemma: main intermediate spectral result}, which characterizes the spectrum of the population covariance matrix $\bSigma$. A similar result for random features was shown in~\citep{mei2022}. Unlike \citep{mei2022}, we do not characterize all the eigenvalues of $\bSigma$; however, simply partitioning them into the three categories is sufficient for our purposes.

Finally, we use the $\bW_L, \bW_Q$ to construct a $\bW^*$ which has small regularized loss. Recall the definition of the regularizers, after setting $r = n_k$:
\begin{align}
\cR_1(\bW) &= \cR_1(\bW; n_k) = \|f_L(\cdot; \bP_{>k}\bW)\|_{L^2}^2 = \EE_\mu\left[(f_L(\bx; \bP_{>k}\bW))^2\right]\\
\cR_2(\bW) &= \cR_2(\bW; n_k) =\|f_L(\cdot; \bP_{\le k}\bW)\|_{L^2}^2 = \EE_\mu\left[(f_L(\bx; \bP_{\le k}\bW))^2\right]\\
\cR_3(\bW) &= \cR_3(\bW; n_k) = \EE_n\left[(f_L(\bx; \bP_{>k}\bW))^2\right]\\
\cR_4(\bW) &= \|\bW\|_{2, 4}^8
\end{align}
Also, define the empirical loss of the quadratic model as:
\begin{equation}
\hat{L}^Q(\bW) := \EE_n[\ell(y, f_L(\bx; \bW) + f_Q(\bx; \bW))]
\end{equation}

The following theorem is the central expressivity result:

\begin{theorem}
\label{thm: main expressivity}
For $\varepsilon_{min} > 0$, let $m \gtrsim \max(d^{3(k-1)}\varepsilon_{min}^{-4}, d^{10k})$, $n \gtrsim \max(d^k\varepsilon_{min}^{-2}, d^k\log d)$. With probability $1 - 1/\text{poly}(d)$, there exists $\bW^*$ such that 
\begin{equation}
	\hat{L}^Q(\bW^*) \le  \varepsilon_{min}
\end{equation}
and:
\begin{align}
\cR_1(\bW^*) \lesssim m^{-\frac12}d^{\frac{k-1}{2}} \qquad \cR_2(\bW^*) \lesssim 1 \qquad \cR_3(\bW^*) \lesssim m^{-\frac12}d^{\frac{k-1}{2}} \qquad \cR_4(\bW^*) \lesssim d^{2(k-1)}.
\end{align}
\end{theorem}

The proof of this Theorem is presented in Appendix~\ref{sec: main expressivity proof}, and again relies on the eigendecomposition of $\bSigma$. As outlined in the introduction, we show $\bSigma = \bQ_1\bLambda_1\bQ_1^T + \bQ_2\bLambda_2\bQ_2^T + \bQ_3\bLambda_3\bQ_3^T$, where $\bLambda_2$ are the medium eigenvalues and $\bLambda_3$ are the small eigenvalues. Formally, $\bLambda_2$ contain $\Theta(d^{-i + 1})$ with multiplicity $\Theta(d^i)$, for integers $i \in [k + 1, 2k]$, and the entires of $\bLambda_3$ are equal to $1/m$ on average. This tells us that the medium directions $\bQ_2$ are undesirable, as any $\Omega_m(1)$ movement in these directions will case $f_L(\bx; \bW)$ to grow large. On the other hand, movement in a ``sufficiently random'' $\bQ_3$ direction will minimally affect the population value of $f_L(\bx; \bW)$.

To prove Theorem~\ref{thm: main expressivity}, we will construct $\bW^*$ to be of the form $\bW_L + \bW_Q$, where $\bW_L$ fits the low-degree term and $\bW_Q$ fits the sparse term. The issue with this direct construction is that $\|\bW_Q\|_F \gg 1$, so a priori $\bW_Q$ can have a large effect on the linear term. We thus require $f_L(\bx; \bW_Q)$ to be small, both on the sample and over the population. The key insight is the following: since $\text{dim}(\bV_1) \ll \text{dim}(\bV_2)$, any sufficiently random direction lies almost entirely in $\bV_2$. If $\bu$ is this random direction, then $\EE_n\left[(f_L(\bx; \bu))^2\right]$ is small. Similarly, the random direction $\bu$ lies almost entirely in $\bQ_3$. Since the $\bQ_3$ directions have eigenvalues $1/m$ on average, $\EE_\mu\left[(f_L(\bx; \bu))^2\right] \approx \frac{1}{m}\|\bu\|^2 \ll 1$.

We thus consider a ``sufficiently random'' version of $\bW_Q$ so that $\bW_Q$ minimally affects $f_L$. Specifically, we consider the weight $\bS\bW_Q$, where $\bS \in \RR^{m \times m}$ is a diagonal matrix of random signs. By definition $f_Q(\bx; \bS\bW_Q) = f_Q(\bx; \bW_Q)$, furthermore, we show $\bS\bW_Q$ is now sufficiently random in that $\EE_n\left[(f_L(\bx; \bu))^2\right], \EE_\mu\left[(f_L(\bx; \bu))^2\right]$ are both small. This allows us to prove that in expectation over $\bS$, a solution of the form $\bW_L + \bS\bW_Q$ acheives small regularized loss, which implies the existence of such a desirable solution via the probabilistic method.

\subsection{Landscape}
Define first and second-order stationary points~\citep{jin2019} as follows:

\begin{definition}
$\bW$ is a $\nu$-first-order stationary point of $f$ if $\|\nabla f(\bW)\| \le \nu$.
\end{definition}
\begin{definition}
$\bW$ is a $(\nu, \gamma)$-second-order stationary point (SOSP) of $f$ if $\nabla^2f(\bW) \succeq -\gamma \bI$.
\end{definition}

The following lemma is our central landscape result:
\begin{lemma}
\label{lemma: central landscape result}
Let $r = d_k, \lambda_2 = \varepsilon_{min}$, $\lambda_3 = m^{\frac12}d^{-\frac{k-1}{2}}\varepsilon_{min}$, $\lambda_4 = d^{-2(k-1)}\varepsilon_{min}$. Assume $m \ge n^4d^{\frac{26(k + 1)}{3}}\varepsilon_{min}^{-22/3}$, and $\nu \le m^{-\frac14}$. Let $\bW$ be a $\nu$-first order stationary point, and let $\bW^*$ be the solution constructed in~\ref{thm: main expressivity}. Then,
\begin{align}
	\mathbb{E}_{\bS}\left[\nabla^2L_\lambda(\bW)[\bS\bW^*, \bS\bW^*]\right] - \langle \nabla L_\lambda(\bW), \bW - 2\bW^*_L + \bW_L \rangle + 2L_\lambda(\bW) - 2L_\lambda(\bW^*) \lesssim \varepsilon_{min}.
\end{align}
\end{lemma}
As a corollary, we show that any $(\nu, \gamma)$-SOSP of $L_{\blambda}(\bW)$ has small loss.
\begin{corollary}
\label{cor: main landscape cor}
Set $r, \lambda$ as in Lemma~\ref{lemma: central landscape result}. Let $\hat{\bW}$ be a $(\nu, \gamma)$-second-order stationary point of $L_{\blambda}(\bW)$, with $\nu \le m^{-1/2}, \gamma \le m^{-3/4}$. Then $L_{\blambda}(\hat{\bW}) \lesssim \varepsilon_{min}$.
\end{corollary}
The proofs of these results are deferred to Appendix~\ref{app: landscape proofs}.

\subsection{Optimization}
We next invoke the main theorem of~\citep{jin2017, jin2019}, which is that perturbed gradient descent will find a $(\nu, \gamma)$-SOSP in $\text{poly}(1/\nu, 1/\gamma)$ time. The challenge with applying these results directly is that $L_\lambda$ is no longer smooth or Hessian-Lipschitz due to the $\cR_4$ regularizer. To circumvent this, we first prove that the iterates in perturbed gradient descent are bounded in a Frobenius norm ball. Then, it suffices to use the bound on smoothness and Hessian-Lipschitzness in this ball to prove convergence. Our main optimization result, with proof in Appendix~\ref{app: optimization proofs}, is the following:

\begin{theorem}\label{thm: main optimization}
There exists a choice of learning rate $\eta$ and perturbation radius $\sigma$ such that with probability $1-1/\text{poly}(d)$, perturbed gradient descent (c.f~\cite[Algorithm 1]{jin2019}) reaches a $(\nu, \gamma)$-SOSP within $\mathscr{T} = \text{poly}(m)$ timesteps.
\end{theorem}

\subsection{Generalization}
Finally, we conclude by showing that any $\hat\bW$ with small $L_\lambda(\hat \bW)$ also has small test loss:
\begin{theorem}[Main generalization theorem]\label{thm: main generalization}
Let $r = d_k, \lambda_1 = m^{\frac12}d^{-\frac{k-1}{2}}\varepsilon_{min}$, $\lambda_2 = \varepsilon_{min}$, $\lambda_3 = m^{\frac12}d^{-\frac{k-1}{2}}\varepsilon_{min}$, $\lambda_4 = d^{-2(k-1)}\varepsilon_{min}$. Assume $m \gtrsim \varepsilon_{min}^{-4}d^{3(k-1)}$. With probability $1 - 1/\text{poly}(d)$ over the draw of $\cD$, any data dependent $\hat{\bW}$ with $L_\lambda(\hat \bW) \le C\varepsilon_{min}$ has population loss
\begin{equation}
L(\hat \bW) \lesssim \varepsilon_{min} + \sqrt{\frac{d^k}{n}}.
\end{equation}
\end{theorem}

The proof of this theorem is presented in Appendix~\ref{app: generalization proofs}. The key is to use the small values of the regularizers at $\hat \bW$ to bound the test loss. $\cR_2$ and $\cR_4$ are used to bound the Rademacher complexities of the linear and quadratic terms respectively, while $\cR_1$ and $\cR_3$ control the influence of the high degree component of the linear term on the population loss and empirical loss respectively.

\section{Experiments}
We conclude with experiments to support our main theorem. In Figure~\ref{fig: main figure} we train the joint linear and quadratic model $f_L(\bx; \bW) + f_Q(\bx; \bW)$ via gradient descent on the square loss, with signal $f^*(\bx) = f_1(\bx) + (\beta^T\bx)^2$ where $f_1(\bx) = x_1 - 1$. Our covariates are dimension $d=100$, we have $n = d^{1.5} = 1000$ samples, and our network has width $m = 10000$.

We train our model with the regularizer $\cR_3(\bW)$ for varying values of $\lambda_3$. Rather than computing $\bSigma$, we use the top $n_k$ right singular vectors of $\bPhi$ to estimate $\cR_3$. We observe that for all $\lambda_3$ we reach near zero training error. However, when $\lambda_3$ is zero or very small, the model struggles to learn the entire signal, and test error plateaus near 0.6. In the leftmost pane, we plot the value of $\cR_3(\bW)$ over the course of training. We observe that $\cR_3(\bW)$ grows large for small $\lambda_3$, which implies that the model is using the ``bad" directions in $\bV_2$ to overfit to the training data. For large values of $\lambda_3$, however, $\cR_3$ prevents $\bW$ from moving in the ``bad" directions to overfit the data. This is seen in the leftmost pane, where the value of $\cR_3(\bW)$ is small. Instead, the parameter moves in the ``good'' $\bQ_3$ directions, and the Quad-NTK term kicks in to fit the remaining component of the signal while generalizing. This leads to the consistently lower test loss (around 0.2) as shown in the rightmost pane.

\begin{figure}
\centering
\includegraphics[width=\textwidth]{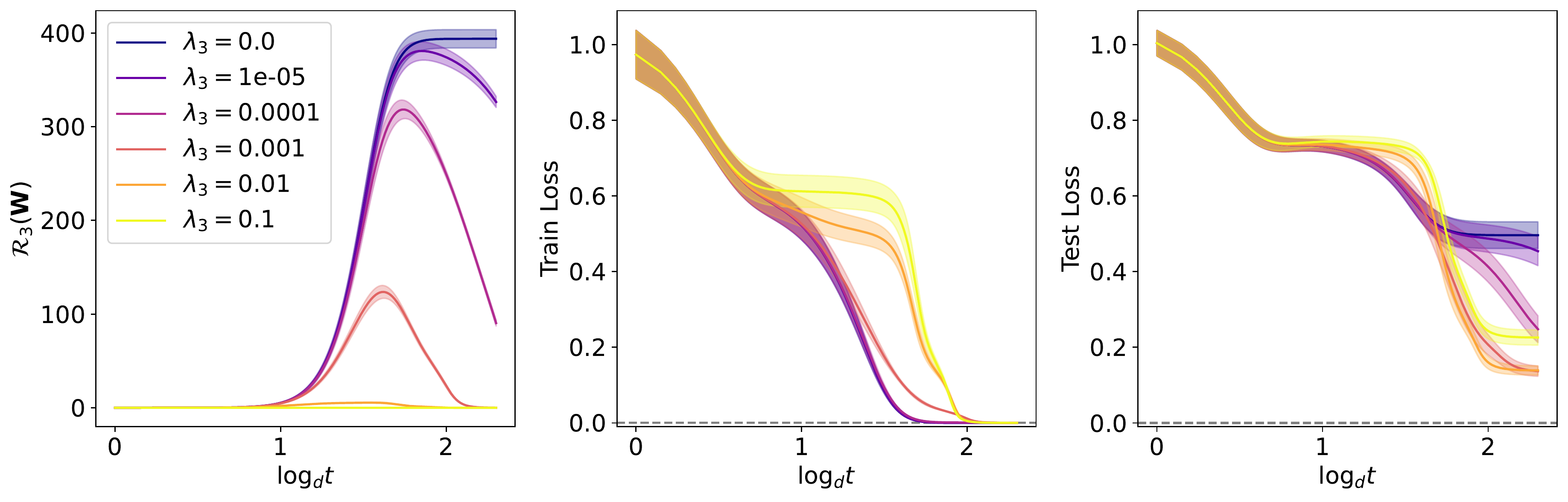}
\caption{We train $f_L + f_Q$ with varying $\lambda_3$. When $\lambda_3$ is small, the NTK overfits the high degree signal and test error is large. When $\lambda_3$ is large, the QuadNTK can learn the high degree signal, and test error is smaller. Results are averaged over 5 trials, with one standard deviation shown.}
\label{fig: main figure}
\end{figure}

\paragraph{On the regularizer $\cR_1$.} Our proof required the existence of $\cR_1$ to prevent movement in the bad $\bQ_2$ directions. While $\cR_1$ being small is necessary for generalization, in Figure~\ref{fig: main figure} we observe that we generalize well without explicitly regularizing $\cR_1$. We hypothesize that the noise in the perturbed gradient descent update may be implicitly regularizing $\cR_1$. A heuristic justification is as follows. Recall that any sufficiently random direction $\bu$ lies mostly in $\bQ_3$. Due to the rotational symmetry of the perturbations, one might expect that the solutions $\bP_{\le k}\hat \bW + \bP_{\> k}\hat \bW\bS$, over all choices of random signs $\bS$, can be reached with roughly equal probability. Since this set of solutions is ``sufficiently random,'' in expectation over $\bS$ it generalizes well, and thus gradient descent is likely to converge to a solution which generalizes well. In Appendix~\ref{app:extra_experiments} we conduct experiments with both $\cR_1$ and $\cR_3$ to support the claim that explicit regularization of $\cR_1$ is not needed to acheive low test error; however, proving this claim is left for future work.

\clearpage

\begin{wrapfigure}{r}{0.37\textwidth} 
\centering
\includegraphics[width=0.37\textwidth]{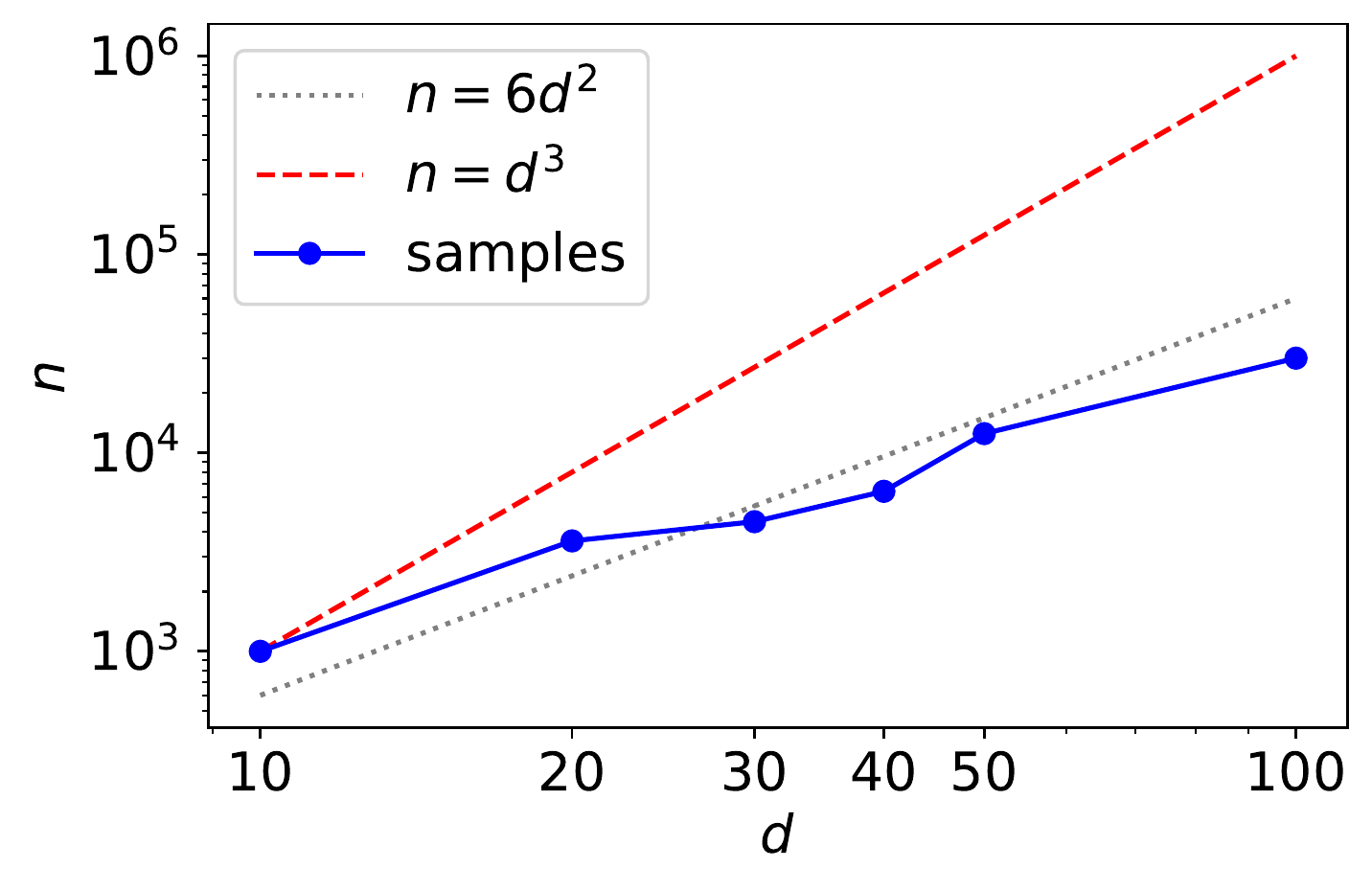}
\caption{Neural networks optimally learn the ``dense quadratic plus sparse cubic" task.}
\label{fig:real}
\end{wrapfigure}

\paragraph{On the Low-Degree Plus Sparse Task.} Theorem~\ref{thm: main thm} shows that two-layer neural networks trained via noisy gradient descent with a specific regularizer can more efficiently learn low-degree plus sparse target functions. In Figure~\ref{fig:real} we show empirically that neural networks with standard initialization and trained via vanilla gradient descent efficiently learn a ``dense quadratic plus sparse cubic.'' For varying values of $d$, we train a two-layer neural network to convergence and plot the smallest $n$ such that the test loss was $<0.1$. We observe this minimal $n$ roughly scales with $d^2$, the optimal (in $d$) sample complexity. This provides evidence that standard networks trained with vanilla GD can effectively learn the low-degree plus sparse task, and hence that this is a sensible task to study and that our work presents a step towards understanding why standard neural networks perform better in practice more generally. See Appendix~\ref{app:extra_experiments} for more details.

\paragraph{On Second-Order Taylor Expansions of Standard Networks.} Throughout this paper we studied the quadratic Taylor expansion of a two-layer neural network. In Appendix~\ref{app:extra_experiments}, we complement the results of~\cite{bai2020taylorized} and show that for a standard architecture and data distribution, the quadratic Taylor expansion better approximates the optimization dynamics and acheives lower test loss than the linearization (NTK) does.



\section{Conclusion}
The goal of this work is to better understand how neural networks can escape the NTK regime. By analyzing the eigendecomposition of the feature covariance matrix, we identified 3 sets of directions -- ones that can fit low degree signal, ``bad'' directions which either cause the NTK to overfit or the test predictions to explode, and ``good'' directions in which the parameters can move a large distance. We then constructed a regularizer which encourages movement in these good directions, and showed how a network can jointly use the linear and quadratic terms in its Taylor expansion to fit a low-degree plus sparse signal. Altogether, we provided an end-to-end convergence and generalization guarantee with a provable sample complexity improvement over the NTK and QuadNTK.

As discussed above, one interesting direction of future work is to understand the role of $\cR_1$. Other directions of future work include understanding whether our analysis can be used to leverage higher-order terms in the Taylor expansion, understanding the connection between the QuadNTK and feature learning, and investigating whether increasing the depth of the network can allow the NTK and QuadNTK to jointly learn a hierarchical representation~\citep{chen2020}.

\section*{Acknowledgements}
EN acknowledges support from a National Defense Science \& Engineering Graduate Fellowship. JDL and EN acknowledge support of the ARO under MURI Award W911NF-11-1-0304, the Sloan Research Fellowship, NSF CCF 2002272, NSF IIS 2107304, NSF CIF 2212262, ONR Young Investigator Award, and NSF-CAREER under award \#2144994. The authors would like to thank Alex Damian for helpful discussions throughout the project.

\bibliographystyle{plainnat}
\bibliography{main}

\clearpage
\appendix 

\section{Spherical Harmonics: Technical Background}\label{app: spherical harmonics}
Below we present relevant results on spherical harmonics, Gegenbauer polynomials, and Hermite polynomials. These results are from~\cite{montanari2021, montanari2020}, and a more in depth discussion of the technical background can be found in those references.

For $\ell \in \ZZ^{\ge 0}$, define $B(d, \ell)$ as
\begin{align}
	B(d, \ell) := \frac{2\ell + d - 2}{\ell}\binom{\ell + d - 3}{\ell - 1} = (1 + o_d(1))\frac{d^\ell}{\ell!},
\end{align}
and define
\begin{align}
	n_k = \sum_{\ell = 0}^k B(d, \ell) = \Theta_d(d^k).
\end{align}

We let $L^2(\cS^{d-1}(\sqrt{d}), \mu)$ denote the space of square-integrable functions over the sphere of radius $\sqrt{d}$, with respect to the uniform probability measure $\mu$. We use the shorthand $\langle \cdot, \cdot \rangle_{L^2} = \langle \cdot, \cdot \rangle_{L^2(\cS^{d-1}(\sqrt{d}), \mu)}$, and likewise for $\|f\|_{L^2}$.

The \emph{normalized spherical harmonics} $\{Y_{k, i}^{(d)}\}_{0 \le i \le B(d, k), k \ge 0}$ are a sequence of polynomials such that $Y^{(d)}_{k, i}$ is degree $k$, and the $Y^{(d)}_{k, i}$ form an orthonormal basis of $L^2(\cS^{d-1}(\sqrt{d}), \mu)$, i.e:
\begin{align}
 \langle Y^{(d)}_{k, i}, Y^{(d)}_{m, j} \rangle_{L^2} = \delta_{km}\delta_{ij}.
\end{align}

For $\bx \sim \cS^{(d-1)}(\sqrt{d})$, let $\tilde \tau_{d-1}^1$ be the probability measure of $\sqrt{d}\langle \bx, \be \rangle$ and $\tau_{d-1}^1$ be the measure of $\langle \bx, \be \rangle$, where $\be$ is an arbitrary unit vector.

The \emph{Gegenbauer polynomials} $\{Q^{(d)}\}_{k \ge 0}$ are a basis of $L^2([-d, d], \tilde \tau_{d-1}^1)$ such that $Q^{(d)}_k$ is a degree $k$ polynomial with $Q_k^{(d)}(d) = 1$ and 
\begin{align}
	\langle Q_k^{(d)}, Q_j^{(d)} \rangle_{L^2([-d, d], \tilde \tau_{d-1}^1)} = \frac{1}{B(d, k)}\delta_{jk}.
\end{align}

The following identity relates spherical harmonics and Gegenbauer polynomials:
\begin{align}\label{eq:spherical-to-geg}
	Q_k^{(d)}(\langle \bx, \by \rangle) = \frac{1}{B(d, k)}\sum_{i=1}^{B(d, k)}Y_{ki}^{(d)}(\bx)Y_{ki}^{(d)}(\by).
\end{align}

We also use the following fact about Gegenbauer polynomials:
\begin{align}\label{eq:geg-identity}
\langle Q_j^{(d)}(\langle \bx, \cdot \rangle), Q_k^{(d)}(\langle \by, \cdot \rangle) \rangle_{L^2} = \frac{1}{B(d, k)}\delta_{jk}Q_k^{(d)}(\langle \bx, \by \rangle).
\end{align}

Furthermore, $f \in L^2([-\sqrt{d}, \sqrt{d}], \tau_{d-1}^1)$ can be decomposed via Gegenbauer polynomials as
\begin{align}
f(x) &= \sum_{k=0}^\infty \lambda_{d, k}(\sigma')B(d, k)Q_k^{(d)}(\sqrt{d}x)\\
\lambda_{d, k}(f) &:= \langle f, Q_k^{(d)}(\sqrt{d}\cdot) \rangle_{L^2([-\sqrt{d}, \sqrt{d}], \tau_{d-1}^1)}.
\end{align}
Thus $\|f\|_{L^2([-\sqrt{d}, \sqrt{d}], \tau_{d-1}^1)} = \sum_{k \ge 0}B(d, k)\lambda^2_{d, k}(f)$

Let $\gamma$ be the measure of a standard Gaussian on $\mathbb{R}$. The \emph{normalized Hermite polynomials} $\{h_k\}_{h \ge 0}$ are an orthonormal basis of $L^2(\RR, \gamma)$ such that $h_k$ is degree $k$. For $f \in L^2(\RR, \gamma)$, let $\mu_k(f) := \langle f, h_k \rangle_{L^2(\RR, \gamma)}$ be the $k$th hermite coefficient. One observes that $\tau_{d-1}^1$ converges weakly to $\gamma$. As a result we have the following connection between Hermite and Gegenbauer coefficients:
\begin{align}
\mu_k(f) = \lim_{d \rightarrow \infty} B(d, k)^{1/2}\lambda_{d, k}(f).
\end{align}

\subsection{Assumptions}

We additionally require $\sigma', \sigma''$ to satisfy the following assumptions, which are that the Hermite/Gegenbauer coefficients are nonzero and well behaved:\\

\begin{assumption}\label{assume: hermite coeffs}
We assume $\sigma, \sigma'$ satisfy the following:
\begin{enumerate}[(a)]
	\item Let $\sigma'$ satisfy $\mu_\ell(\sigma') \neq 0$ for $\ell \le 4k$ and $\sum_{\ell > 4k}\mu^2_\ell(\sigma') > 0$. As a result, we can let $\mu^2_\ell(\sigma') = \Theta_d(1)$ for $\ell < 4k$.
	\item Let $\sigma''$ satisfy
	\begin{align}
		d^{k-1}\cdot \min_{\ell \le k-1}\lambda^2_{d, \ell}(\sigma'') = \Omega_d(1),
	\end{align}
	where $\lambda_{d, \ell}(\sigma'') = \langle \sigma'', Q^{(d)}_\ell(\sqrt{d} \cdot) \rangle_{L^2([-\sqrt{d}, \sqrt{d}], \tau_{d-1}^1)}$.
\end{enumerate}
\end{assumption}

\subsection{Computing \texorpdfstring{$\bSigma$}{Sigma}}
Below we discuss how to express $\bSigma$ in terms of the Gegenbauer coefficients of $\sigma'$, $\lambda_{d, \ell}(\sigma')$. First, observe that that $\bSigma$ is a matrix of $d \times d$ blocks, where the $i,j$th block is equal to
\begin{align*}
	\frac{a_ia_j}{m}u(\bw_{0, i}, \bw_{0, j}),
\end{align*}
where $\bu : \cS^{d-1}(\sqrt{d}) \times \cS^{d-1}(\sqrt{d}) \rightarrow \RR^{d \times d}$ is the function
\begin{align*}
\bu(\btheta_1, \btheta_2) = \EE_\mu\left[\sigma'(\btheta_1^T\bx)\sigma'(\btheta_2^T\bx)\bx\bx^T\right].
\end{align*}

\cite[Lemma 7]{montanari2021} shows that there exist scalar valued functions $u_1, u_2, u_3$ such that
\begin{align*}
\bu(\btheta_1, \btheta_2) = u_1(\btheta_1^T\btheta_2)\bI_d + u_2(\btheta_1^T\btheta_2)[\btheta_1\btheta_2^T + \btheta_2\btheta_2^T] + u_3(\btheta_1^T\btheta_2)[\btheta_1\btheta_1^T + \btheta_2\btheta_2^T],
\end{align*}
where $u_1, u_2, u_3$ can be computed in terms of the quantities $$\text{Tr}(\bu(\btheta_1, \btheta_2)), \quad \btheta_1^T\bu(\btheta_1, \btheta_2)\btheta_2, \quad \btheta_1^T\bu(\btheta_1, \btheta_2)\btheta_1.$$ 
It thus suffices to compute these quantities. We assume that we can compute arbitrarily many Gegenbauer coefficients of $\sigma'$.

Note that
\begin{align*}
	\text{Tr}(\bu(\btheta_1, \btheta_2)) &= d\cdot \EE_\mu[\sigma'(\btheta_1^T\bx)\sigma'(\btheta_2^T\bx)]\\
	\btheta_1^T\bu(\btheta_1, \btheta_2)\btheta_2 &= \EE_\mu[\sigma'(\btheta_1^T\bx)\btheta_1^T\bx\cdot\sigma'(\btheta_2^T\bx)\btheta_2^T\bx]\\
	\btheta_1^T\bu(\btheta_1, \btheta_2)\btheta_1 &= \EE_\mu[\sigma'(\btheta_1^T\bx)(\btheta_1^T\bx)^2\cdot\sigma'(\btheta_2^T\bx)].
\end{align*}
All these expressions are of the form $\EE_\mu[f(\btheta_1^T\bx)g(\btheta_2^T\bx)]$. For arbitrary $f, g$, let their decompositions into Gegenbauer polynomials be:
\begin{align*}
f(z) = \sum_{k \ge 0}\lambda_{d,k}(f)B(d, k)Q_k^{(d)}(\sqrt{d}z), \qquad g(z) = \sum_{k \ge 0}\lambda_{d,k}(g)B(d, k)Q_k^{(d)}(\sqrt{d}z).
\end{align*}
Then, by Equation~\ref{eq:geg-identity},
\begin{align*}
	\EE_\mu[f(\btheta_1^T\bx)g(\btheta_2^T\bx)] &= \sum_{k, \ell \ge 0}\lambda_{d,k}(f)\lambda_{d, \ell}(g)B(d, k)B(d, \ell)\EE_\mu\left[Q_k^{(d)}(\sqrt{d}\btheta_1^T\bx)Q_\ell^{(d)}(\sqrt{d}\btheta_2^T\bx)\right]\\
	&= \sum_{k \ge 0} \lambda_{d, k}(f)\lambda_{d, k}(g)B(d, k)Q_k^{(d)}(d\btheta_1^T\btheta_2),
\end{align*}
which can be computed to desired precision by truncating this infinite sum accordingly. This only requires knowledge of the Gegenbauer coefficients of $f, g$. Given the Gegenbauer coefficients of a function $\psi(z)$, \cite[Lemma 6]{montanari2021} gives a formula for the Gegenbauer coefficients of $z\psi(z)$. We can therefore write the Gegenbauer coefficients of $z\sigma'(z), z^2\sigma'(z)$ in terms of those of $\sigma'$, and thus we can approximate this sum to the desired precision. This procedure allows us to express $\bSigma$ in terms of the Gegenbauer coefficients of $\sigma'$.

\section{Expressivity Proofs}
\subsection{Quad-NTK Proofs}
\subsubsection{Preliminaries}

\begin{lemma}[Expressing polynomials with random features]\label{lemma: express polynomial RF}
Let $p \ge 0$, and let $\sigma$ satisfy the following two assumptions: 
\begin{enumerate}[(a)]
	\item $\sigma \in L^2([-\sqrt{d}, \sqrt{d}], \tau_{d-1}^1)$
	\item $d^p\cdot \min_{k \le p}\lambda_{d, k}(\sigma) = \Omega_d(1)$, where $\lambda_{d, k}(\sigma)^2 = \langle \sigma, Q^{(d)}_k(\sqrt{d} \cdot) \rangle_{L^2([-\sqrt{d}, \sqrt{d}], \tau_{d-1}^1)}$.
\end{enumerate}
For $|\alpha| \le 1$, $\|\beta\| = 1$, there exists a function $a \in L^2(\cS^{d-1}(1))$ such that
\[
	\mathbb{E}_{\bw_0 \sim \cS^{d-1}(1)}\left[\sigma(\bw_0^T\bx)a(\bw_0)\right] = \alpha(\beta^T\bx)^p,
\]
and $a$ satisfies the norm bound
\[
	\|a\|^2_{L^2(\cS^{d-1}(1))} \lesssim d^p
\]
\end{lemma}

\begin{proof}
We can decompose $\sigma$ into a sum over Gegenbauer polynomials
\[
	\sigma(x) = \sum_{k = 0}^\infty \lambda_{d, k}(\sigma)B(d, k)Q_k^{(d)}(\sqrt{d}x),
\]
By Equation~\ref{eq:spherical-to-geg},
\[
	\sigma(\bw_0^T\bx) = \sum_{k \ge 0}\sum_{i=1}^{B(d, k)}\lambda_{d, k}(\sigma)Y_{k, i}^{(d)}(\bx)Y_{k, i}^{(d)}(\sqrt{d}\bw_0).
\]
Let $a$ be decomposed into spherical harmonics as
\[
	a(\bw_0) = \sum_{k \ge 0}\sum_{i=1}^{B(d, k)}c_{k, i}Y_{k, i}^{(d)}(\sqrt{d}\bw_0),
\]
for some coefficients $c_{k, i}$ with $\sum_{k \ge 0}\sum_{i=1}^{B(d, k)}c_{k, i}^2 < \infty$. Since the spherical harmonics form an orthonormal basis of $\cS^{d-1}(\sqrt{d})$, we have
\begin{align*}
	\mathbb{E}_{\bw_0 \sim \cS^{d-1}(1)}\left[\sigma(\bw_0^T\bx)a(\bw_0)\right] &= \sum_{k \ge 0}\sum_{i=1}^{B(d, k)} \lambda_{d, k}(\sigma)c_{k, i}Y_{k, i}^{(d)}(\bx).
\end{align*}

Next, for an arbitrary function $f \in L^2([-\sqrt{d}, \sqrt{d}], \tau_{d-1}^1)$, we can decompose
\begin{align*}
f(\beta^T\bx) &= \sum_{k\ge 0}\lambda_{d, k}(f)B(d, k)Q_k(\sqrt{d}\beta^T\bx)\\
&= \sum_{k\ge 0}\sum_{i=1}^{B(d, k)}\lambda_{d, k}(f)Y_{k, i}^{(d)}(\bx)Y_{k, i}^{(d)}(\sqrt{d}\beta).
\end{align*}
For $f(t) = \alpha t^p$, $\lambda_{d, k}(f) = 0$ for $k > p$, and thus
\[
	\alpha(\beta^T\bx)^p = \sum_{k=0}^p\sum_{i=1}^{B(d, k)}\lambda_{d, k}(f)Y_{k, i}^{(d)}(\bx)Y_{k, i}^{(d)}(\sqrt{d}\beta).
\]

Define the sequence of coefficients $\{c_{k, i}\}_{0 \le k \le p, 1 \le i \le B(d, k)}$ by
\[
	c_{k, i} = Y_{k, i}^{(d)}(\sqrt{d}\beta)\lambda_{d, k}(f)\lambda^{-1}_{d, k}(\sigma),
\]
which are well defined for sufficiently large $d$ by assumption (b) of the lemma. Since there are only finitely many nonzero $c_{k, i}$, the function $a(\bw_0)$ is in $L_2(\cS^{d-1}(1))$. Also,
\begin{align*}
	\EE_{\bw_0 \sim \cS^{d-1}(1)}\left[\sigma(\bw_0^T\bx)a(\bw_0)\right] &= \sum_{k \ge 0}\sum_{i=1}^{B(d, k)} \lambda_{d, k}(\sigma)c_{k, i}Y_{k, i}^{(d)}(\bx)\\
	&= \sum_{k = 0}^p\sum_{i=1}^{B(d, k)} \lambda_{d, k}(f)Y_{k, i}^{(d)}(\bx)Y_{k, i}^{(d)}(\sqrt{d}\beta)\\
	&= \alpha(\beta^T\bx)^p,
\end{align*}
as desired. To obtain a norm bound on $a$, we can write
\begin{align*}
\|a\|_{L^2}^2 &= \sum_{k = 0}^p \sum_{i = 1}^{B(d, k)} c_{k, i}^2\\
&= \sum_{k = 0}^p \frac{\lambda_{d, k}(f)^2}{\lambda_{d, k}(\sigma)^2}\sum_{i = 1}^{B(d, k)}Y_{k, i}^{(d)}(\beta\sqrt{d})^2\\
&= \sum_{k = 0}^p \frac{\lambda_{d, k}(f)^2}{\lambda_{d, k}(\sigma)^2}B(d, k)\\
&\lesssim d^p\sum_{k=0}^p\lambda_{d, k}(f)^2 B(d, k)\\
&= d^p\|f\|_{L^2([-\sqrt{d}, \sqrt{d}], \tau_{d-1}^1)}\\
&\lesssim d^p,
\end{align*}
since 
\begin{align*}
\|f\|_{L^2([-\sqrt{d}, \sqrt{d}], \tau_{d-1}^1)} \rightarrow_{d \rightarrow \infty} \|f\|_{L^2(\mathbb{R}, \gamma)},
\end{align*}
and thus $\|f\|_{L^2([-\sqrt{d}, \sqrt{d}], \tau_{d-1}^1)}  = \Theta_d(1)$.
\end{proof}

\begin{lemma}[Expressivity via infinitely many neurons]\label{lemma: infinite width expressivity}
Let $k \ge 1$, and let $\sigma$ be a twice-differentiable activation such that $\sigma''$ satisfies Assumption~\ref{assume: hermite coeffs}. Then, there exist functions $\bw_+, \bw_- : \cS^{d-1}(1) \rightarrow \RR$ such that
\[
	\EE_{\bw_0}\left[\sigma''(\bw_0^T\bx)\left((\bw_+^T\bx)^2 - (\bw_-^T\bx)^2\right)\right] = \alpha(\beta^T\bx)^{k+1},
\]
and
\[
	\EE_{\bw_0}\left[\|\bw_+\|_2^4 + \|\bw_-\|_2^4\right] \lesssim d^{k-1}
\]
\end{lemma}

\begin{proof}
Note that since $\sigma''$ is continuous and bounded, $\sigma'' \in L^2([-\sqrt{d}, \sqrt{d}], \tau^1_{d-1})$. Therefore applying Lemma~\ref{lemma: express polynomial RF} with activation $\sigma''$ and degree $k-1$, there exists a function $a$ satisfying
\[
	\EE_{\bw_0 \sim \cS^{d-1}(1)}\left[\sigma''(\bw_0^T\bx)a(\bw_0)\right] = (\beta^T\bx)^{k-1}
\]
and
\[
	 \|a\|^2_{L^2(\cS^{d-1}(1))} \lesssim d^{k-1}.
\]
Define
\begin{align*}
	\bw_+(\bw_0) &= \sqrt{\max(0, a(\bw_0))}\cdot \beta\\
	\bw_-(\bw_0) &= \sqrt{-\min(0, a(\bw_0))}\cdot \beta.
\end{align*}
Then
\[
	(\bw_+^T\bx)^2 - (\bw_-^T\bx)^2 = (\beta^T\bx)^2\left(\max(0, a(\bw_0)) + \min(0, a(\bw_0))\right) = a(\bw_0)(\beta^T\bx)^2,
\]
so
\begin{align*}
	\EE_{\bw_0}\left[\sigma''(\bw_0^T\bx)\left((\bw_+^T\bx)^2 - (\bw_-^T\bx)^2\right)\right] &= \EE_{\bw_0}\left[\sigma''(\bw_0^T\bx)a(\bw_0)\right](\beta^T\bx)^2\\
	&= \alpha(\beta^T\bx)^{k-1}(\beta^T\bx)^2\\
	&= \alpha(\beta^T\bx)^{k + 1}.
\end{align*}

Finally, we have the norm bound
\begin{align*}
\EE_{\bw_0}\left[\|\bw_+\|_2^4 + \|\bw_-\|_2^4\right] &= \EE_{\bw_0}\left[\max(0, a(\bw_0))^2 + \min(0, a(\bw_0))^2\right]\\
&= \EE_{\bw_0}\left[a(\bw_0)^2\right]\\
&\lesssim d^{k-1}.
\end{align*}
\end{proof}

\subsubsection{Proof of Lemma~\ref{lemma: quad-ntk expressivity}}\label{sec: quad-ntk expressivity proof}
\begin{proof}
We show this Lemma holds with probability $1 - d^{-10}$. 

Define $M := \lfloor \frac{m}{2R} \rfloor$. Define for $i \in [R]$, define the subnetwork $f^i_Q(\bx, \bW_Q)$ by
\begin{align*}
	f^i_Q(\bx, \bW_Q) &:= \frac{1}{2\sqrt{m}}\sum_{r=(i-1)M + 1}^{iM}\sigma''(\bw_{0, r}^T\bx)\left((\bx^T\bw_{Q, r})^2 - (\bx^T\bw_{Q, {r + m/2}})^2\right).
\end{align*}
We will now construct $\bW_Q \in \RR^{d \times m}$. For $RM < r \le m/2$, set $\bw_{Q, r} = \bw_{Q, r + m/2} = \bZero$. As a result, we have that
\begin{align*}
	f_Q(\bx, \bW_Q) = \sum_{i \in [R]}f_Q^i(\bx, \bW_Q).
\end{align*}
Our construction will proceed by expressing $\alpha_i(\beta_i^T\bx)^{k+1}$ with $f_Q^i(\bx, \bW_Q)$. For fixed $i \in [R]$, let $\bw^i_+, \bw^i_-$ be from the infinite width construction in Lemma~\ref{lemma: infinite width expressivity}, so that
\[
	\EE_{\bw_0}\left[\sigma''(\bw_0^T\bx)\left(({\bw^i_+}^T\bx)^2 - ({\bw^i_-}^T\bx)^2\right)\right] = \alpha_i(\beta_i^T\bx)^{k+1}.
\]
For integers $(i-1)M + 1 \le r \le iM$, define
\[
(\bw_{Q, r}, \bw_{Q, r+m/2}) = \sqrt{2/M}\cdot(m^{1/4}\bw^i_+(\bw_{0, r}), m^{1/4}\bw^i_-(\bw_{0, r})).
\]
Then,
\begin{align*}
f^i_Q(\bx; \bW_Q) &= \frac{1}{2\sqrt{m}}\sum_{r=(i-1)M + 1}^{iM}\sigma''(\bw_{0, r}^T\bx)\left((\bx^T\bw_{Q, r})^2 - (\bx^T\bw_{Q, r+m/2})^2\right)\\
&= \frac{1}{M}\sum_{r=(i-1)M + 1}^{iM}\sigma''(\bw_{0, r}^T\bx)a^i(\bw_{0, r})(\beta^T\bx)^2.
\end{align*}
Note that we can bound
\begin{align*}
|a^i(\bw_0)| &\le \sum_{k'=0}^{k-1} \sum_{j=1}^{B(d, k')}\left|c_{k', j}Y_{k, j}^{(d)}(\sqrt{d}\bw_0)\right|\\
&\le \left(\sum_{k'=0}^{k-1} \sum_{j=1}^{B(d, k')}c_{k', j}^2\right)^{1/2}\left(\sum_{k'=0}^{k-1} \sum_{i=1}^{B(d, k')}{Y_{k', j}^{(d)}(\sqrt{d}\bw_0)}^2\right)^{1/2}\\
&= \|a^i\|_{L^2}\left(\sum_{k'=0}^{k-1} B(d, k)\right)^{1/2}\\
&\lesssim d^{k-1}.
\end{align*}
Therefore letting $Z_r = \sigma''(\bw_{0, r}^T\bx)a^i(\bw_{0, r})(\beta^T\bx)^2$, we have $|Z_r| \lesssim d^{k}$. Also, the $Z_r$ are i.i.d and satisfy $\EE[Z_r] = (\beta^T\bx)^{k+1}$. Therefore by Hoeffding's inequality, with probability $1 - \frac12d^{-11}n^{-1}$ we have
\begin{align*}
	\left|f_Q^i(\bx_j; \bW_Q) -  \alpha_i(\beta_i^T\bx_j)^{k+1}\right| &= \left|\frac{1}{M}\sum_{r=(i-1)M+1}^{iM}Z_r - \EE[Z_r]\right|\\
	&= \tilde O \left( \frac{d^k}{\sqrt{m}} \right),
\end{align*}
where we omit $\text{poly}(R)$ dependencies inside the big $O$ notation.

Union bounding over $j \in [n]$, with probability $1 - \frac12d^{-11}$ over the initialization,
\[
	\max_{j \in [n]}\left|f^i_Q(\bx_j; \bW_Q) -  \alpha_i(\beta_i^T\bx_j)^{k+1}\right| \le \tilde O \left( \frac{d^k}{\sqrt{m}} \right).
\]
Since the above holds for all $i \in [R]$, union bounding over $i$ yields that with probability $1 - \frac12d^{-11}R \ge 1 - \frac12d^{-10}$,
\begin{align*}
	\max_{j \in [n]}\left|f_Q(\bx_j, \bW_Q) - f_{sp}(\bx)\right| &\le \sum_{i \in [R]}\max_{j \in [n]}\left|f^i_Q(\bx_j; \bW_Q) -  \alpha_i(\beta_i^T\bx_j)^{k+1}\right|\\
	&\le \tilde O \left( \frac{d^k}{\sqrt{m}} \right).
\end{align*}

To bound the norm of $\bW_Q$, observe that
\begin{align*}
\|\bW_Q\|^4_{2, 4} &= \frac{4m}{M^2}\sum_{i \in [R]}\sum_{r = (i-1)M + 1}^{iM}\|\bw^i_+(\bw_{0, r})\|_2^4 + \|\bw^i_-(\bw_{0, r})\|_2^4.
\end{align*}
Since we can upper bound
\[
	\|\bw^i_+(\bw_{0, r})\|_2^4 + \|\bw^i_-(\bw_{0, r})\|_2^4 \le a^i(\bw_{0, r})^2 \lesssim d^{2(k-1)},
\]
for fixed $i \in [R]$ by Hoeffding we have that with probability $1 - \frac12d^{-11}$ over the initialization
\[
	\left|\frac{1}{M}\sum_{r=(i-1)M + 1}^{iM}\|\bw^i_+(\bw_{0, r})\|_2^4 + \|\bw^i_-(\bw_{0, r})\|_2^4 - \EE_{\bw_0}\left[\|\bw^i_+\|_2^4 + \|\bw^i_-\|_2^4\right]\right| \le \tilde O\left(\frac{d^{2(k-1)}}{\sqrt{m}}\right).
\]
Union bounding over each $i \in [R]$ and using $M = \Theta(m/R) = \Theta(m)$, with probability $1 - \frac12d^{-10}$ over the initialization we have that
\begin{align*}
	\|\bW_Q\|^4_{2, 4} &\lesssim \sum_{i \in [R]}\EE_{\bw_0}\left[\|\bw^i_+\|_2^4 + \|\bw^i_-\|_2^4\right] + \frac{d^{2(k-1)}}{\sqrt{m}}\\
	&\lesssim d^{k-1} + \frac{d^{2(k-1)}}{\sqrt{m}}\\
	&\lesssim d^{k-1},
\end{align*}
as desired.
\end{proof}

\begin{corollary}\label{cor: quad ntk l-infty bound}
The solution $\bW_Q$ constructed in Lemma~\ref{lemma: quad-ntk expressivity} satisfies
\[
	\|\bW_Q\|_{2, \infty} \lesssim m^{-1/4}d^{\frac{k-1}{2}}.
\]
\end{corollary}
\begin{proof}
From the proof of Lemma~\ref{lemma: quad-ntk expressivity}, either $\|\bw_r\|_2 = 0$, or $(i-1)M + 1 \le r \le iM$ or $m/2 + (i-1)M + 1 \le r \le m/2 + iM$ for some $i$ in which case
\[
	\|\bw_r\|_2 \le 2m^{-1/4}\sqrt{|a^i(\bw_0)|} \le 2m^{-1/4}d^{\frac{k-1}{2}}.
\]
\end{proof}
\subsection{NTK Proofs}
\subsubsection{Symmetric Initialization} Recall the definition of the NTK featurization map
\[
	\varphi(\bx) = \tvec{\nabla_\bW f(\bx; \bW_0)} = \tvec{\{\frac{a_r}{\sqrt{m}}\sigma'(\bw_{0, r}^T\bx)\bx\}_{r \in [m]}\}} \in \RR^{md}
\]
The symmetric initialization makes this different from the NTK features in \cite{montanari2021, montanari2020}, which for width $\tilde m = m/2$ is given by
\[
	\tilde \varphi(\bx) = \tvec{\nabla_\bW f(\bx; \bW_0)} = \tvec{\{\frac{1}{\sqrt{m}}\sigma'(\bw_{0, r}^T\bx)\bx\}_{r \in [\tilde{m}]}\}} \in \RR^{\tilde m d}
\]
These two features are related by
\[
	\varphi(\bx) = \begin{bmatrix} \tilde \varphi(\bx) \\ -\tilde \varphi(\bx) \end{bmatrix}.
\]
For the bulk of this section we consider the features $\tilde \varphi(\bx)$ in order to invoke the results from \cite{montanari2021, montanari2020}.

\subsubsection{Preliminaries}

For arbitary $N \ll md$, let $\cD_N = \{\bx_i\}_{i \in [N]}$ be a dummy dataset of size $N$, where each $\bx_i$ is sampled i.i.d from $\cS^{d-1}(\sqrt{d})$. We define the following random matrices which depend on $\cD_N$.

Denote by
\[
	\bPhi_N = \begin{bmatrix} \tilde \varphi(\bx_1)^T \\ \tilde \varphi(\bx_2)^T \\ \cdots \\ \tilde \varphi(\bx_N)^T \end{bmatrix} \in \RR^{N \times \tilde md}
\]
the feature matrix, and let
\[
	\bK_N = \bPhi_N\bPhi_N^T \in \RR^{N \times N}
\]
be the empirical kernel matrix.

The infinite-width kernel matrix $\bK^\infty_N \in \RR^{N \times N}$ has entries
\[
	\{\bK^\infty_N\}_{i,j} = \EE_\bw[\sigma'(\bw^T\bx_i)\sigma'(\bw^T\bx_j)\bx_i^T\bx_j].
\]

Also, define $\bSigma_N \in \RR^{\tilde md \times \tilde md}$ to be the empirical covariance matrix, so that
\[
	\bSigma_N = \frac{1}{N}\bPhi_N^T\bPhi_N = \frac{1}{N}\sum_{i=1}^N \tilde \varphi(\bx_i)\tilde \varphi(\bx_i)^T,
\]
and let
\[
	\tilde \bSigma = \EE_\mu\left[\tilde \varphi(\bx)\tilde \varphi(\bx)^T\right] \in \RR^{\tilde md \times \tilde md}
\]
be the population covariance matrix.

We let $\sigma'$ satisfy assumption \ref{assume: hermite coeffs}. Along with the boundedness of $\sigma'$, this allows us to invoke the following lemmas from~\cite{montanari2020}:\\

\begin{lemma}[\cite{montanari2020}, Theorem 3.2]\label{lemma: infinite width kernel concentration}
With probability $1 - d^{-11}$,
\[
\|\{\bK^\infty_N\}^{-\frac12}\bK_N{\bK^\infty_N}^{-\frac12} - \bI_N\|_{op} \le \Tilde O\left(\sqrt{\frac{N}{\tilde md}} + \frac{N}{\tilde md}\right)
\]
\end{lemma}

As in \cite{montanari2020}, let $\bPsi_{\le \ell} \in \RR^{N \times n_\ell}$ be the evaluations of the degree $\le \ell$ spherical harmonics on the $N$ data points.\\

\begin{lemma}[\cite{montanari2020}, Lemma 2]\label{lemma:kernel decomp}
Let $d^{\ell}\log^2 d \ll N \ll d^{\ell + 1}/\log d^C$. With probability $1 - d^{-11}$ the infinite-width kernel matrix can be decomposed as
\[
	\frac{1}{d}\bK^\infty_N := \gamma_{> \ell}\bI_N + \bPsi_{\le \ell}\bLambda_{\le \ell}^2\bPsi_{\le \ell}^T + \bDelta,
\]
where $\gamma_k \ge 0$ is a sequence satisfying 
\begin{align}
\gamma_0 = d^{-1}(1 + o_d(1))\mu_1^2(\sigma'), \quad \gamma_k = \mu_{k-1}^2(\sigma') + o_d(1)~\text{for}~k \ge 1, \quad \gamma_{>\ell} := \sum_{k' > \ell}\gamma_k
\end{align} and $\bLambda_{\le \ell}^2$ is a diagonal matrix where $B(d, k)^{-1}\gamma_k$ has multiplicity $B(d, k)$, for $k \le \ell$.
Furthermore, the remainder $\bDelta$ satisfies
\[
	\|\bDelta\|_{op} \le \tilde O \left(\sqrt{\frac{N}{d^{\ell + 1}}}\right),
\]
and the spherical harmonic features $\bPsi_{\le \ell}$ satisfy
\[
	\left\|\frac{1}{N}\bPsi_{\le \ell}^T\bPsi_{\le \ell} - \bI_{n_\ell} \right\|_{op} \le \tilde O\left(\sqrt{\frac{d^\ell}{N}}\right).
\]
\end{lemma}
By assumption~\ref{assume: hermite coeffs}, $\gamma_0 = \Theta(d^{-1})$ and $\gamma_{\ell'} = \Theta(1)$ for $\ell' \ge 1$.

The following Lemma gives the eigenvalues of the empirical kernel matrix $\bK_N$ and the infinite-width kernel matrix $\bK_N^\infty$:\\
\begin{lemma}[Follows from \cite{montanari2020}, Lemma 6]\label{lemma:kernel evals} With probability $1 - d^{-11}$, the following all hold:

For $1 \le \ell' \le \ell$, $n_{\ell'-1} < i \le n_{\ell'}$,
\[
	\lambda_i(\bK^\infty_N), \lambda_i(\bK_N) = \Theta(N\cdot d^{1 - \ell'}).
\]
Additionally,
\[
	\lambda_1(\bK^\infty_N), \lambda_1(\bK_N) = \Theta(N).
\]
Finally, for $i > n_\ell$,
\[
	\lambda_i(\bK^\infty_N), \lambda_i(\bK_N) = \Theta(d).
\]
\end{lemma}

We also use the following classical results throughout the proofs in this section.\\

\begin{lemma}[Weyl's Inequality] For two psd matrices $\bA_1, \bA_2 \in \RR^{p \times p}$,
\begin{equation}
\left|\lambda_i(\bA_1) - \lambda_i(\bA_2)\right| \le \|\bA_1 - \bA_2\|_{op}
\end{equation}
for all $i \in [p]$.
\end{lemma}
~
\begin{definition}[\cite{spectral_methods}]
For $r \le p$, let $\bU_1, \bU_2 \in \RR^{p \times r}$ both have orthonormal columns. Then the distance between the subspaces spanned by $\bU_1, \bU_2$ is
\begin{align}
\mathrm{dist}(\bU_1, \bU_2) := \min_{\bO \in \mathcal{O}^{r \times r}}\left\|\bU_1\bO - \bU_2\right\|_{op},
\end{align}
where $\mathcal{O}^{r \times r}$ is the space of $r \times r$ orthogonal matrices.
\end{definition}
~
\begin{lemma}[Davis-Kahan sin-$\theta$ theorem~\cite{daviskahan, spectral_methods}]
For two psd matrices $\bA_1, \bA_2 \in \RR^{p \times p}$, let $\bA_1 = \bU_1\bLambda_1\bU_1^T$ and $\bA_2 = \bU_2\bLambda_2\bU_2^T$ be their eigendecompositions (sorted by decreasing eigenvalues), and let $\bU_1 = \begin{bmatrix} \bU_{1, \le r}& \bU_{1, > r} \end{bmatrix}$, $\bU_2 = \begin{bmatrix} \bU_{2, \le r}& \bU_{2, > r} \end{bmatrix}$ be their eigenvectors, where $\bU_{1, \le r}, \bU_{2, \le r} \in \RR^{p \times r}$. Furthermore, assume $\|\bA_1 - \bA_2\|_{op} < (1 - 1/\sqrt{2})(\lambda_r(\bA_1) - \lambda_{r+1}(\bA_1))$. Then
\begin{equation}
\mathrm{dist}(\bU_{1, \le r}, \bU_{2, \le r}) \le \frac{2\|\bA_1 - \bA_2\|_{op}}{\lambda_r(\bA_1) - \lambda_{r+1}(\bA_1)}
\end{equation}
\end{lemma}

\subsubsection{Eigenvector Lemmas}

Throughout this section, we condition on the event where Lemmas~\ref{lemma: infinite width kernel concentration}, \ref{lemma:kernel decomp}, \ref{lemma:kernel evals} are all true.\\

\begin{lemma}\label{lemma: spherical harmonics to evecs}
Define
\[
	\Tilde \bK = d\gamma_{> k}\bI_N + d\bPsi_{\le \ell}\bLambda_{\le \ell}^2\bPsi_{\le \ell}^T,
\]
and let $\Tilde \bK$ have eigendecomposition $\Tilde \bU {\Tilde \bLambda}^2 {\Tilde \bU}^T$. Define $\bPsi_{\le k} \in \RR^{N \times n_k}$ to be the first $n_k$ columns of $\bPsi_{\le k}$ (where $k < \ell$), and also let $\tilde \bU_{\le k}$ be the first $n_k$ columns of $\tilde \bU$, with $\tilde \bU_{>k}$ the remaining columns. Then
\[
	\left\|\frac{1}{\sqrt{N}}\bPsi_{\le k}^T\tilde\bU_{>k}\right\|_{op} \lesssim \frac{1}{\sqrt{d}}.
\]
\end{lemma}
\begin{proof}
Let $\bPsi_{k:\ell} \in \RR^{N \times n_\ell - n_k}$ be the $(n_k + 1)$-th to $n_\ell$-th columns. We can then decompose
\[
	\tilde \bK = d\bPsi_{\le k}\bLambda_{\le k}^2\bPsi_{\le k}^T + d\bPsi_{k:\ell}\bLambda_{k:\ell}^2\bPsi_{k:\ell}^T + d\gamma_{> k}\bI_N
\]
and
\[
	\tilde{\bK} = \begin{bmatrix} \tilde \bU_{\le k} & \tilde \bU_{>k} \end{bmatrix}\text{diag}(\tilde \bLambda_{\le k}, \tilde \bLambda_{>k})\begin{bmatrix} \tilde \bU_{\le k} & \tilde \bU_{>k} \end{bmatrix}^T,
\]
where $\tilde \bU_{\le k} \in \RR^{N \times n_k}$. Let $\bu \in \RR^{N - n_k}$ be a vector such that $\|\bu\|_2 = 1$ and $\|\bPsi_{\le k}^T\tilde \bU_{>k}\bu\|_2 = \|\bPsi_{\le k}^T\tilde\bU_{>k}\|_{op}$. Finally, define $\tilde \bu = \bPsi_{\le k}^T\tilde\bU_{>k}\bu$. We then have:
\begin{align*}
\bu^T\tilde\bLambda_{>k}\bu &= \bu^T\tilde\bU_{>k}^T\tilde \bK\tilde\bU_{>k}\bu\\
	&\ge d\bu^T\tilde\bU_{>k}^T\bPsi_{\le k}\bLambda_{\le k}^2\bPsi_{\le k}^T\tilde\bU_{>k}\bu + d\gamma_{>k}.\\
	&= d\tilde \bu^T \bLambda_{\le k}^2 \tilde \bu + d\gamma_{>k}\\
	&\gtrsim d\cdot d^{-k}\|\tilde \bu\|_2^2 + d\gamma_{>k},
\end{align*}
since $\lambda_{min}(\bLambda_{\le k}^2) = \Theta(d^{-k}).$

Define
\[
	\tilde \bK_{\le k} = d\bPsi_{\le k}\bLambda_{\le k}^2\bPsi_{\le k}^T,
\]
and define $\tilde \bK_{>k} = \tilde \bK - \tilde \bK_{\le k}$. By Weyl's inequality,
\[
	\left|\lambda_i(\bK_{\le k}) - \lambda_i(\tilde \bK)\right| \le \|\tilde \bK_{>k}\|_{op}.
\]
For $i > n_k$, $\lambda_i(\bK_{\le k}) = 0$, in which case
\[
	\lambda_i(\tilde \bK) \le \|\tilde \bK_{>k}\|_{op} \lesssim dN \cdot d^{-k - 1}.
\]
Therefore we can upper bound
\[
	\bu^T\tilde\bLambda_{>k}\bu \lesssim dN\cdot d^{-k-1}.
\]
Altogether, we have
\[
	dN\cdot d^{-k-1} \gtrsim d\cdot d^{-k}\|\tilde \bu\|_2^2 + d\gamma_{>k},
\]
which yields
\[
	\frac{1}{\sqrt{N}}\|\tilde \bu\|_2 \lesssim \frac{1}{\sqrt{d}}.
\]
By definition, $\|\tilde \bu\|_2 = \|\bPsi_{\le k}^T\tilde\bU_{>k}\|_{op}$, so
\[
	\left\|\frac{1}{\sqrt{N}}\bPsi_{\le k}^T\tilde\bU_{>k}\right\|_{op} \lesssim \frac{1}{\sqrt{d}},
\]
as desired. 
\end{proof}




\begin{lemma}\label{lemma: evecs to infinite width}
Let $\bK^\infty_N$ have eigendecomposition $\bK^\infty_N = \overline \bU {\overline \bLambda}^2 {\overline \bU}^T$, where $\overline \bU_{\le k}, \overline \bU_{> k}$ are the first $n_k$, remaining $N - n_k$ columns of $\overline \bU$ respectively. Then
\[
	\mathrm{dist}\left(\tilde \bU_{\le k}, \overline \bU_{\le k} \right) \le \tilde{O}\left(\sqrt{\frac{d^{2k-\ell -1}}{N}}\right)
\]
\end{lemma}
\begin{proof}
By Lemma \ref{lemma:kernel evals}, $\lambda_{n_k}(\bK^\infty_N) - \lambda_{n_k + 1}(\bK^\infty_N) \gtrsim dN\cdot d^{-k}$. Since 
\[\|\bK^\infty_N - \tilde \bK\|_{op} = d\|\bDelta\|_{op} \le \tilde O\left(\sqrt{\frac{N}{d^{\ell - 1}}}\right) \ll \lambda_{n_k}(\bK^\infty_N) - \lambda_{n_k + 1}(\bK^\infty_N)\]
by Davis-Kahan we have
\[
	\mathrm{dist}\left(\tilde \bU_{\le k}, \overline \bU_{\le k} \right) \lesssim \frac{d\|\bDelta\|_{op}}{dN\cdot d^{-k}} = \tilde{O}\left(\sqrt{\frac{d^{2k - \ell - 1}}{N}}\right).
\]
\end{proof}


\begin{lemma}\label{lemma: infinite width to finite width}
Let $\bK_N$ have eigendecomposition $\bK_N = \bU_N \bLambda_N^2 \bU_N^T$. Then
\[
	\mathrm{dist}\left(\overline \bU_{\le k}, \bU_{N, \le k}\right) = \tilde{O}\left(\sqrt{\frac{Nd^{2k-3}}{\tilde m}}\right)
\]
\end{lemma}
\begin{proof}
By Lemma~\ref{lemma: infinite width kernel concentration}, we can write
\[
	{\bK^\infty_N}^{-1/2}\bK_N{\bK^\infty_N}^{-1/2} = \bI_N + \bDelta_N,
\]
where $\bDelta_N \le \tilde O\left(\sqrt{\frac{N}{\tilde md}}\right)$. Rearranging, we get
\[
	\bK_N = {\bK^\infty_N} + {\bK^\infty_N}^{1/2}\bDelta_N{\bK^\infty_N}^{1/2},
\]
where we can bound
\[
	\|{\bK^\infty_N}^{1/2}\bDelta_N{\bK^\infty_N}^{1/2}\|_{op} \le \|\bDelta_N\|_{op} \|{\bK^\infty_N}\|_{op} \lesssim \sqrt{\frac{N^{5/2}}{\tilde md}},
\]
since $\|\bK_N^\infty\|_{op} \lesssim N$. Since $\bK^\infty_N$ has eigengap $\Theta(N\cdot d^{1-k})$ we can again apply Davis-Kahan to get
\[
	\mathrm{dist}\left(\overline \bU_{\le k}, \bU_{N, \le k}\right) \lesssim \frac{\|{\bK^\infty_N}^{1/2}\bDelta_n{\bK^\infty_N}^{1/2}\|_{op}}{N\cdot d^{1-k}} = \tilde O\left(\sqrt{\frac{Nd^{2k - 3}}{\tilde m}}\right).
\]
\end{proof}



\begin{lemma}\label{lemma: right singular vectors match}
Let $\bSigma_N$ have eigendecomposition $\bSigma_N = \bV_N \frac{\bLambda^2_N}{N} \bV_N^T$, where $\bV_N \in \RR^{md \times N}$. Let $\tilde \bSigma$ have eigendecomposition $\tilde \bSigma = \bV \bLambda^2 \bV^T$. Let $\bV_N = \begin{bmatrix} \bV_{N, \le k} & \bV_{N, >k} \end{bmatrix}$, $\bV = \begin{bmatrix} \bV_{\le k} & \bV_{>k} \end{bmatrix}$ where $\bV_{N, \le k}, \bV_{\le k} \in \RR^{md \times n_k}$. Then, with probability $1 - n^{-11}$,
\[
	\mathrm{dist}\left(\bV_{N, \le k}, \bV_{\le k}\right) \le \Tilde O \left(\sqrt{\frac{d^{2k}}{N}}\right).
\]
\end{lemma}
\begin{proof}
By \cite[5.6.4]{vershynin}, with probability $1 - n^{-11}$ we have
\[
	\|\bSigma_N - \tilde \bSigma\|_{op} \le \Tilde O \left(\sqrt{\frac{r}{N}}\right)\cdot \|\tilde \bSigma\|_{op},
\]
where $r = \text{Tr}(\tilde \bSigma)/\|\tilde \bSigma\|_{op}$ is the effective rank. We can upper bound $\text{Tr}(\tilde \bSigma) \le d$, and thus
\[
	\|\bSigma_N - \tilde \bSigma\|_{op} \le \Tilde O\left(\sqrt{\frac{d}{N}\|\tilde \bSigma\|_{op}}\right) = \Tilde O\left(\sqrt{\frac{d^2}{N}}\right),
\]
since $\|\tilde \bSigma\|_{op} \le d$. By Lemma \ref{lemma:kernel evals},
\[
	\lambda_{n_k}(\bSigma_N) - \lambda_{n_k + 1}(\bSigma_N) = \frac{1}{N}\cdot \Theta(N\cdot d^{1 - k}) = \Theta(d^{1 - k}).
\]
Therefore by Davis-Kahan, we can bound
\[
	\mathrm{dist}\left(\bV_{N, \le k}, \bV_{\le k}\right) \lesssim \cdot \frac{\|\bSigma_N - \tilde \bSigma\|_{op}}{d^{-k + 1}} \lesssim \sqrt{\frac{d^{2k}}{N}}.
\]
\end{proof}

The following lemma is a consequence of the preceeding eigenstructure and matrix perturbation lemmas, and partitions the eigenvectors of $\bSigma$ into groups corresponding to large, medium, and small eigenvalues.\\
\begin{lemma}\label{lemma: main intermediate spectral result}
Let $N \ge d^{4k}, m \ge N^{5/2}$. For $1 \le k' \le 2k$, $n_{k' - 1} < i \le n_{k'}$, we have $\lambda_i(\bSigma) = \Theta(d^{1 - k'})$. Also, for $i > n_{2k}$, $\lambda_i(\bSigma) \ll \lambda_{n_{2k}}(\bSigma)$.
\end{lemma}
\begin{proof}
Since
\[
	\|\bSigma_N - \tilde\bSigma\|_{op} \lesssim \sqrt{\frac{d^2}{N}},
\]
by Weyl's inequality we have
\[
	|\lambda_i(\bSigma_N) - \lambda_i(\tilde \bSigma)| \lesssim \sqrt{\frac{d^2}{N}}
\]
Since for $n_{k' - 1} < i \le n_{k'}$ we have $\lambda_i(\bSigma_N) = N^{-1}\lambda_i(\bK_N) = \Theta(d^{1 - k'}) \gg \sqrt{\frac{d^2}{N}}$, since $N \ge d^{4k}$; therefore $\lambda_i(\bSigma) = \Theta(d^{1 - k'})$. Furthermore, for $i > n_{2k}$,
\[
	\lambda_i(\bSigma) \lesssim \sqrt{\frac{d^2}{N}} + \lambda_i(\bSigma_N) \ll \Theta(d^{1 - 2k}).
\]
As a consequence, we can write the eigendecomposition of $\bSigma$ as
\begin{align*}
\bSigma = \begin{bmatrix} \bQ_1 & \bQ_2 &\bQ_3 \end{bmatrix} \begin{bmatrix} \bLambda_1 & \bZero & \bZero \\ \bZero & \bLambda_2 &\bZero \\ \bZero & \bZero & \bLambda_3 \end{bmatrix} \begin{bmatrix} \bQ_1^T \\ \bQ_2^T \\ \bQ_3^T\end{bmatrix},
\end{align*}
where $\bQ_1 \in \RR^{md \times n_k}$ are the large eigenvectors, $\bQ_2 \in \RR^{md \times n_{2k} - n_k}$ are the medium eigenvectors, and $\bQ_3 \in \RR^{md \times md - n_{2k}}$ are the small eigenvectors; concretely,
\[
\lambda_{min}(\bLambda_1) = \Theta(d^{1 - k}) \gg \Theta(d^{-k}) = \lambda_{max}(\bLambda_2)
\]
and
\[
\lambda_{min}(\bLambda_2) = \Theta(d^{1 - 2k}) \gg \lambda_{max}(\bLambda_3).
\]
\end{proof}

\subsubsection{Proof of Lemma~\ref{lemma: NTK fit low degree}}\label{sec: prove NTK fit lemma}
\begin{proof}
We show this Lemma holds with probability $1 - d^{-10}$. Condition on the events where Lemmas~\ref{lemma: infinite width kernel concentration}, \ref{lemma:kernel decomp}, \ref{lemma:kernel evals}, \ref{lemma: right singular vectors match} hold.

Pick $N$ so that $d^\ell \ll N \ll d^{\ell + 1}$ for $\ell = 4k$, and choose $\tilde m = N^{5/2} $. We form a dataset of $N$ samples by adding another $N - n$ samples i.i.d from $\cS^{d-1}(\sqrt{d})$. Let $\by_N \in \RR^{N}$ be the vector where $\by_{N, i} = f_k(\bx_i)$ where $i \in [N]$. Recall that $\bPsi_{\le k} \in \RR^{N \times n_k}$ denotes the evaluations of the degree $\le k$ spherical harmonics on the $N$ data points. Since $f_k(\bx)$ is a degree $\le k$ polynomial, we can orthogonally decompose it in terms of the degree $\le k$ spherical harmonics, i.e
\[
	f_k(\bx) = \sum_{k'=0}^k\sum_{t=1}^{B(d, k)}w_{k't}Y_{k't}(\bx),
\] 
where $\sum_{k'=0}^k\sum_{t=1}^{B(d, k)}w_{k't}^2 = 1$. Therefore we have
\[
	\by_N = \bPsi_{\le k}\bw^*
\]
for $\bw^* \in \RR^{n_k}$ where $\|\bw^*\|_2 = 1$.

Observe that $\bPhi_N$ has SVD $\bPhi_N = \bU_N\bLambda_N\bV_N^T$. Define $\tilde \bw^* = \tilde\bU_{\le k}^T\bPsi_{\le k}\bw^*$. Then
\begin{align*}
 \tilde\bU_{\le k}\tilde \bw^* &= \tilde\bU_{\le k}\tilde\bU_{\le k}^T\bPsi_{\le k} \bw^*\\
 &= (\bI_{n_k} - \tilde\bU_{> k}\tilde\bU_{>k}^T)\bPsi_{\le k}\bw^*\\
 &= \by_N - \tilde\bU_{> k}\tilde\bU_{>k}^T\bPsi_{\le k}\bw^*,
\end{align*}
and thus
\begin{align*}
\|\by_N - \tilde\bU_{\le k}\tilde \bw^*\| &\le \|\tilde\bU_{> k}\tilde\bU_{>k}^T\bPsi_{\le k}\bw^*\|\\
&\le \|\tilde\bU_{>k}^T\bPsi_{\le k}\|_{op}\|\bw^*\|\\
&\lesssim \sqrt{\frac{N}{d}},
\end{align*}
by Lemma~\ref{lemma: spherical harmonics to evecs}. Also, 
\[
	\|\tilde \bw^*\| \le \|\bPsi_{\le k}\|_{op}\|\bw^*\| \lesssim \sqrt{N}.
\]

By Lemmas~\ref{lemma: evecs to infinite width} and\ref{lemma: infinite width to finite width}, we have
\[
	\mathrm{dist}\left(\tilde \bU_{\le k}, \bU_{N, \le k}\right) \lesssim \frac{1}{\sqrt{d}}.
\]
Thus there exists an orthogonal $\bO_1$ such that $\left\|\tilde\bU_{\le k} - \bU_{N, \le k}\bO_1\right\|_{op} = \mathrm{dist}\left(\tilde \bU_{\le k}, \bU_{N, \le k}\right)$ and hence
\begin{align*}
	\|\by_N - \bU_{N, \le k}\bO_1\tilde\bw^*\| &\le \|\by_N - \tilde \bU_{\le k}\tilde\bw^*\| + \left\|\tilde\bU_{\le k} - \bU_{N, \le k}\bO_1\right\|_{op}\|\tilde \bw^*\|\\
	&\lesssim \sqrt{\frac{N}{d}}.
\end{align*}
Since $\bU_{N, \le k} = \bPhi_N\bV_{N, \le k}\bLambda_{N, \le k}^{-1}$, we have
\[
	\|\by_N - \bPhi_N\bV_{N, \le k}\bLambda_{N, \le k}^{-1}\bO_1\tilde\bw^*\| \lesssim \sqrt{\frac{N}{d}}
\]
Finally, by Lemma~\ref{lemma: right singular vectors match}, we have 
\[
\mathrm{dist}(\bV_{N, \le k}, \bV_{\le k}) \lesssim \sqrt{\frac{d^{2k}}{N}} \le d^{-k},
\]
and thus there exists an orthogonal $\bO_2$ such that $\|\bV_{N, \le k} - \bV_{\le k}\bO_2\|_{op} \le d^{-k}$ and hence
\begin{align*}
	\|\by_N - \bPhi_N\bV_{\le k}\bO_2\bLambda_{N, \le k}^{-1}\bO_1\tilde\bw^*\| &\le \|\by_N - \bPhi_N\bV_{N, \le k}\bLambda_{N, \le k}^{-1}\bO_1\tilde\bw^*\| + \left\|\bPhi_N(\bV_{N, \le k}-\bV_{\le k}\bO_2)\bLambda_{N, \le k}^{-1}\bO_1\tilde\bw^*\right\|\\
	&\lesssim \sqrt{\frac{N}{d}} + \|\bPhi_N\|_{op}\|\bV_{N, \le k} - \bV_{\le k}\bO_2\|_{op}\|\bLambda_{N, \le k}^{-1}\|_{op}\|\tilde \bw^*\|\\
	 &\lesssim  \sqrt{\frac{N}{d}} + \sqrt{N}\cdot d^{-k}\cdot \sqrt{N^{-1}d^{k-1}}\cdot\sqrt{N}\\
	 &\lesssim \sqrt{\frac{N}{d}}.
\end{align*}
Therefore, letting $\bv^* = \bV_{\le k}\bO_2\bLambda_{N, \le k}^{-1}\bO_1\tilde\bw^*$
\[
	\frac{1}{N}\|\by_N - \bPhi_N\bv^*\|_2^2 \lesssim \frac{1}{d}.
\]
Note that $\bv^*$ satisfies
\[
	\|\bv^*\|^2_2 \le \|\bLambda_{N, \le k}^{-1}\|_{op}^2\|\bar\bw^*\|^2 = O\left(d^{k - 1}\right).
\]
Since the $\{\bx_i\}_{i \in [N]}$ are i.i.d, we can treat the dataset $\cD$ as a subsample of $n$ data points. Since $|f_k(\bx)|^2 \le n_k = \Theta(d^k)$, and $|\tilde \varphi(\bx)^T\bv^*|^2 \le \|\tilde\varphi(\bx)\|^2\|\bv^*\|^2 = \Theta(d^k)$, we can bound $(f_k(\bx) - \tilde \varphi(\bx)^T\bv^*)^2 \le d^k$. Also,
\[
	\EE_N\left[(f_k(\bx) - \tilde \varphi(\bx)^T\bv^*)^4\right] \lesssim d^k\EE_N\left[(f_k(\bx) - \tilde \varphi(\bx)^T\bv^*)^2\right] \lesssim d^{k-1}.
\]
Therefore by Bernstein's Inequality, with probability $1 - d^{-11}$, we can bound
\[
	\EE_n\left[(f_k(\bx) - \tilde \varphi(\bx)^T\bv^*)^2\right] \lesssim \frac{1}{d} + \sqrt{\frac{d^{k-1}}{n}} + \frac{d^k}{n} \lesssim \frac{d^k}{n}.
\]

To conclude we must relate $\tilde \varphi$ to $\varphi$. Observe that
\[
	\bSigma = \begin{bmatrix} \tilde \bSigma & -\tilde \bSigma \\ -\tilde \bSigma & \tilde \bSigma \end{bmatrix}.
\]
Therefore $\lambda_i(\bSigma) = 2\lambda_i(\tilde \bSigma)$ for $i \le \tilde m$, and $\lambda_i(\bSigma) = 0$ otherwise. Furthermore, if $\bv$ is an eigenvector of $\tilde \bSigma$, then $\begin{bmatrix} \bv \\ -\bv \end{bmatrix}$ is an eigenvector of $\bSigma$. As a result, if $\bu \in \text{span}(\{\bv_i(\tilde \bSigma)\}_{i \in [n_k]})$, then $\begin{bmatrix} \bu \\ -\bu \end{bmatrix} \in \text{span}(\bP_{\le k})$. Letting $\bz^* = \frac12\begin{bmatrix} \bv^* \\ -\bv^* \end{bmatrix}$, we get that $\bz^* \in \text{span}(\bP_{\le k})$, and also $\varphi(\bx)^T\bz^* = \tilde \varphi(\bx)^T\bv^*$. Therefore
\[
	\EE_n\left[(f_k(\bx) - \varphi(\bx)^T\bz^*)^2\right] \lesssim \frac{d^k}{n},
\]
and also $\|\bz^*\|_2^2 \lesssim d^{k-1}$. Union bounding over all high probability events, this holds with probability $1 - 4d^{-11} \ge 1 - d^{-10}$, as desired.
\end{proof}
\subsection{Proof of Theorem~\ref{thm: main expressivity}}\label{sec: main expressivity proof}
\begin{proof}
We condition on the events of Lemmas~\ref{lemma: quad-ntk expressivity}, \ref{lemma: NTK fit low degree}, \ref{lemma: empirical f_k concentrate}, \ref{lemma: empirical cov concentrate}, \ref{lemma: NTK features concentrate} holding, which occurs with probability $1 - 5d^{-10} \ge 1 - d^{-9}$.

We proceed by the probabilistic method. For $r \in [m]$, let the $\sigma_r$ be random variables with $\sigma_r \sim \text{Unif}(\{\pm 1\})$ i.i.d, and let $\bS = \text{diag}(\sigma_1, \dots, \sigma_m)$ be the diagonal matrix of random signs. Let $\bW^*_\bS = \bW_L + \bW_Q\bS$ be a (random) solution. We will show that in expectation over $\bS$ the training error and all the regularizers are small, which implies the existence of a $\bS$ which makes all these quantities small.

\textbf{Bounding the Empirical Loss.}
First, observe that we can write
\begin{align}
&\EE_n\left|f_L(\bx; \bW^*_\bS) + f_Q(\bx; \bW^*_\bS) - f_L(\bx; \bW_L) - f_Q(\bx; \bW_Q)\right|\\
 \le~&\mathbb{E}_n\left|f_L(\bx; \bW_Q\bS)\right| + \mathbb{E}_\mathcal{D}\left|f_Q(\bx; \bW^*_S) - f_Q(\bx; \bW_Q)\right|.
\end{align}

Since $f_Q(\bx; \bW_Q) = f_Q(\bx; \bW_Q\bS)$, the second term can be deterministically bounded as
\begin{align*}
&\EE_n\left|f_Q(\bx; \bW_L + \bW_Q\bS ) - f_Q(\bx; \bW_Q\bS)\right|\\
\le~&\frac{1}{2\sqrt{m}}\sum_{r=1}^m \EE_n\left|\sigma''(\bw_{0, r}^T\bx)\left((\sigma_r(\bw_Q)_r^T\bx + (\bw_L)_r^T\bx)^2 - ((\bw_Q)_r^T\bx)^2\right) \right|\\
\le~&\frac{1}{2\sqrt{m}}\sum_{r=1}^m \EE_n\left|(\bw_L)_r^T\bx\bx^T(2\sigma_r(\bw_Q)_r + (\bw_L)_r) \right|\\
\le~&\frac{d}{\sqrt{m}}\sum_{r=1}^m\|(\bw_L)_r\|_2(\|(\bw_L)_r\|_2 + \|(\bw_Q)_r\|_2)\\
\le~&\frac{d}{\sqrt{m}}(\|\bW_L\|_F^2 + \|\bW_L\|_F\|\bW_Q\|_F)\\
\lesssim~& m^{-\frac14}d^{\frac{3k+1}{4}}.
\end{align*}

We next bound the first term in expectation using Lemma~\ref{lemma: linear term small in expectation}:
\begin{align*}
\EE_\bS \EE_n \left|f_L(\bx; \bW_Q\bS)\right| &\le  \left(\EE_\bS \EE_n \left(f_L(\bx; \bW_Q\bS)\right)^2\right)^{1/2}\\
&\le \frac{1}{\sqrt{m}}\|\bW_Q\|_F\\
&\lesssim m^{-\frac14}d^{\frac{k-1}{4}}.
\end{align*}

Since the loss is Lipschitz, we can bound the empirical loss as
\begin{align*}
\EE_\bS\hat{L}^Q(\bW^*_S)=~&\EE_\bS\EE_n[\ell(f^*(\bx), f_Q(\bx; \bW^*_\bS) + f_L(\bx; \bW^*_\bS))]\\
\le~&\EE_\bS\EE_n\left|f^*(\bx) -  f_Q(\bx; \bW^*_\bS) - f_L(\bx; \bW^*_\bS)\right|\\
\le~&\EE_\bS\EE_n|f_{sp}(\bx) - f_Q(\bx; \bW_Q)|\\
&+\EE_\bS\EE_n|f_k(\bx) - f_L(\bx; \bW_L)|\\
&+\EE_\bS\EE_n|f_L(\bx; \bW^*_\bS) + f_Q(\bx; \bW^*_\bS) - f_L(\bx; \bW_L) - f_Q(\bx; \bW_Q)|\\
\lesssim~& \frac{d^k}{\sqrt{m}} + \sqrt{\frac{d^k}{n}} + m^{-\frac14}d^{\frac{3k+1}{4}}\\
\le~& \varepsilon_{min},
\end{align*}
where we used the bounds in Lemmas~\ref{lemma: quad-ntk expressivity}, \ref{lemma: NTK fit low degree}, along with the lower bounds on $m$ in the assumption of the theorem.

\textbf{Bounding the Regularizers.}
We begin with $\cR_1$:
\begin{align*}
\cR_1(\bW^*_S) &= \|f_L(\bx; \bP_{>k}\bW_L + \bP_{>k}\bW_Q\bS)\|^2_{L^2}\\
&\le \|f_L(\bx; \bP_{>k}\bW_Q\bS)\|_{L^2}^2\\
&\le \|f_L(\bx; \bW_Q\bS)\|_{L^2}^2,
\end{align*}
since our construction guarantees that $\bP_{>k}\bW_L = \mathbf{0}$.

By Lemma~\ref{lemma: linear term small in expectation}, 
\[
	\EE_\bS\|f_L(\bx; \bW_Q\bS)\|_{L^2}^2 \le \frac{1}{m}\|\bW_Q\|_F^2 \lesssim m^{-\frac12}d^{\frac{k-1}{2}}.
\]
Therefore
\[
	\EE_\bS\cR_1(\bW^*_\bS) \lesssim m^{-\frac12}d^{\frac{k-1}{2}}.
\]

$\cR_2$ can be bounded as
\begin{align*}
\cR_2(\bW^*_\bS) &= \|f_L(\cdot; \bP_{\le k}\bW^*_\bS)\|_{L^2}^2\\
&\le\|f_L(\cdot; \bW^*_\bS)\|_{L^2}^2\\
&\lesssim \|f_L(\cdot; \bW_L)\|_{L^2}^2 + \|f_L(\cdot; \bW_Q\bS)\|_{L^2}^2
\end{align*}
By Lemma~\ref{lemma: empirical to pop small directions},
\begin{align*}
\|f_L(\cdot; \bW_L)\|_{L^2}^2 &\lesssim \EE_n\left[(f_L(\bx; \bW_L))^2\right]\\
&\lesssim \EE_n\left[f_k(\bx)^2\right] + \EE_n\left[(f_L(\bx; \bW_L) - f_k(\bx))^2\right]\\
&\lesssim 1 + \frac{d^k}{n} \lesssim 1,
\end{align*}
where the last step uses Lemma~\ref{lemma: empirical f_k concentrate}. Therefore, again applying Lemma~\ref{lemma: linear term small in expectation}, 
\[
	\EE_\bS\cR_2(\bW^*_\bS) \lesssim 1 + \EE_\bS\|f_L(\cdot; \bW_Q\bS)\|_{L^2}^2 \lesssim 1 + m^{-\frac12}d^{\frac{k-1}{2}} \lesssim 1.
\]

$\cR_3$ can be bounded as
\begin{align*}
\cR_3(\bW^*_\bS) &= \mathbb{E}_n\left[(f_L(x; \bP_{>k}\bW_L + \bP_{>k}\bW_Q\bS)^2\right]\\
&= \mathbb{E}_n\left[(f_L(x; \bP_{>k}\bW_Q\bS))^2\right]\\
&\lesssim \mathbb{E}_n\left[(f_L(x; \bW_Q\bS))^2\right] + \mathbb{E}_n\left[(f_L(x; P_{\le k}\bW_Q\bS))^2\right],
\end{align*}
since our construction guarantees $\bP_{>k}\bW_L = \mathbf{0}$.

By Lemma~\ref{lemma: linear term small in expectation}, 
\[\mathbb{E}_\bS \mathbb{E}_n\left[(f_L(\bx; \bW_Q\bS))^2\right] \le m^{-\frac12}d^{\frac{k-1}{2}}.\]
By Lemma~\ref{lemma: empirical to pop small directions},
\[
	\mathbb{E}_n\left[(f_L(x; \bP_{\le k}\bW_Q\bS))^2\right] \le 2\|f_L(\bx; \bP_{\le k}\bW_Q\bS)\|_{L^2}^2,
\]
which can be upper bounded by $2\|f_L(\bx; \bW_Q\bS)\|_{L^2}^2$. Therefore
\[
	\EE_\bS\mathbb{E}_n\left[(f_L(x; \bP_{\le k}\bW_Q\bS))^2\right] \lesssim \EE_\bS\|f_L(x; \bW_Q\bS)\|_{L^2}^2 \lesssim m^{-\frac12}d^{\frac{k-1}{2}}.
\]
Altogether,
\[
	\EE_\bS\cR_3(\bW^*_\bS) \lesssim m^{-\frac12}d^{\frac{k-1}{2}}.
\]

Finally, to bound $\cR_4$, observe that
\[
\|\bW^*_\bS\|_{2, 4} \le \|\bW_L\|_{2, 4} + \|\bW_Q\bS\|_{2, 4} = \|\bW_L\|_{2, 4} + \|\bW_Q\|_{2, 4}.
\]
By the construction $\|\bW_Q\|_{2, 4} \le d^{\frac{k-1}{4}}$. Also, since $\bW_L \in \text{span}(\bP_{\le k})$, by Lemma~\ref{lemma: linear sol small l-infty} we have 
\[
	\|\bW_L\|_{2, 4} \le m^{-\frac14}d^{\frac{k}{2}}\|\bW_L\|_F \lesssim m^{-\frac14}d^{k - \frac12}.
\]
Therefore $\|\bW^*_\bS\|_{2, 4} \lesssim d^{\frac{k-1}{4}}$, and thus
\[
	\cR_4(\bW^*_\bS) \lesssim d^{2(k-1)}.
\]

\end{proof}

\begin{corollary}\label{cor: bound frob norm}
The solution $\bW^*$ satisfies
\[
	\|\bW^*\|_F \lesssim m^{\frac14}d^{\frac{k-1}{4}}.
\]
\end{corollary}
\begin{proof}
We have
\[
	\|\bW^*\|_F \le \|\bW^*_L\|_F + \|\bW^*_Q\|_F.
\]
By construction $\|\bW^*_L\|_F \lesssim d^{\frac{k-1}{2}}$, and also
\[
	\|\bW^*_Q\|_F \le m^{\frac14}\|\bW^*_Q\|_{2, 4} \lesssim m^{\frac14}d^{\frac{k-1}{4}}.
\]
The desired claim follows since we assume $m \gg d^{k-1}$.
\end{proof}

\begin{corollary}\label{cor: full sol l-infty bound}
\[
	\|\bW^*\|_{2, \infty} \lesssim m^{-1/4}d^{\frac{k-1}{2}}
\]
\end{corollary}

\begin{proof}
We can writte
\[
	\|\bW^*\|_{2, \infty} \le \|\bW^*_L\|_{2, \infty} + \|\bW^*_Q\bS^*\|_{2, \infty} = \|\bW^*_L\|_{2, \infty} + \|\bW^*_Q\|_{2, \infty}.
\]
By Corollary~\ref{cor: quad ntk l-infty bound}, $\|\bW^*_Q\|_{2, \infty} \lesssim m^{-1/4}d^{\frac{k-1}{2}}$. Also, since $\bW^*_L \in \text{span}(\bP_{\le k})$, we can apply Lemma~\ref{lemma: linear sol small l-infty} to get $\|\bW^*_L\|_{2, \infty} \le m^{-\frac12}d^{\frac{k}{2}}\|\bW*_L\|_F \le m^{-\frac12}d^{k - \frac12}$. Altogether, since $m \ge d^{2k}$, we get
\[
	\|\bW^*\|_{2, \infty} \lesssim m^{-1/4}d^{\frac{k-1}{2}}.
\]
\end{proof}
\subsection{Expressivity Helper Lemmas}
\begin{lemma} \label{lemma: empirical f_k concentrate}
With probability $1 - d^{-10}$,
\[
	\EE_n[f_k(\bx)^2] \lesssim 1.
\]
\end{lemma}
\begin{proof}
Since $\|f_k\|_{L^2} = 1$ and $f$ is degree $k$, $|f_k(\bx)|^2 \le n_k$ for all $\bx$. Furthermore,
\[
	\EE_\mu[(|f_k(\bx)|^2 - 1)^2] \le \EE_\mu[f_k(\bx)^4] \le n_k.
\]
Therefore by Bernstein's Inequality, with probability $1 - d^{-10}$ we have
\[
	\EE_n[f_k(\bx)^2] - 1 \le C\left(\sqrt{\frac{n_k \log d}{n}} + \frac{n_k \log d}{n}\right) \lesssim 1,
\]
since $n \gtrsim d^k\log d \gtrsim n_k\log d$
\end{proof}

\begin{lemma}[{\cite[Exercise 4.7.3]{vershynin}}] \label{lemma: empirical cov concentrate}
With probability $1 - d^{-10}$,
\[
	\left\| \EE_n \bx\bx^T - \bI_d\right\|_{op} \le \frac12
\]
\end{lemma}

\begin{lemma}\label{lemma: NTK features concentrate}
Recall
\[
	\bSigma_{\le n_k} := \mathbb{E}_\mu\left[\varphi(\bx)^T\bP_{\le k}\varphi(\bx)\right] = \sum_{i=1}^{n_k}\lambda_i\bv_i\bv_i^T.
\]
With probability $1 - d^{-10}$,
\[
	\left\|\mathbb{E}_n\left[(\bSigma_{\le n_k}^{\dagger})^{\frac12}\varphi(\bx)\varphi(\bx)^T(\bSigma_{\le n_k}^{\dagger})^{\frac12}\right] - \bP_{\le k} \right\|_{op} \le \frac12
\]
\end{lemma}
\begin{proof}
Observe that 
\begin{align*}
	\EE_\mu\left[(\bSigma_{\le n_k}^{\dagger})^{\frac12}\varphi(\bx)\varphi(\bx)^T(\bSigma_{\le n_k}^{\dagger})^{\frac12}\right] &= (\bSigma_{\le n_k}^{\dagger})^{\frac12}\EE_\mu\left[\varphi(\bx)\varphi(\bx)^T\right](\bSigma_{\le n_k}^{\dagger})^{\frac12}\\
	&= (\bSigma_{\le n_k}^{\dagger})^{\frac12}\bSigma(\bSigma_{\le n_k}^{\dagger})^{\frac12}\\
	&= \bP_{\le k}.
\end{align*}
Therefore by \cite[5.6.4]{vershynin}, with probability $1 - d^{-10}$
\begin{align*}
	\left\|\mathbb{E}_n\left[(\bSigma_{\le n_k}^{\dagger})^{\frac12}\varphi(\bx)\varphi(\bx)^T(\bSigma_{\le n_k}^{\dagger})^{\frac12}\right] - \bP_{\le k} \right\|_{op} &\le C\left(\sqrt{\frac{d^k\log d}{n}} + \frac{d^k\log d}{n}\right)\\
	&\le \frac12,
\end{align*}
since $n \gtrsim d^k\log d$.
\end{proof}

\begin{lemma}
\label{lemma: linear term small in expectation}
On the event where Lemma~\ref{lemma: empirical cov concentrate} holds, for all $\bW$ we have
\[
	\EE_\bS\EE_n[(f_L(\bx; \bW\bS))^2] \lesssim \frac{1}{m}\|\bW\|_F^2.
\]
Furthermore, we have
\[
	\EE_\bS\|f_L(\bx; \bW\bS)\|_{L^2}^2 = \EE_\bS\EE_\mu[(f_L(\bx; \bW\bS))^2] \lesssim \frac{1}{m}\|\bW\|_F^2.
\]
\end{lemma}
\begin{proof}
We have
\begin{align*}
\EE_\bS\EE_n[(f_L(\bx; \bW\bS))^2] &= \EE_\bS \EE_n \left(\frac{1}{\sqrt{m}}\sum_{r=1}^m \sigma'(\bw_{0, r}^T\bx)\bx^T\bw_r\sigma_r\right)^2\\
&\le \EE_n\frac{1}{m}\sum_{r=1}^m (\bw_r^T\bx)^2\\
&\le \frac{1}{m}\sum_{r=1}^m\bw_r^T\EE_n[\bx\bx^T]\bw_r\\
&\lesssim \frac{1}{m}\|\bW\|^2_F
\end{align*}
The proof in the population case is identical.
\end{proof}

\begin{lemma}\label{lemma: empirical to pop small directions}
On the event where Lemma~\ref{lemma: NTK features concentrate} holds,
\[
	\frac12\|f_L(\bx; \bP_{\le k}\bW)\|_{L^2}^2 \le \EE_n[(f_L(\bx; \bP_{\le k}\bW))^2] \le \frac32\|f_L(\bx; \bP_{\le k}\bW)\|_{L^2}^2.
\]
for all $\bW$.
\end{lemma}
\begin{proof}
We can write
\begin{align*}
&\mathbb{E}_n\left[(f_L(x; \bP_{\le k}\bW))^2\right]\\
=~& \text{vec}(\bW)^T \mathbb{E}_n\left[\bP_{\le k}\varphi(\bx)\varphi(\bx)^T\bP_{\le k}\right]\text{vec}(\bW)\\
=~& \text{vec}(\bW)^T\bSigma_{\le n_k}^{\frac12} \mathbb{E}_n\left[(\bSigma_{\le n_k}^{\dagger})^{\frac12}\varphi(\bx)\varphi(\bx)^T(\bSigma_{\le n_k}^{\dagger})^{\frac12}\right]\bSigma_{\le n_k}^{\frac12}\text{vec}(\bW)\\
=~& \text{vec}(\bW)^T\bSigma_{\le n_k}^{\frac12} \left(\bP_{\le k} + \mathbb{E}_n\left[(\bSigma_{\le n_k}^{\dagger})^{\frac12}\varphi(\bx)\varphi(\bx)^T(\bSigma_{\le n_k}^{\dagger})^{\frac12}\right] - \bP_{\le k}\right)\bSigma_{\le n_k}^{\frac12}\text{vec}(\bW).
\end{align*}
Therefore, by Lemma~\ref{lemma: NTK features concentrate},
\begin{align*}
&~~\left|\EE_n[(f_L(\bx; \bP_{\le k}\bW))^2] -  \|f_L(\bx; \bP_{\le k}\bW)\|_{L^2}^2\right|\\
&= \left|\text{vec}(\bW)^T\bSigma_{\le n_k}^{\frac12} \left(\mathbb{E}_n\left[(\bSigma_{\le n_k}^{\dagger})^{\frac12}\varphi(\bx)\varphi(\bx)^T(\bSigma_{\le n_k}^{\dagger})^{\frac12}\right] - \bP_{\le k}\right)\bSigma_{\le n_k}^{\frac12}\text{vec}(\bW) \right|\\
&\le \text{vec}(\bW)^T\bSigma_{\le n_k}\text{vec}(\bW)\left(\left\|\mathbb{E}_n\left[(\bSigma_{\le n_k}^{\dagger})^{\frac12}\varphi(\bx)\varphi(\bx)^T(\bSigma_{\le n_k}^{\dagger})^{\frac12}\right] - \bP_{\le k} \right\|_{op}\right)\\
&\le \frac12\text{vec}(\bW)^T\bSigma_{\le n_k}\text{vec}(\bW)\\
&= \frac12\|f_L(\bx; \bP_{\le k}\bW)\|_{L^2}^2,
\end{align*}
as desired.
\end{proof}

\begin{lemma}
\label{lemma: linear sol small l-infty}
For any $\bW \in \text{span}(\bP_{\le k})$, we have
\[
	\|\bW\|_{2, \infty} \le m^{-\frac12}d^{\frac{k}{2}}\|\bW\|_F
\]
and
\[\|\bW\|_{2, 4} \le m^{-\frac14}d^{\frac{k}{2}}\|\bW\|_F\]
\end{lemma}
\begin{proof}
For $(r, s) \in [m] \times [d]$, let $\be_{(r, s)}$ denote the $((d-1)r + s)$th canonical basis vector in $\RR^{md}$, so that $\be_{(r, s)}^T\text{vec}(\bW) = \{\bw_r\}_s$. Let $c_1, \dots, c_{n_k}$ be scalars such that $\sum_{i=1}^{n_k}c_i^2 = \|\bW\|_F^2$ and
\[
	\text{vec}(\bW) = \sum_{i=1}^{n_k}c_i\bv_i.
\]
By Cauchy, we can bound
\[
	\left|\langle \bW, \be_{(r, s)} \rangle\right| \le \|\bW\|_{\bSigma^{-1}}\|e_{(r, s)}\|_{\bSigma}.
\]
Observe that
\[
	\|\bW\|^2_{\bSigma^{-1}} = \sum_{i=1}^{n_k}c_i^2\lambda_i^{-1} \le \lambda_{n_k}^{-1}\|\bW\|_F^2.
\]
Furthermore, since $\be_{(r, s)}^T\varphi(\bx) = \frac{1}{\sqrt{m}}\sigma'(\bw_{0, r}^T\bx)\bx_s$
\[
	\|\be_{(r, s)}\|^2_{\Sigma} = \EE_\mu\left[(\be_{(r, s)}^T\varphi(\bx))^2\right] \le \frac{1}{m}\EE_\mu[\bx_s^2] \le\frac{1}{m}.
\]
By Lemma~\ref{lemma: main intermediate spectral result} $\lambda_{n_k} = \Theta(d^{-k + 1})$, and thus we can bound
\[
	\left|\langle \bW, \be_{}{(r, s)} \rangle\right| \le m^{-\frac12}\lambda_{n_k}^{-\frac12}\|\bW\|_F \le m^{-\frac12}d^{\frac{k-1}{2}}\|\bW\|_F.
\]
Thus every row $\bw_r$ satisfies
\[
	\|\bw_r\|_2 \le m^{-\frac12}d^{\frac{k}{2}}\|\bW\|_F,
\]
so
\[
	\|\bW\|_{2, 4} \le m^{\frac14}\cdot m^{-\frac12}d^{\frac{k}{2}}\|\bW\|_F \le m^{-\frac14}d^{\frac{k}{2}}\|\bW\|_F.
\]
\end{proof}

\section{Landscape Proofs}\label{app: landscape proofs}
\subsection{Coupling Lemmas}
Recall that $\hat{L}^Q(\bW) = \EE_n[\ell(y, f_L(\bx; \bW) + f_Q(\bx; \bW))]$ denotes the empirical loss of the quadratic model. As in~\cite{bai2020}, we begin by showing the losses, gradients, and Hessians for $\hat{L}^Q(\bW)$ and $\hat L(\bW)$ are close for $\bW$ satisfying a norm bound. This is given by the following 3 coupling lemmas.\\

\begin{lemma}[Coupling of Losses]
\label{lemma: loss coupling}
\[
	\left|\hat{L}(\bW) - \hat{L}^Q(\bW)\right| \lesssim d^{3/2}m^{-1/4}\|\bW\|_{2, 4}^3
\]
\end{lemma}
~
\begin{lemma}[Coupling of Gradients]
\label{lemma: gradient coupling}

\[
	\left|\langle \nabla \hat{L}(\bW), \tilde{\bW} \rangle - \langle \nabla \hat{L}^Q(\bW), \tilde{\bW} \rangle \right| \lesssim d^{3/2}m^{-1/4}(\|\bW\|_{2, 4}^3 + \|\tilde{\bW}\|_{2, 4}^3)\cdot\max_{i \in [n]}\left|\langle \nabla_\bW f(\bx; \bW), \tilde{\bW} \rangle \right|
\]

\end{lemma}
~
\begin{lemma}[Coupling of Hessians]
\label{lemma: Hessian coupling}

\begin{align*}
	&\left|\nabla^2\hat{L}(\bW)[\tilde{\bW}, \tilde{\bW}] -  \nabla^2\hat{L}^Q(\bW)[\tilde{\bW}, \tilde{\bW}]\right|\\
	&~~\lesssim d^{3/2}m^{-\frac14}\left(\|\bW\|_{2, 4}^3 + \|\tilde{\bW}\|_{2, 4}^3\right)\left(d\|\tilde{\bW}\|_{2, 4}^2 + \max_{i \in [n]}\left|\langle \nabla_\bW f(\bx_i; \bW), \tilde{\bW} \rangle\right|^2 + \max_{i \in [n]}\left|\langle \nabla f(\bx_i, \bW), \tilde \bW \rangle\right|\right)\\
	&~~~~+ d^3\|\bW\|_{2, 4}^4\|\tilde{\bW}\|_{2, \infty}^2.
\end{align*}
\end{lemma}
While these Lemmas are similar in structure to those in~\cite{bai2020}, extra care must be taken to properly deal with the effect of the $f_L$ terms. These proofs are presented in Appendix~\ref{app: prove coupling lemmas}.

\subsubsection{Auxiliary Results}
We first prove some intermediate results which are used in the proofs of the coupling lemmas:\\
\begin{lemma}[Coupling of Function Values]
\label{lemma: function values}
\[
	\left|f(\bx; \bW) - f_L(\bx; \bW) - f_Q(\bx; \bW) \right| \le \frac{1}{\sqrt{m}}\sum_{r=1}^m |\bw_r^T\bx|^3
\]
\end{lemma}
\begin{proof}
\begin{align*}
	&\left|f(\bx; \bW) - f_L(\bx; \bW) - f_Q(\bx; \bW) \right|\\
	&= \frac{1}{\sqrt{m}}\left|\sum_{r=1}^m \sigma(\bw_{0, r}^T\bx + \bw_r^T\bx) - \sigma(\bw_{0, r}^T\bx) - \sigma^\prime(\bw_{0, r}^T\bx)(\bw_r^T\bx) - \frac12\sigma^{\prime \prime}(\bw_{0, r}^T\bx)(\bw_r^T\bx)^2 \right|\\
	&\le \frac{1}{\sqrt{m}}\sum_{r=1}^m \left|\sigma^{\prime\prime\prime}\right|_\infty |\bw_r^T\bx|^3\\
	&\le \frac{1}{\sqrt{m}}\sum_{r=1}^m|\bw_r^T\bx|^3.
\end{align*}
\end{proof}

\begin{lemma}[Coupling of Function Gradients]
\label{lemma: function gradients}
\[
	\left|\langle \nabla_\bW f_L(\bx; \bW) + \nabla_\bW f_Q(\bx; \bW), \tilde{\bW} \rangle - \langle \nabla_\bW f(\bx; \bW), \tilde{\bW} \rangle \right| \le \frac{1}{\sqrt{m}}\sum_{r=1}^m |\tilde{\bw}_r^T\bx||\bw_r^T\bx|^2
\]
\end{lemma}
\begin{proof}
Taylor expanding $\sigma'$, we have
\[
	\left|\sigma'(\bw_{0, r}^T\bx + \bw_r^T\bx) - \sigma^\prime(\bw_{0, r}^T\bx) + \sigma^{\prime\prime}(\bw_{0, r}^T\bx)(\bw_r^T\bx)\right| \le |\bw_r^T\bx|^2.
\]
Therefore
\begin{align*}
	& \left|\langle \nabla_\bW f_L(\bx; \bW) + \nabla_\bW f_Q(\bx; \bW), \tilde{\bW} \rangle - \langle \nabla_\bW f(\bx; \bW), \tilde{\bW} \rangle \right|\\
	=~&\left|\frac{1}{\sqrt{m}}\sum_{r=1}^m a_r\left(\sigma^\prime(\bw_{0, r}^T\bx)(\tilde{\bw}_r^T\bx) + \sigma^{\prime\prime}(\bw_{0, r}^T\bx)(\bw_r^T\bx)(\tilde{\bw}_r^T\bx) - \sigma'(\bw_{0, r}^T\bx + \bw_r^T\bx)(\tilde{\bw}_r^T\bx)\right)\right|\\
	\le~&\frac{1}{\sqrt{m}}\sum_{r=1}^m|\tilde{\bw}_r^T\bx|\left|\sigma'(\bw_{0, r}^T\bx + \bw_r^T\bx) - \sigma^\prime(\bw_{0, r}^T\bx) + \sigma^{\prime\prime}(\bw_{0, r}^T\bx)(\bw_r^T\bx) \right|\\
	\le~& \frac{1}{\sqrt{m}}\sum_{r=1}^m|\tilde{\bw}_r^T\bx||\bw_r^T\bx|^2.
\end{align*}
\end{proof}

\begin{lemma}\label{lemma: bound max gradient}
\[
	\max_{i \in [n]}\left|\langle \nabla_\bW f(\bx_i; \bW), \tilde{\bW} \rangle \right| \lesssim n^{\frac12}\EE_n\left[(f_L(\bx; \tilde{\bW}))^2\right]^{\frac12} + dn^{\frac12}m^{-\frac12}\|\bW\|_F\|\tilde{\bW}\|_F
\]
\end{lemma}
\begin{proof}
We can decompose
\begin{align*}
\EE_n\left[\langle \nabla_\bW f(\bx; \bW), \tilde{\bW} \rangle ^2\right] &\le 2\EE_n\left[\langle \nabla_\bW f(\bx; \mathbf{0}), \tilde{\bW} \rangle ^2\right] + 2\EE_n\left[\langle \nabla_\bW f(\bx; \mathbf{0}) - \nabla_\bW f(\bx; \bW), \tilde{\bW} \rangle ^2\right]
\end{align*}

We can bound
\begin{align*}
\langle \nabla_\bW f(\bx; \bZero) - \nabla_\bW f(\bx; \bW), \tilde \bW \rangle^2 &\le \|\nabla_\bW f(\bx; \bZero) - \nabla_\bW f(\bx; \bW)\|_2^2\|\tilde \bW\|_F^2\\
&\le \|\tilde \bW\|_F^2\sum_{r=1}^m\frac{d}{m}\left|\sigma'(\bw_{0, r}^T\bx) - \sigma'(\bw_{0, r}^T\bx + \bw_{r}^T\bx) \right|^2\\
&\le \|\tilde \bW\|_F^2\sum_{r=1}^m\frac{d}{m}(\bw_r^T\bx)^2\\
&\le \frac{d^2}{m}\|\bW\|_F^2\|\tilde \bW\|_F^2.
\end{align*}

Thus
\begin{align*}
\EE_n\left[\langle \nabla_\bW f(\bx; \bW), \tilde{\bW} \rangle ^2\right] \lesssim \EE_n\left[(f_L(\bx; \tilde{\bW}))^2\right] + \frac{d^2}{m}\|\bW\|_F^2\|\tilde{\bW}\|_F^2,
\end{align*}
So
\begin{align*}
\max_{i \in [n]}\left|\langle \nabla_\bW f(\bx_i; \bW), \tilde{\bW} \rangle \right| &\le \left(\sum_{i=1}^n\langle \nabla_\bW f(\bx_i; \bW), \tilde{\bW} \rangle^2\right)^{1/2}\\
&\lesssim \left(n\left(\EE_n\left[(f_L(\bx; \tilde{\bW}))^2\right] + \frac{d^2}{m}\|\bW\|_F^2\|\tilde{\bW}\|_F^2\right)\right)^{1/2}\\
&\lesssim n^{\frac12}\EE_n\left[(f_L(\bx; \tilde{\bW}))^2\right]^{\frac12} + dn^{\frac12}m^{-\frac12}\|\bW\|_F\|\tilde{\bW}\|_F.
\end{align*}
\end{proof}

\subsubsection{Proof of Coupling Lemmas}\label{app: prove coupling lemmas}

\begin{proof}[Proof of Lemma~\ref{lemma: loss coupling}]
By Lipschitzness of the loss,
\begin{align*}
|\hat{L}(\bW) - \hat{L}^Q(\bW)| &\le \mathbb{E}_n\left|\ell(y, f(\bx; \bW)) - \ell(y, f_L(\bx; \bW) + f_Q(\bx; \bW))\right|\\
&\le \mathbb{E}_n\left|f(\bx; \bW) - f_L(\bx; \bW) - f_Q(\bx; \bW) \right|\\
&\le \frac{1}{\sqrt{m}}\sum_{r=1}^m\mathbb{E}_n|\bw_r^T\bx|^3\\
&\le C\frac{d^{3/2}}{\sqrt{m}}\sum_{r=1}^m\|\bw_r\|_2^3\\
& \le Cm^{-1/4}d^{3/2}\|\bW\|_{2, 4}^3.
\end{align*}
\end{proof}

\begin{proof}[Proof of Lemma~\ref{lemma: gradient coupling}]
The gradients are
\begin{align*}
\langle \nabla \hat{L}(\bW), \tilde{\bW} \rangle &= \mathbb{E}_n\left[\ell'(y, f(\bx; \bW))\cdot \langle \nabla_\bW f(\bx; \bW), \tilde{\bW} \rangle \right]\\
\langle \nabla \hat{L}^Q(\bW), \tilde{\bW} \rangle&= \mathbb{E}_n\left[\ell'(y, f_L(\bx; \bW) + f_Q(\bx; \bW))\cdot \langle \nabla_\bW f_L(\bx; \bW) + \nabla_\bW f_Q(\bx; \bW), \tilde{\bW} \rangle \right]
\end{align*}
Therefore
\begin{align*}
&\left|\langle \nabla \hat{L}(\bW), \tilde{\bW} \rangle - \langle \nabla \hat{L}^Q(\bW), \tilde{\bW} \rangle \right|\\
 &\le \mathbb{E}_n\left[\left|\ell'(y, f_L(\bx; \bW) + f_Q(\bx; \bW))\right| \cdot \left|\langle \nabla_\bW f_L(\bx; \bW) + \nabla_\bW f_Q(\bx; \bW), \tilde{\bW} \rangle - \langle \nabla_\bW f(\bx; \bW), \tilde{\bW} \rangle \right|\right]\\
 &+\mathbb{E}_n\left[\left|\ell'(y, f_L(\bx; \bW) + f_Q(\bx; \bW)) - \ell'(y, f(\bx; \bW)) \right|\cdot \left|\langle \nabla_\bW f(\bx; \bW), \tilde{\bW} \rangle\right|\right]\\
 &\le \frac{1}{\sqrt{m}}\sum_{r=1}^m\mathbb{E}_n\left[|\tilde{\bw}_r^T\bx||\bw_r^T\bx|^2\right] + \mathbb{E}_n\left[\left|f(\bx; \bW) - f_L(\bx; \bW) - f_Q(\bx; \bW)\right|\cdot\left|\langle \nabla_\bW f(\bx; \bW), \tilde{\bW}\rangle\right|\right],
\end{align*}
by Lemma~\ref{lemma: function gradients}. The first term is
\begin{align*}
\frac{1}{\sqrt{m}}\sum_{r=1}^m\mathbb{E}_n\left[|\tilde{\bw}_r^T\bx||\bw_r^T\bx|^2\right] &\le \frac{1}{\sqrt{m}}\sum_{r=1}^m\mathbb{E}_n\left[\frac13|\tilde{\bw}_r^T\bx|^3 + \frac23|\bw_r^T\bx|^3\right]\\
&\le \frac{Cd^{3/2}}{\sqrt{m}}\sum_{r=1}^m \|\bw_r\|_2^3 + \|\tilde{\bw}_r\|_2^3\\
&\le Cd^{3/2}m^{-1/4}\left(\|\bW\|_{2, 4}^3 + \|\tilde{\bW}\|_{2, 4}^3\right).
\end{align*}
The second term can be bounded as
\begin{align*}
&\mathbb{E}_n\left[\left|f(\bx; \bW) - f_L(\bx; \bW) - f_Q(\bx; \bW)\right|\cdot\langle \nabla_\bW f(\bx; \bW), \tilde{\bW}\rangle\right]\\
 \le~&\EE_n\left|f(\bx; \bW) - f_L(\bx; \bW) - f_Q(\bx; \bW)\right|\cdot \max_{i \in [n]}\left|\langle \nabla_\bW f(\bx_i; \bW), \tilde{\bW} \rangle \right|\\
 \le~&Cd^{3/2}m^{-1/4}\|\bW\|_{2, 4}^3\cdot \max_{i \in [n]}\left|\langle \nabla_\bW f(\bx; \bW), \tilde{\bW} \rangle \right|,
\end{align*}
by Lemma~\ref{lemma: function values}.


\end{proof}

\begin{proof}[Proof of Lemma~\ref{lemma: Hessian coupling}]

The Hessians are
\begin{align*}
&\nabla^2 \hat{L}^Q(\bW)[\tilde{\bW}, \tilde{\bW}]\\
=~& \EE_n\left[\ell'(y, f_L(\bx; \bW) + f_Q(\bx; \bW))\frac{1}{\sqrt{m}}\sum_{r=1}^m a_r\sigma''(\bw_{0, r}^T\bx)(\tilde{\bw}_r^T\bx)^2\right]\\
+~& \EE_n\left[\ell''(y, f_L(\bx; \bW) + f_Q(\bx; \bW))\left(\langle \nabla_\bW f_L(\bx; \bW) + \nabla_\bW f_Q(\bx; \bW), \tilde{\bW} \rangle)\right)^2\right]
\end{align*}
and
\begin{align*}
\nabla^2 \hat{L}(\bW)[\tilde{\bW}, \tilde{\bW}] =~& \EE_n\left[\ell'(y, f(\bx; \bW))\frac{1}{\sqrt{m}}\sum_{r=1}^m a_r\sigma''((\bw_{0, r} + \bw_r)^Tx)(\tilde{\bw}_r^T\bx)^2\right]\\
+~& \EE_n\left[\ell''(y, f(\bx; \bW))\langle \nabla_\bW f(\bx; \bW), \tilde{\bW} \rangle^2\right]
\end{align*}
The difference between the first terms can be bounded by
\begin{align*}
&\mathbb{E}_n\left[\left(\ell'(y, f_L(\bx; \bW) + f_Q(\bx; \bW)) - \ell'(y, f(\bx; \bW))\right)\sum_{r=1}^m a_r\sigma''(\bw_{0, r}^T\bx)(\tilde{\bw}_r^T\bx)^2\right]\\
+~&\mathbb{E}_n\left[\frac{1}{\sqrt{m}}\sum_{r=1}^m\left|(\tilde{\bw}_r^T\bx)^2\left(\sigma''(\bw_{0, r}^T\bx) - \sigma''(\bw_{0, r}^T\bx + \bw_r^T\bx)\right)\right|\right]\\
\le~& \mathbb{E}_n\left[\left|f_L(\bx; \bW) + f_Q(\bx; \bW) - f(\bx; \bW)\right|\frac{1}{\sqrt{m}}\sum_{r=1}^m |\tilde{\bw}_r^T\bx|^2 \right] + \frac{1}{\sqrt{m}}\mathbb{E}_n\left[\sum_{r=1}^m |\bw_r^T\bx||\tilde{\bw}_r^T\bx|^2\right]\\
\le~& \frac{1}{m}\EE_n\left[\left(\sum_{r=1}^m|\bw_r^T\bx|^3\right)\left(\sum_{r=1}^m|\tilde{\bw}_r^T\bx|^2\right)\right] + \frac{1}{\sqrt{m}}\EE_n\left[\sum_{r=1}^m|\bw_r^Tx||\tilde{\bw}_r^Tx|^2\right]\\
\le~&\frac{1}{m}d^{3/2}\left(\sum_{r=1}^m\|\bw_r\|_2^3\right)\EE_n\left[\sum_{r=1}^m|\tilde{\bw}_r^T\bx|^2\right] + \frac{1}{\sqrt{m}}\EE_n\left[\sum_{r=1}^m\frac13|\bw_r^Tx|^3 + \frac23|\tilde{\bw}_r^Tx|^3\right]\\
\le~&\frac{d^{5/2}}{m}\left(\sum_{r=1}^m\|\bw_r\|_2^3\right)\left(\sum_{r=1}^m\|\tilde{\bw}_r\|^2_2\right) + \frac{d^{3/2}}{\sqrt{m}}\sum_{r=1}^m\left(\|\bw_r\|_2^3 + \|\tilde{\bw}_r\|_2^3\right)\\
\le~&\frac{d^{5/2}}{m}\cdot m^{1/4}\|\bW\|_{2, 4}^3\cdot m^{1/2}\|\tilde{\bW}\|_{2, 4}^2 + \frac{d^{3/2}}{\sqrt{m}}m^{1/4}\left(\|\bW\|_{2, 4}^3+ \|\tilde{\bW}\|_{2, 4}^3\right)\\
\le~& O\left(d^{5/2}m^{-1/4}\left(\|\bW\|_{2, 4}^3 + \|\tilde{\bW}\|_{2, 4}^3\right)\|\tilde{\bW}\|_{2, 4}^2\right).
\end{align*}

The difference between the second terms is upper bounded by
\begin{align*}
&\mathbb{E}_n\left[(\ell''(y, f_L(\bx; \bW) + f_Q(\bx; \bW)) - \ell''(y, f(\bx; \bW)))\langle \nabla_\bW f(\bx; \bW), \tilde{\bW} \rangle^2 \right]\\
+~&\mathbb{E}_n\left|\left(\langle \nabla_\bW f_L(\bx; \bW) + \nabla_\bW f_Q(\bx; \bW), \tilde{\bW} \rangle\right)^2 - \langle \nabla_\bW f(\bx; \bW), \tilde{\bW} \rangle^2\right|
\end{align*}
The first term can be bounded by
\begin{align*}
&\max_{i \in [n]}\langle \nabla_\bW f(\bx_i; \bW), \tilde{\bW} \rangle^2 \cdot \EE_n\left|f_L(\bx; \bW) + f_Q(\bx; \bW) - f(\bx; \bW)\right|\\
\le~&\max_{i \in [n]}\langle \nabla_\bW f(\bx_i; \bW), \tilde{\bW} \rangle^2\cdot Cd^{3/2}m^{-1/4} \|\bW\|_{2, 4}^3.
\end{align*}
Also, by Lemma~\ref{lemma: function gradients} we can bound
\begin{align*}
&\mathbb{E}_n\left|\left(\langle \nabla_\bW f_L(\bx; \bW) + \nabla_\bW f_Q(\bx; \bW), \tilde{\bW} \rangle\right)^2 - \langle \nabla_\bW f(\bx; \bW), \tilde{\bW} \rangle^2\right|\\
\le~&\mathbb{E}_n\left[\left(\frac{1}{\sqrt{m}}\sum_{r=1}^m |\tilde{\bw}_r^T\bx||\bw_r^T\bx|^2\right)\cdot\left|\langle \nabla_\bW f_L(\bx; \bW) + \nabla_\bW f_Q(\bx; \bW), \tilde{\bW} \rangle + \langle \nabla_\bW f(\bx; \bW), \tilde{\bW} \rangle \right|\right]\\
\le~&\mathbb{E}_n\left[\left(\frac{1}{\sqrt{m}}\sum_{r=1}^m |\tilde{\bw}_r^T\bx||\bw_r^T\bx|^2\right)\cdot\left(\frac{1}{\sqrt{m}}\sum_{r=1}^m |\tilde{\bw}_r^T\bx||\bw_r^T\bx|^2 + 2\langle \nabla f(\bx; \bW), \tilde \bW \rangle \right)\right]\\
\lesssim~& \EE_n\left[\frac{1}{m}\left(\sum_{r=1}^m |\tilde{\bw}_r^T\bx||\bw_r^T\bx|^2\right)^2\right] + \max_{i \in [n]}\langle \nabla f(\bx_i, \bW), \tilde \bW \rangle \cdot \EE_n\left[\frac{1}{\sqrt{m}}\sum_{r=1}^m |\tilde{\bw}_r^T\bx||\bw_r^T\bx|^2\right]\\
\lesssim~& d^3\sum_{r=1}^m \|\tilde \bw_r\|^2\|\bw_r\|^4 + \max_{i \in [n]}\langle \nabla f(\bx_i, \bW), \tilde \bW \rangle \cdot d^{3/2}m^{-\frac14}\left(\|\bW\|_{2, 4}^3 + \|\tilde \bW\|_{2, 4}^3\right)\\
\lesssim~& d^3\|\bW\|_{2, 4}^4\|\tilde{\bW}\|_{2, \infty}^2 + \max_{i \in [n]}\langle \nabla f(\bx_i, \bW), \tilde \bW \rangle \cdot d^{3/2}m^{-\frac14}\left(\|\bW\|_{2, 4}^3 + \|\tilde \bW\|_{2, 4}^3\right).
\end{align*}
\end{proof}

\subsection{Proof of Lemma~\ref{lemma: central landscape result}}
\begin{proof}
We would first like to show that the quadratic model has good landscape properties. To prove this, we begin by showing any approximate stationary point must be ``localized," in that the values of the regularizers at these stationary points must not be too large.
\begin{lemma}[Localization]\label{lemma: localization}
Let $\lambda_2 = \varepsilon_{min}$, $\lambda_3 = m^{\frac12}d^{-\frac{k-1}{2}}\varepsilon_{min}$, $\lambda_4 = d^{-2(k-1)}\varepsilon_{min}$. Assume $m \ge \max\left(d^{4k + 4}n^2\varepsilon_{min}^{-2}, d^{16(k+1)/3}\varepsilon_{min}^{-8/3}, n^4\varepsilon_{min}^{-4}\right)$, and $\nu \le m^{-\frac14}$. Then, for any $\nu$-first-order stationary point $\bW$ of $L_\lambda$, we have
\begin{align*}
	\cR_2(\bW) &\le d^{2(k+1)/3}\varepsilon_{min}^{-4/3}\\
	\cR_3(\bW) &\le m^{-\frac12}d^{\frac{7k+1}{6}}\varepsilon_{min}^{-4/3}\\
	\cR_4(\bW) &\le d^{\frac{8k-4}{3}}\varepsilon_{min}^{-4/3}.
\end{align*}
\end{lemma}

Next, we show that for these localized points, the landscape of $\hat{L}^Q$ is ``good.''
\begin{lemma}
\label{lemma: quad landscape good}
Let $\bW^*_L = \bP_{\le k}\bW^*$, $\bW_L = \bP_{\le k}\bW$. There exists a universal constant $C$ such that

\begin{align}
&\mathbb{E}_{\bS}\left[\nabla^2\hat{L}^Q(\bW)[\bW^*\bS, \bW^*\bS]\right] - \langle \nabla \hat{L}^Q(\bW), \bW - 2\bW^*_L + \bW_L \rangle + 2\hat{L}^Q(\bW) - 2\hat{L}^Q(\bW^*)\\
 \le~&C\left(m^{-\frac12}d^{\frac{k+1}{2}} + m^{-1}d^{\frac{k+3}{2}}\|\bW\|_{2, 4}^2 + d^2m^{-\frac12}\|\bW\|_F(\|\bW^*_L\|_F + \|\bW_L\|_F) + \cR_3(\bW^*)^{\frac12} + \cR_3(\bW)^{\frac12}\right)
\end{align}
\end{lemma} 

As a corollary, we obtain that for localized $\bW$, this error term can be made arbitrarily small for sufficiently large $m$.
\begin{corollary}
\label{cor: quad landscape good}
Assume $m \gtrsim d^{\frac{14k + 20}{3}}\varepsilon_{min}^{-22/3}$. Then, additionally under the assumptions of Lemma~\ref{lemma: central landscape result}, any $\nu$-first-order stationary point $\bW$ satisfies
\[
	\mathbb{E}_{\bS}\left[\nabla^2\hat{L}^Q(\bW)[\bW^*\bS, \bW^*\bS]\right] - \langle \nabla \hat{L}^Q(\bW), \bW - 2\bW^*_L + \bW_L \rangle + 2\hat{L}^Q(\bW) - 2\hat{L}^Q(\bW^*) \le \varepsilon_{min}.
\]
\end{corollary}

Our final step is to show that the landscape of $L_\lambda$ is good. To do this, we first show that the landscapes of the quadratic model and original model couple for localized points. This coupling is given by the following lemma.

\begin{lemma}
\label{lemma: quad model couples}
Assume the conditions of Lemma~\ref{lemma: central landscape result}, Corollary~\ref{cor: quad landscape good}, and let $\bW$ be a $\nu$-first order stationary point. Furthermore, assume $m \gtrsim n^4d^{\frac{26(k+1)}{3}}\varepsilon_{min}^{-22/3}$. Then,
\begin{align*}
	\left|\hat{L}(\bW) - \hat{L}^Q(\bW)\right| &\le \varepsilon_{min}\\
	\left|\hat{L}(\bW^*) - \hat{L}^Q(\bW^*)\right| &\le \varepsilon_{min}\\
	\left|\langle \nabla \hat{L}(\bW), \bW - 2\bW^*_L + \bW_L \rangle - \langle \nabla \hat{L}^Q(\bW), \bW - 2\bW^*_L + \bW_L \rangle \right| &\le \varepsilon_{min}\\
	\left|\EE_\bS\left[\nabla^2\hat{L}(\bW)[\bW^*\bS, \bW^*\bS]\right] -  \EE_\bS\left[\nabla^2\hat{L}^Q(\bW)[\bW^*\bS, \bW^*\bS]\right]\right|&\le \varepsilon_{min}.
\end{align*}
\end{lemma}

An immediate corollary is that the landscape of $\hat{L}$ must be good.
\begin{corollary}
\label{cor: empirical landscape good}
Let $\bW$ be a $\nu$-first-order stationary point of $L_\lambda$. Under the same conditions of Lemma~\ref{lemma: central landscape result}, Corollary~\ref{cor: quad landscape good}, Lemma~\ref{lemma: quad model couples}, we have
\[
	\mathbb{E}_{\bS}\left[\nabla^2\hat{L}(\bW)[\bW^*\bS, \bW^*\bS]\right] - \langle \nabla \hat{L}(\bW), \bW - 2\bW^*_L + \bW_L \rangle + 2\hat{L}(\bW) - 2\hat{L}(\bW^*) \lesssim \varepsilon_{min}.
\]
\end{corollary}

To conclude, we must show that adding the regularizers has a benign effect on the landscape

\begin{lemma}\label{lemma: add regularizers}
Define 
\begin{equation}
\mathcal{R}(\bW) = \lambda_1\cR_1(\bW) + \lambda_2\cR_2(\bW) + \lambda_3\cR_3(\bW)+ \lambda_4\cR_4(\bW)
\end{equation}
to be the total regularization term. Under the conditions of Lemma~\ref{lemma: central landscape result}, we have
\begin{equation}
	\EE_\bS\nabla^2 \cR(\bW)[\bW^*\bS, \bW^*\bS] - \langle \nabla \cR(\bW), \bW  - 2\bW^*_L + \bW_L \rangle + 2\cR(\bW) - 2\cR(\bW^*) \lesssim \varepsilon_{min}.
\end{equation}
\end{lemma}
Lemma~\ref{lemma: central landscape result} now directly follows by adding the results of Corollary~\ref{cor: empirical landscape good} and Lemma~\ref{lemma: add regularizers}.

\end{proof}

\subsection{Proof of Corollary~\ref{cor: main landscape cor}}
\begin{proof}
Let $\bW$ be an $(\nu, \gamma)$-SOSP of $L_\lambda(\bW)$. Then
\begin{align*}
	\langle \nabla L_\lambda(\bW), \bW - 2\bW^*_L + \bW_L \rangle &\le \nu \|\bW - 2\bW^*_L + \bW_L\|_F\\
	& \le \nu\cdot m^{1/4}d^{\frac{k}{3} - \frac16}\varepsilon^{-1/6}\\
	&\le m^{-1/4}d^{\frac{k}{3} - \frac16}\varepsilon^{-1/6}\\
	&\le \varepsilon_{min},
\end{align*}
since we have chosen $\nu \le m^{-1/2}, m \ge d^{\frac{4k - 2}{3}}\varepsilon^{-14/3}$. Also,
\begin{align*}
	\nabla^2L_\lambda(\bW)[\bW^*\bS, \bW^*\bS] &\ge -\gamma\|\bW^*\|_F^2\\
	&\ge -\gamma m^{1/2}d^{(k-1)/2}\\
	&\ge -m^{-1/4}d^{(k-1)/2}\\
	 &\ge -\varepsilon_{min},
\end{align*}
since we have chosen $\gamma \le m^{-3/4}, m\ge d^{2(k-1)}\varepsilon^{-4}$. Altogether, we can bound
\begin{align*}
L_\lambda(\bW) \lesssim L_\lambda(\bW^*) + \langle \nabla L_\lambda(\bW), \bW - 2\bW^*_L + \bW_L \rangle - \EE_\bS[\nabla^2L_\lambda(\bW)[\bW^*\bS, \bW^*\bS]] + \varepsilon_{min} \lesssim \varepsilon_{min},
\end{align*}
as desired.
\end{proof}

\subsection{Proofs of Intermediate Results}
\subsubsection{Proof of Lemma~\ref{lemma: localization}}
\begin{proof}
Let $\bW$ be an $\nu$-first-order stationary point of $L_\lambda$. Then
\begin{equation}\label{eq: SOSP}
	\langle \nabla L_\lambda(\bW), \bW \rangle \le \nu\|\bW\|_F.
\end{equation}
We have that
\begin{align*}
	\langle \nabla \hat{L}^Q(\bW), \bW \rangle = \EE_n\left[\ell'(y, f_L(\bx; \bW) + f_Q(\bx; \bW))\cdot(2f_Q(\bx; \bW) + f_L(\bx; \bW))\right].
\end{align*}
First, by convexity we can bound
\[
	\EE_n\left[\ell'(y, f_L(\bx; \bW) + f_Q(\bx; \bW))\cdot(f_Q(\bx; \bW) + f_L(\bx; \bW))\right] \ge \hat{L}^Q(\bW) - \hat{L}^Q(\mathbf{0}) \ge -1.
\]
Secondly, we can bound
\begin{align*}
	\left|\EE_n\left[\ell'(y, f_L(\bx; \bW) + f_Q(\bx; \bW))f_Q(\bx; \bW)\right]\right| \le \EE_n\left[\frac{1}{\sqrt{m}}\sum_{r=1}^m(\bw_r^T\bx)^2\right] \le \frac{d}{\sqrt{m}}\|\bW\|_F^2.
\end{align*}
Finally, by Lemma~\ref{lemma: gradient coupling}, 
\[
	\left| \langle \nabla \hat{L}^Q(\bW), \bW \rangle - \langle \nabla \hat{L}(\bW), \bW \rangle \right| \lesssim d^{3/2}m^{-1/4}\|\bW\|_{2, 4}^3\cdot \max_{i \in [n]}\left|\langle \nabla f(\bx; \bW), \bW \rangle \right|.
\]
Altogether, 
\[
	\langle \nabla \hat{L}(\bW), \bW \rangle \ge -1 - \frac{d}{\sqrt{m}}\|\bW\|_F^2 - Cd^{3/2}m^{-1/4}\|\bW\|_{2, 4}^3\cdot \max_{i \in [n]}\left|\langle \nabla f(\bx; \bW), \bW \rangle \right|.
\]

We next turn to the regularizers. $\cR_i$ for $i = 1, 2, 3$ are all quadratics, so $\langle \nabla \cR_i(\bW), \bW \rangle = 2\cR_i(\bW)$. Also it is true that $\langle \nabla \cR_4(\bW), \bW \rangle = 8\cR_4(\bW)$. Plugging into \eqref{eq: SOSP}, we get
\begin{align*}
	\nu\|\bW\|_F &\ge \left\langle \nabla \left(\hat{L}(\bW) + \lambda_1\cR_1(\bW) + \lambda_2\cR_2(\bW) + \lambda_3\cR_3(\bW) + \lambda_4\cR_4(\bW)\right), \bW \right\rangle\\
	&\ge 2\lambda_1\cR_1(\bW) + 2\lambda_2\cR_2(\bW) + 2\lambda_3\cR_3(\bW) + 8\lambda_4\cR_4(\bW)\\
	&~~~~-1 - dm^{-\frac12}\|\bW\|_F^2 - Cd^{3/2}m^{-\frac14}\|\bW\|_{2, 4}^3\cdot \max_{i \in [n]}\left|\langle \nabla f(\bx; \bW), \bW \rangle \right|.
\end{align*}
Therefore (using $\nu \le m^{-1/4}$),
\begin{align*}
&\lambda_1\cR_1(\bW) + \lambda_2\cR_2(\bW) + \lambda_3\cR_3(\bW) + \lambda_4\cR_4(\bW)\\
\lesssim~& 1 + dm^{-\frac12}\|\bW\|_F^2 + d^{3/2}m^{-\frac14}\|\bW\|_{2, 4}^3\cdot \max_{i \in [n]}\left|\langle \nabla f(\bx; \bW), \bW \rangle \right| + \nu\|\bW\|_F\\
\lesssim~& 1 + dm^{-\frac12}\|\bW\|_F^2 + d^{3/2}m^{-\frac14}\|\bW\|_{2, 4}^3\cdot (n^{\frac12}\EE_n\left[(f_L(\bx; \bW))^2\right]^{\frac12} + dn^{\frac12}m^{-\frac12}\|\bW\|^2_F),
\end{align*}
where the last step follows from Lemma~\ref{lemma: bound max gradient}.

We can bound
\begin{align*}
\EE_n\left[(f_L(\bx; \bW))^2\right] &= \EE_n\left[(f_L(\bx; \bP_{>k}\bW) + f_L(\bx; \bP_{\le k}\bW))^2\right]\\
&\lesssim \EE_n\left[(f_L(\bx; \bP_{>k}\bW))^2 + (f_L(\bx; \bP_{\le k}\bW))^2\right]\\
&\lesssim \EE_n\left[(f_L(\bx; \bP_{>k}\bW))^2\right] + \|f_L(\bx; \bP_{\le k}\bW)\|_{L^2}^2\\
&= \cR_2(\bW) + \cR_3(\bW). 
\end{align*}

Therefore by AM-GM,
\begin{align}
&\lambda_2\cR_2(\bW) + \lambda_3\cR_3(\bW) + \lambda_4\cR_4(\bW)\\
&~~\lesssim 1 + dm^{-\frac12}\|\bW\|_F^2 + m^{-\frac14}d^3\|\bW\|_{2, 4}^6 + m^{-\frac14}n\EE_n\left[(f_L(\bx; \bW))^2\right] + m^{-\frac14}\|\bW\|_{2, 4}^3d^{5/2}n^{\frac12}m^{-\frac12}\|\bW\|_F^2\\
&~~\lesssim 1 + d\|\bW\|_{2, 4}^2 + m^{-\frac14}d^{5/2}n^{\frac12}\|\bW\|_{2, 4}^5 + m^{-\frac14}d^3\|\bW\|_{2, 4}^6 + m^{-\frac14}n\left(\cR_2(\bW) + \cR_3(\bW)\right),\label{eq: bound all reg}
\end{align}
where the last step uses $\|\bW\|_F \le m^{\frac14}\|\bW\|_{2, 4}$.

Since $m^{-\frac14}n \le \varepsilon_{min}/2 \le \lambda_2/2, \lambda_3/2$, we have
\begin{equation}\label{eq: bound on 2,4}
	\lambda_4\|\bW\|_{2, 4}^8 \lesssim 1 + d\|\bW\|_{2, 4}^2 + m^{-\frac14}d^{5/2}n^{\frac12}\|\bW\|_{2, 4}^5 + m^{-\frac14}d^3\|\bW\|_{2, 4}^6
\end{equation}

Therefore, plugging in $\lambda_4 = d^{-2(k-1)}\varepsilon_{min},$
\begin{equation}
\|\bW\|_{2, 4} \lesssim \max\left(d^{\frac{k-1}{4}}\varepsilon_{min}^{-1/8}, d^{\frac{2k - 1}{6}}\varepsilon_{min}^{-1/6}, m^{-1/12}d^{2k/3 + 1/6}n^{1/6}\varepsilon_{min}^{-1/3}, m^{-\frac18}d^{k + 1/2}\varepsilon_{min}^{-1/2}\right).
\end{equation}
Since $\varepsilon_{min} < 1$ we trivially have $d^{\frac{k-1}{4}}\varepsilon_{min}^{-1/8} < d^{\frac{2k - 1}{6}}\varepsilon_{min}^{-1/6}$. Also, since $m \ge d^{4k + 4}n^2\varepsilon_{min}^{-2}$, we have
\[
	d^{\frac{2k - 1}{6}}\varepsilon_{min}^{-1/6} \ge m^{-1/12}d^{\frac{2k}{3} + \frac16}n^{1/6}\varepsilon_{min}^{-1/3}.
\]
Additionally, assuming $m \ge d^{16/3(k+1)}\varepsilon_{min}^{-8/3}$, we have
\[
	d^{\frac{2k - 1}{6}}\varepsilon_{min}^{-1/6} \ge m^{-\frac18}d^{k + 1/2}\varepsilon_{min}^{-1/2}.
\]
Therefore we can bound
\[
	\|\bW\|_{2, 4} \lesssim d^{\frac{2k-1}{6}}\varepsilon_{min}^{-1/6},
\]
and thus
\begin{equation}
\cR_4(\bW) = \|\bW\|^8_{2, 4} \lesssim d^{\frac{8k-4}{3}}\varepsilon_{min}^{-4/3}.
\end{equation}
In this case, the RHS of \eqref{eq: bound on 2,4} can be upper bounded by $d\|\bW\|_{2, 4}^2$. Plugging back into~\eqref{eq: bound all reg}, we get
\[
	\lambda_2\cR_2(\bW) + \lambda_3\cR_3(\bW) \lesssim d^{2(k+1)/3}\varepsilon_{min}^{-1/3},
\]
which yields the bounds
\begin{align}
\cR_2(\bW) &\lesssim d^{2(k+1)/3}\varepsilon_{min}^{-4/3}\\
\cR_3(\bW) &\lesssim m^{-\frac12}d^{\frac{7k + 1}{6}}\varepsilon_{min}^{-4/3}.
\end{align}
\end{proof}

\subsubsection{Proof of Lemma~\ref{lemma: quad landscape good}}
\begin{proof}
Observe that
\begin{align*}
&\nabla \hat{L}^Q(\bW)[\tilde{\bW}, \tilde{\bW}]\\
=~& \EE_n\left[\ell'(y, f_L(\bx; \bW) + f_Q(\bx; \bW))\frac{1}{\sqrt{m}}\sum_{r=1}^m a_r\sigma''(\bw_{0, r}^T\bx)(\tilde{\bw}_r^T\bx)^2\right]\\
+~& \EE_n\left[\ell''(y, f_L(\bx; \bW) + f_Q(\bx; \bW))\left(\langle \nabla_\bW f_L(\bx; \bW) + \nabla_\bW f_Q(\bx; \bW), \tilde{\bW} \rangle)\right)^2\right],
\end{align*}
and therefore for any diagonal matrix of random signs $\bS = \text{diag}(\sigma_1, \dots, \sigma_m)$, 
\begin{align*}
&\nabla \hat{L}^Q(\bW)[\bW^*\bS, \bW^*\bS]\\
=~& 2\EE_n\left[\ell'(y, f_L(\bx; \bW) + f_Q(\bx; \bW))f_Q(\bx; \bW^*)\right]\\
+~& \EE_n\left[\ell''(y, f_L(\bx; \bW) + f_Q(\bx; \bW))\left(\frac{1}{\sqrt{m}}\sum_{r=1}^ma_r\sigma_r\left(\sigma'(\bw_{0, r}^T\bx)\bx^T\bw^*_r + \sigma''(\bw_{0, r}^T\bx)\bw_r^T\bx\bx^T\bw_r^*\right)\right)^2\right],
\end{align*}
The expectation of the second term over the random signs $\bS$ can be upper bounded by
\begin{align*}
&\EE_\bS\EE_n\left[\left(\frac{1}{\sqrt{m}}\sum_{r=1}^ma_r\sigma_r\left(\sigma'(\bw_{0, r}^T\bx)\bx^T\bw^*_r + \sigma''(\bw_{0, r}^T\bx)\bw_r^T\bx\bx^T\bw_r^*\right)\right)^2\right]\\
 =~& \EE_n\left[\frac{1}{m}\sum_{r=1}^m\left(\sigma'(\bw_{0, r}^T\bx)\bx^T\bw^*_r + \sigma''(\bw_{0, r}^T\bx)\bw_r^T\bx\bx^T\bw_r^*\right)^2\right]\\
 \lesssim~&\EE_n\left[\frac{1}{m}\sum_{r=1}^m (\bx^T\bw_r^*)^2 + (\bw_r^T\bx\bx^T\bw^*_r)^2\right]\\
 \lesssim~&\frac{1}{m}\sum_{r=1}^m \left(d\|\bw^*_r\|^2 + d^2\|\bw_r\|^2\|\bw^*_r\|^2\right)\\
 \lesssim~& \frac{d}{m}\|\bW^*\|_F^2 + \frac{d^2}{m}\|\bW\|_{2, 4}^2\|\bW^*\|_{2, 4}^2\\
 \lesssim~& m^{-\frac12}d^{\frac{k+1}{2}} + m^{-1}d^{\frac{k+3}{2}}\|\bW\|_{2, 4}^2
\end{align*}
Therefore
\begin{align}
&\mathbb{E}_{\bS}\left[\nabla^2\hat{L}^Q(\bW)[\bW^*\bS, \bW^*\bS]\right]\\
&~~\le 2\EE_n\left[\ell'(y, f_L(\bx; \bW) + f_Q(\bx; \bW))f_Q(\bx; \bW^*)\right] + C(m^{-\frac12}d^{\frac{k+1}{2}} + m^{-1}d^{\frac{k+3}{2}}\|\bW\|_{2, 4}^2)\label{eq: Hessian}
\end{align}
Next, define $\bDelta = \bW - 2\bW^*_L + \bW_L$. We have
\begin{align*}
&\langle \hat{L}^Q(\bW), \bDelta \rangle\\
&~= \EE_n\left[\ell'(y, f_L(\bx; \bW) + f_Q(\bx; \bW))\left(\frac{1}{\sqrt{m}}\sum_{r=1}^m\left(a_r\sigma'(\bw_{0, r}^T\bx)\bx^T\bDelta_r + a_r\sigma''(\bw_{0, r}^T\bx)\bw_r^T\bx\bx^T\bDelta_r\right)\right)\right]\\
&~= \EE_n\left[\ell'(y, f_L(\bx; \bW) + f_Q(\bx; \bW))\left(f_L(\bx; \bDelta) + \frac{1}{\sqrt{m}}\sum_{r=1}^ma_r\sigma''(\bw_{0, r}^T\bx)\bw_r^T\bx\bx^T\bDelta_r\right)\right]\\
&~= \EE_n\left[\ell'(y, f_L(\bx; \bW) + f_Q(\bx; \bW))\left(f_L(\bx; \bW) - 2f_L(\bx; \bW^*_L) + f_L(\bx; \bW_L) + 2f_Q(\bx; \bW)\right)\right]\\
&~~~+\EE_n\left[\ell'(y, f_L(\bx; \bW) + f_Q(\bx; \bW))\left(\frac{1}{\sqrt{m}}\sum_{r=1}^ma_r\sigma''(\bw_{0, r}^T\bx)\bw_r^T\bx\bx^T(-2(\bw^*_L)_r + (\bw_L)_r)\right)\right].
\end{align*}
The second term can be bounded in magnitude as
\begin{align*}
&\left| \EE_n\left[\ell'(y, f_L(\bx; \bW) + f_Q(\bx; \bW))\left(\frac{1}{\sqrt{m}}\sum_{r=1}^ma_r\sigma''(\bw_{0, r}^T\bx)\bw_r^T\bx\bx^T(-2(\bw^*_L)_r + (\bw_L)_r)\right)\right] \right|\\
\le~& \EE_n\left[\frac{1}{\sqrt{m}}\sum_{r=1}^m \left|\bw_r^T\bx\bx^T(-2(\bw^*_L)_r + (\bw_L)_r) \right|\right]\\
\lesssim~& \frac{d^2}{\sqrt{m}}\sum_{r=1}^m \|\bw_r\|(\|(\bw^*_L)_r\| + \|(\bw_L)_r\|)\\
\le~& d^2m^{-\frac12}\|\bW\|_F(\|\bW^*_L\|_F + \|\bW_L\|_F)
\end{align*}
Also, we have that
\begin{align*}
\left| \EE_n[\ell'(y, f_L(\bx; \bW) + f_Q(\bx; \bW))(f_L(\bx; \bW) - f_L(\bx; \bW_L))] \right| &\le \EE_n\left|f_L(\bx; \bW - \bW_L) \right|\\
&\le \EE_n\left[(f_L(\bx; \bP_{>k}\bW))^2\right]^{\frac12}\\
& = \cR_3(\bW)^{\frac12},
\end{align*}
and similarly
\[
\left| \EE_n[\ell'(y, f_L(\bx; \bW) + f_Q(\bx; \bW))(f_L(\bx; \bW^*) - f_L(\bx; \bW^*_L))] \right| \le \EE_n\left|f_L(\bx; \bW^* - \bW^*_L) \right| \lesssim \cR_3(\bW^*)^{\frac12}.
\]
Altogether, we have for some constant $C' > 0$, 
\begin{align}
\langle \hat{L}^Q(\bW), \bDelta \rangle \ge& 2\EE_n\left[\ell'(y, f_L(\bx; \bW) + f_Q(\bx; \bW))(f_Q(\bx; \bW) + f_L(\bx; \bW) - f_L(\bx; \bW^*))\right]\\
 &- C'\cdot(d^2m^{-\frac12}\|\bW\|_F(\|\bW^*_L\|_F + \|\bW_L\|_F) + \cR_3(\bW^*)^{\frac12} + \cR_3(\bW)^{\frac12}). \label{eq: grad}
\end{align}

Finally, by convexity of $\ell$ we have
\begin{align}
&\hat{L}^Q(\bW) - \hat{L}^Q(\bW^*)\\
\le~& \EE_n\left[\ell'(y, f_L(\bx; \bW) + f_Q(\bx; \bW))(f_L(\bx; \bW) + f_Q(\bx; \bW) - f_L(\bx; \bW^*) - f_Q(\bx; \bW^*))\right].\label{eq: convex}
\end{align}

Combining \eqref{eq: Hessian}, \eqref{eq: grad}, and \eqref{eq: convex}, we get that
\begin{align*}
&\mathbb{E}_{\bS}\left[\nabla^2\hat{L}^Q(\bW)[\bW^*\bS, \bW^*\bS]\right] - \langle \hat{L}^Q(\bW), \bDelta \rangle + 2(\hat{L}^Q(\bW) - \hat{L}^Q(\bW^*))\\
\le~& C\left(m^{-\frac12}d^{\frac{k+1}{2}} + m^{-1}d^{\frac{k+3}{2}}\|\bW\|_{2, 4}^2 + d^2m^{-\frac12}\|\bW\|_F(\|\bW^*_L\|_F + \|\bW_L\|_F) + \cR_3(\bW^*)^{\frac12} + \cR_3(\bW)^{\frac12}\right),
\end{align*}
as desired.
\end{proof}

\subsubsection{Proof of Corollary~\ref{cor: quad landscape good}}
\begin{proof}
First, note that we have the bounds
\begin{align*}
\|\bW^*_L\|_F &\lesssim d^{(k-1)/2}\\
\cR_3(\bW^*)^{\frac12} &\lesssim m^{-\frac14}d^{(k-1)/4}.
\end{align*}
Also, by Lemma~\ref{lemma: localization}, for $\nu$-first order stationary point $\bW$ we can bound
\begin{align*}
	\|\bW\|_F &\le m^{1/4}\|\bW\|_{2, 4} \le m^{1/4}d^{\frac{k}{3} - \frac16}\varepsilon_{min}^{-1/6} \\
	\cR_3(\bW)^{\frac12} &\le m^{-1/4}d^{\frac{7k + 1}{12}}\varepsilon_{min}^{-2/3}.
\end{align*}
Furthermore, since $\tvec{\bW_L} \in \text{span}(\bP_\le k)$, we can write
\begin{align*}
	\cR_2(\bW_L) = \tvec{\bW_L}^T\bSigma_{\le k}\tvec{\bW_L} \ge \lambda_{n_k}(\bSigma)\|\bW_L\|^2_F = \Theta(d^{1 - k})\cdot\|\bW_L\|_F^2.
\end{align*} 
Therefore
\begin{align*}
	\|\bW_L\|_F &\le d^{(k-1)/2}\cR_2(\bW)^{\frac12} \le d^{\frac{5k-1}{6}}\varepsilon_{min}^{-2/3}.
\end{align*}
Altogether, by Lemma~\ref{lemma: quad landscape good}, we can bound
\begin{align*}
	&\mathbb{E}_{\bS}\left[\nabla^2\hat{L}^Q(\bW)[\bW^*\bS, \bW^*\bS]\right] - \langle \nabla \hat{L}^Q(\bW), \bW - 2\bW^*_L + \bW_L \rangle + 2\hat{L}^Q(\bW) - 2\hat{L}^Q(\bW^*)\\
	&~~\lesssim m^{-\frac12}d^{\frac{k+1}{2}} + m^{-1}d^{\frac{k+3}{2}}\|\bW\|_{2, 4}^2 + d^2m^{-\frac12}\|\bW\|_F(\|\bW^*_L\|_F + \|\bW_L\|_F) + \cR_3(\bW^*)^{\frac12} + \cR_3(\bW)^{\frac12}\\
	&~~\lesssim m^{-\frac12}d^{\frac{k+1}{2}} + m^{-1}d^{\frac{k+3}{2}}d^{\frac{2k - 1}{3}}\varepsilon_{min}^{-1/3} + d^2m^{-\frac12}\cdot m^{\frac14}d^{\frac{k}{3} - \frac16}\varepsilon_{min}^{-1/6}\cdot d^{\frac{5k - 1}{6}}\varepsilon_{min}^{-2/3} + m^{-\frac14}d^{\frac{7k + 1}{12}}\varepsilon_{min}^{-2/3}\\
	&~~= m^{-\frac12}d^{\frac{k-1}{2}} + m^{-1}d^{\frac{7k + 7}{6}}\varepsilon_{min}^{-1/3} + m^{-\frac14}d^{\frac{7k + 10}{6}}\varepsilon_{min}^{-5/6} + m^{-\frac14}d^{\frac{7k + 1}{12}}\varepsilon_{min}^{-2/3}\\
	&~~\le \varepsilon_{min},
\end{align*}
since we have assumed $m \gtrsim d^{\frac{14k + 20}{3}}\varepsilon_{min}^{-22/3}$.
\end{proof}

\subsubsection{Proof of Lemma~\ref{lemma: quad model couples}}
\textbf{Loss Coupling.} By Lemma~\ref{lemma: loss coupling}, we can bound
\begin{align*}
\left| \hat L (\bW) - \hat L^Q (\bW) \right| &\le d^{3/2}m^{-1/4}\|\bW\|_{2, 4}^3\\
&\lesssim m^{-1/4}d^{k+1}\varepsilon_{min}^{-1/2}\\
&\le \varepsilon_{min}
\end{align*}
for $m \ge d^{4k + 4}\varepsilon_{min}^{-6}$.

Similarly,
\begin{align*}
\left| \hat L (\bW^*) - \hat L^Q (\bW^*) \right| &\le d^{3/2}m^{-1/4}\|\bW^*\|_{2, 4}^3\\
&\lesssim d^{3/2}m^{-1/4}d^{3(k-1)/4}\\
&\le \varepsilon_{min}
\end{align*}
since $m \ge d^{3(k+1)}\varepsilon_{min}^{-4}$.

\textbf{Gradient Coupling.} Next, by Lemma~\ref{lemma: gradient coupling}, we have
\begin{align*}
&\left|\langle \nabla \hat{L}(\bW), \bW - 2\bW^*_L + \bW_L \rangle - \langle \nabla \hat{L}^Q(\bW), \bW - 2\bW^*_L + \bW_L \rangle \right|\\
&~~\lesssim d^{3/2}m^{-1/4}(\|\bW\|_{2, 4}^3 + \|\bDelta\|_{2, 4}^3)\cdot\max_{i \in [n]}\left|\langle \nabla_\bW f(\bx; \bW), \bDelta \rangle \right|,
\end{align*}
where $\bDelta := \bW - 2\bW^*_L + \bW_L$. First, observe that
\begin{align*}
\|\bDelta\|_{2, 4} &\le \|\bW\|_{2, 4}  + 2\|\bW^*_L\|_{2, 4} + \|\bW_L\|_{2, 4}\\
&\le \|\bW\|_{2, 4}  + 2m^{-1/4}d^{k/2}\|\bW^*_L\|_F + m^{-1/4}d^{k/2}\|\bW_L\|_F,
\end{align*}
where we applied Lemma~\ref{lemma: linear sol small l-infty}. Plugging in the bounds for $\|\bW\|_{2, 4}$, $\|\bW_L\|_{2, 4}$, we get
\begin{align*}
	\|\bDelta\|_{2, 4} &\le d^{\frac{k}{3} - \frac16}\varepsilon_{min}^{-1/6} + m^{-1/4}d^{k/2}\cdot d^{\frac{5k - 1}{6}}\varepsilon_{min}^{-2/3}\\
	&\le d^{\frac{k}{3} - \frac16}\varepsilon_{min}^{-1/6},
\end{align*}
since $m \ge d^{4k}\varepsilon_{min}^{-2}$.

Also, by Lemma~\ref{lemma: bound max gradient},
\begin{align*}
&\max_{i \in [n]}\left|\langle \nabla_\bW f(\bx; \bW), \bDelta \rangle \right|\\
&~~\le n^{1/2}\EE_n\left[(f_L(\bx; \Delta))^2\right]^{1/2} + dn^{\frac12}m^{-1/2}\|\bW\|_F\|\bDelta\|_F\\
&~~\le n^{1/2}\EE_n\left[(f_L(\bx; \Delta))^2\right]^{1/2} + dn^{\frac12}\|\bW\|_{2, 4}\|\bDelta\|_{2, 4}.
\end{align*}
Observe that
\begin{align*}
\EE_n\left[(f_L(\bx; \Delta))^2\right] &\lesssim \EE_n\left[(f_L(\bx; \bP_{>k}\bW))^2\right] + \EE_n\left[(f_L(\bx; \bW_L))^2\right] + \EE_n\left[(f_L(\bx; \bW^*_L))^2\right]\\
&\lesssim \cR_3(\bW) + \EE_n\left[(f_L(\bx; \bW_L))^2\right] + \EE_n\left[(f_L(\bx; \bW^*_L))^2\right],
\end{align*}
where we first decomposed $\bW = \bP_{>k}\bW + \bW_L$ and then used the definition of $\cR_3$. Since $\bW_L\in \text{span}(\bP_{\le k})$, by Lemma~\ref{lemma: empirical to pop small directions} we have
\[
	\EE_n\left[(f_L(\bx; \bW_L))^2\right] \lesssim \|f_L(\bx; \bW_L)\|_{L^2}^2 = \cR_2(\bW),
\]
and similarly
\[
	\EE_n\left[(f_L(\bx; \bW^*_L))^2\right] \lesssim \EE_n\left[(f_k(\bx)^2\right] + \EE_n\left[(f_L(\bx; \bW^*_L) - f_k(\bx))^2\right] \lesssim 1.
\]
Altogether
\[
	\EE_n\left[(f_L(\bx; \bDelta))^2\right] \lesssim \cR_2(\bW) + \cR_3(\bW) \lesssim d^{2(k+1)/3}\varepsilon_{min}^{-4/3}.
\]
Plugging this back in,
\begin{align*}
\max_{i \in [n]}\left|\langle \nabla_\bW f(\bx; \bW), \bDelta \rangle \right| &\lesssim n^{1/2}d^{\frac{k+1}{3}}\varepsilon_{min}^{-2/3} + n^{\frac12}d\cdot d^{\frac{2k - 1}{3}}\varepsilon_{min}^{-1/3}\\
&\lesssim n^{\frac12}d^{\frac{2k + 2}{3}}\varepsilon_{min}^{-1/3},
\end{align*}
since $d^{-k} \ll \varepsilon_{min}$. Therefore
\begin{align*}
&\left|\langle \nabla \hat{L}(\bW), \bW - 2\bW^*_L + \bW_L \rangle - \langle \nabla \hat{L}^Q(\bW), \bW - 2\bW^*_L + \bW_L \rangle \right|\\
&~~\lesssim d^{3/2}m^{-1/4}\cdot d^{k - \frac12}\varepsilon_{min}^{-1/2}\cdot n^{\frac12}d^{\frac{2k + 2}{3}}\varepsilon_{min}^{-1/3}\\
&~~=m^{-1/4}n^{\frac12}d^{\frac{5(k+1)}{3}}\varepsilon_{min}^{-5/6}\\
&~~\le \varepsilon_{min},
\end{align*}
since $m \ge n^2d^{\frac{20(k+1)}{3}}\varepsilon_{min}^{-22/3}$.

\textbf{Hessian Coupling.} By Lemma~\ref{lemma: Hessian coupling}, we have
\begin{align*}
&\left|\EE_\bS\left[\nabla^2\hat{L}(\bW)[\bW^*\bS, \bW^*\bS]\right] -  \EE_\bS\left[\nabla^2\hat{L}^Q(\bW)[\bW^*\bS, \bW^*\bS]\right]\right|\\
 &~~\le \EE_\bS \left|\nabla^2\hat{L}(\bW)[\bW^*\bS, \bW^*\bS] -  \nabla^2\hat{L}^Q(\bW)[\bW^*\bS, \bW^*\bS]\right|\\
&~~\lesssim d^{3/2}m^{-\frac14}\left(\|\bW\|_{2, 4}^3 + \|\bW^*\|_{2, 4}^3\right)\left(d\|\bW^*\|_{2, 4}^2 + \EE_\bS\max_{i \in [n]}\left|\langle \nabla_\bW f(\bx_i; \bW), \bW^*\bS \rangle\right|^2 \right)\\
&~~~~ + d^3\|\bW\|_{2, 4}^4\|\bW^*\|_{2, \infty}^2.
\end{align*}

By Lemma~\ref{lemma: bound max gradient}, 
\begin{align*}
&\EE_\bS\max_{i \in [n]}\left|\langle \nabla_\bW f(\bx_i; \bW), \bW^*\bS \rangle\right|^2\\
&~~\lesssim n\EE_\bS\EE_n\left[(f_L(\bx; \bW^*\bS))^2\right] + d^2nm^{-1}\|\bW\|^2_F\|\bW^*\|^2_F\\
&~~\lesssim n\cdot m^{-1}\|\bW^*\|_F^2 + d^2n\|\bW\|^2_{2, 4}\|\bW^*\|^2_{2, 4}\\
&~~\lesssim nd^{\frac{7k+7}{6}}\varepsilon_{min}^{-1/3},
\end{align*}
where we used Lemma~\ref{lemma: linear term small in expectation} to bound $\EE_\bS\EE_n\left[(f_L(\bx; \bW^*\bS))^2\right]$, and then plugged in the bound for $\|\bW\|_{2, 4}$ from Lemma~\ref{lemma: localization} and the bound for $\|\bW^*\|_{2, 4}$ from Theorem~\ref{thm: main expressivity}.

By Corollary~\ref{cor: full sol l-infty bound}, we can bound $\|\bW^*\|_{2, \infty} \lesssim m^{-\frac14}d^{\frac{k-1}{2}}$. Plugging these two bounds in, along with the bound for $\|\bW\|_{2, 4}$, we get
\begin{align*}
&\left|\EE_\bS\left[\nabla^2\hat{L}(\bW)[\bW^*\bS, \bW^*\bS]\right] -  \EE_\bS\left[\nabla^2\hat{L}^Q(\bW)[\bW^*\bS, \bW^*\bS]\right]\right|\\
&~~\le d^{3/2}m^{-\frac14}d^{k-\frac12}\varepsilon_{min}^{-1/2}\cdot nd^{\frac{7k+7}{6}}\varepsilon_{min}^{-1/3} + d^3\cdot d^{\frac{4k - 2}{3}}\varepsilon_{min}^{-2/3}\cdot m^{-1/2}d^{k-1}\\
&~~\le m^{-\frac14}nd^{\frac{13(k+1)}{6}}\varepsilon_{min}^{-5/6} + m^{-1/2}d^{\frac{(7k + 4)}{3}}\varepsilon_{min}^{-2/3}\\
&~~\le \varepsilon_{min},
\end{align*}
since we've assumed $m \ge n^4d^{\frac{26(k+1)}{3}}\varepsilon_{min}^{-22/3}$.

\subsubsection{Proof of Lemma~\ref{lemma: add regularizers}}

\begin{proof}
Let us first consider the regularizer $\cR_4(\bW)$. From the proof of \cite[Corollary 3]{bai2020}, we have
\begin{align*}
\EE_\bS\nabla^2 \cR_4(\bW)[\bW^*\bS, \bW^*\bS] - \langle \nabla \cR_4(\bW), \bW  \rangle + 2\cR_4(\bW) - 2\cR_4(\bW^*) &\le  - \cR_4(\bW) + C\cR_4(\bW^*)
\end{align*}
for an absolute constant $C$. Thus
\begin{align*}
&\EE_\bS\nabla^2 \cR_4(\bW)[\bW^*\bS, \bW^*\bS] - \langle \nabla \cR_4(\bW), \bW -2\bW^*_L + \bW_L \rangle + 2\cR_4(\bW) - 2\cR_4(\bW^*)\\
\le~&  - \cR_4(\bW) + C\cR_4(\bW^*) + 8\|\bW\|_{2, 4}^4\sum \|\bw_r\|^3(2\|\{\bW^*_L\}_r\| + \|\{\bW_L\}_r\|)\\
\le~& - \cR_4(\bW) + C\cR_4(\bW^*) + 8\|\bW\|_{2, 4}^7(2\|\bW^*_L\|_{2, 4} + \|\bW_L\|_{2, 4})
\end{align*}
By Lemma~\ref{lemma: linear sol small l-infty}, we can bound
\begin{align*}
	\|\bW^*_L\|_{2, 4} + \|\bW_L\|_{2, 4} &\le m^{-1/4}d^{k/2}\left(\|\bW^*_L\|_F + \|\bW_L\|_F\right)\\
	&\le m^{-1/4}d^{k/2}d^{\frac{5k - 1}{6}}\varepsilon_{min}^{-2/3}\\
	&\le m^{-1/4}d^{\frac{4k}{3} - \frac16}\varepsilon_{min}^{-2/3}.
\end{align*}
Plugging this in, along with the bound for $\|\bW\|_{2, 4}$, yields
\begin{align*}
&\EE_\bS\nabla^2 \cR_4(\bW)[\bW^*\bS, \bW^*\bS] - \langle \nabla \cR_4(\bW), \bW -2\bW^*_L + \bW_L \rangle + 2\cR_4(\bW) - 2\cR_4(\bW^*)\\
&~~\lesssim \cR_4(\bW^*) + d^{\frac{7k}{3} - \frac{7}{6}}\varepsilon_{min}^{-7/6}\cdot m^{-1/4}d^{\frac{4k}{3} - \frac16}\varepsilon_{min}^{-2/3}\\
&~~\lesssim \cR_4(\bW^*) + m^{-1/4}d^{\frac{11k - 4}{3}}\varepsilon_{min}^{-11/6}
\end{align*}
and thus
\begin{align*}
&\EE_\bS\nabla^2 \lambda_4\cR_4(\bW)[\bW^*\bS, \bW^*\bS] - \langle \nabla \lambda_4\cR_4(\bW), \bW -2\bW^*_L + \bW_L \rangle + 2\lambda_4\cR_4(\bW) - 2\lambda_4\cR_4(\bW^*)\\
&~~\lesssim \lambda_4\cR_4(\bW^*) + \lambda_4m^{-1/4}d^{\frac{11k - 4}{3}}\varepsilon_{min}^{-11/6}\\
&~~\lesssim \varepsilon_{min} + m^{-1/4}d^{\frac{5k + 2}{3}}\varepsilon_{min}^{-5/6}\\
&~~\lesssim \varepsilon_{min},
\end{align*}
where we used $\lambda_4\cR_4(\bW^*) \le \varepsilon_{min}$ and $\lambda_4 = d^{-2(k-1)}\varepsilon_{min}$, and then used the assumption $m \ge d^{(20k +8)/3}\varepsilon_{min}^{-22/3}$.

We next deal with the other 3 regularizers. Observe that we can write
\[
	\cR_{tot}(\bW) := \lambda_1\cR_1(\bW) + \lambda_2\cR_2(\bW) + \lambda_3\cR_3(\bW) = \tvec{\bW}^T\bA\tvec{\bW}
\]
for some psd $\bA \in \RR^{md \times md}$. We get that
\begin{align*}
	&\EE_\bS\nabla^2 \cR_{tot}(\bW)[\bW^*\bS, \bW^*\bS] - \langle \nabla \cR_{tot}(\bW), \bW  - 2\bW^*_L + \bW_L \rangle + 2\cR_{tot}(\bW) - 2\cR_{tot}(\bW^*)\\
	\le~& 2\EE_\bS\cR_{tot}(\bW^*\bS)- 2\cR_{tot}(\bW) + 4\tvec{\bW}^T\bA\tvec{\bW^*_L} - 2\tvec{\bW}^T\bA\tvec{\bW_L} + 2\cR_{tot}(\bW) - 2\cR_{tot}(\bW^*)\\
	\le~& 2\EE_\bS\cR_{tot}(\bW^*\bS) + 4\tvec{\bW}^T\bA\tvec{\bW^*_L} - 2\tvec{\bW}^T\bA\tvec{\bW_L} - 2\cR_{tot}(\bW^*_L).
\end{align*}
Since $\bW_L, \bW_L^* \in \text{span}(\bP_{\le k})$, we get that
\[
	\tvec{\bW}^T\bA\tvec{\bW_L} = \lambda_2\tvec{\bW}^T\bSigma_{\le k}\tvec{\bW},
\]
and
\[
	\tvec{\bW}^T\bA\tvec{\bW^*_L} = \lambda_2\tvec{\bW}^T\bSigma_{\le k}\tvec{\bW^*_L}.
\]
Therefore
\begin{align*}
	4\tvec{\bW}^T\bA\tvec{\bW^*_L} - 2\tvec{\bW}^T\bA\tvec{\bW_L} &\le 2\lambda_2\tvec{\bW^*_L}^T\bSigma_{\le k}\tvec{\bW^*_L}\\
	&= 2\lambda_2\cR(\bW^*_L)\\
	&\lesssim \lambda_2\\
	&\lesssim \varepsilon_{min}.
\end{align*}
Also,
\begin{align*}
&\EE_\bS\cR_{tot}(\bW^*\bS)\\
&~~= \lambda_1\EE_\bS\EE_\mu\left[(f_L(\bx; \bP_{>k}\bW^*\bS))^2\right] + \lambda_2\EE_\bS\EE_\mu\left[(f_L(\bx; \bP_{\le k}\bW^*\bS))^2\right] + \lambda_3\EE_\bS\EE_n\left[(f_L(\bx; \bP_{>k}\bW^*\bS))^2\right]\\
&~~\le m^{\frac12}d^{-\frac{k-1}{2}}\varepsilon_{min}\left(\EE_\bS\EE_\mu\left[(f_L(\bx; \bW^*\bS))^2\right] + \EE_\bS\EE_n\left[(f_L(\bx; \bW^*\bS))^2\right]\right)\\
&~~\lesssim  m^{\frac12}d^{-\frac{k-1}{2}}\varepsilon_{min}\cdot \frac{1}{m}\|\bW^*\|_F^2\\
&~~\lesssim \varepsilon_{min}\cdot d^{-\frac{k-1}{2}}\|\bW^*\|_{2, 4}^2\\
&~~\lesssim \varepsilon_{min}.
\end{align*}

Therefore
\begin{equation}
	\EE_\bS\nabla^2 \cR_{tot}(\bW)[\bW^*\bS, \bW^*\bS] - \langle \nabla \cR_{tot}(\bW), \bW  - 2\bW^*_L + \bW_L \rangle + 2\cR_{tot}(\bW) - 2\cR_{tot}(\bW^*) \lesssim \varepsilon_{min}.
\end{equation}

Summing this with the effect of $\cR_4$ on the landscape, we obtain
\[
	\mathbb{E}_{\bS}\left[\nabla^2\cR(\bW)[\bW^*\bS, \bW^*\bS]\right] - \langle \nabla \cR(\bW), \bW - 2\bW^*_L + \bW_L \rangle + 2\cR(\bW) - 2\cR(\bW^*) \lesssim \varepsilon_{min},
\]
as desired.
\end{proof}

\section{Optimization Proofs}\label{app: optimization proofs}
\subsection{Geometric Properties}
In this section we show that the regularized loss $L_\lambda$ is $\ell$-smooth and $\rho$-Hessian-Lipschitz inside a norm ball.

\begin{lemma}[Loss Hessians are Lipschitz]\label{lemma: hessian lip}
\[
	\left\|\nabla^2 \hat L (\bW_1) - \nabla^2 \hat L (\bW_2)\right\|_{op} \le d^{3/2}\|\bW_1 - \bW_2\|_F.
\]
\end{lemma}
\begin{proof}
Recall that
\begin{align*}
\nabla^2 \hat{L}(\bW)[\tilde{\bW}, \tilde{\bW}] =~& \EE_n\left[\ell'(y, f(\bx; \bW))\frac{1}{\sqrt{m}}\sum_{r=1}^m a_r\sigma''((\bw_{0, r} + \bw_r)^T\bx)(\tilde{\bw}_r^T\bx)^2\right]\\
+~& \EE_n\left[\ell''(y, f(\bx; \bW))\langle \nabla_\bW f(\bx; \bW), \tilde{\bW} \rangle^2\right].
\end{align*}
Thus
\begin{align*}
&\left|\left(\nabla^2 \hat L (\bW_1) - \nabla^2 \hat L (\bW_2)\right)[\tilde \bW, \tilde \bW]\right|\\
&~~\le \frac{1}{\sqrt{m}}\sum_{r=1}^m\EE_n \left|(\tilde \bw_r^T\bx)^2\left(\ell'(y, f(\bx; \bW_1))\sigma''((\bw_{0, r} + \bw_{1, r}^T\bx)) - \ell'(y, f(\bx; \bW_2))\sigma''((\bw_{0, r} + \bw_{2, r}^T\bx))\right) \right|\\
&~~~+ \EE_n\left|\ell''(y, f(\bx; \bW_1))\langle \nabla_\bW f(\bx; \bW_1), \tilde{\bW} \rangle^2 - \ell''(y, f(\bx; \bW_2))\langle \nabla_\bW f(\bx; \bW_2), \tilde{\bW} \rangle^2 \right|
\end{align*}
To bound the first term, since $\ell', \sigma''$ are both Lipschitz and can be upper bounded by 1, we get
\begin{align*}
&\left|\ell'(y, f(\bx; \bW_1))\sigma''((\bw_{0, r} + \bw_{1, r}^T\bx)) - \ell'(y, f(\bx; \bW_2))\sigma''((\bw_{0, r} + \bw_{2, r}^T\bx))\right|\\
&~~\le \left|\ell'(y, f(\bx; \bW_1)) - \ell'(y, f(\bx; \bW_2))\right| + \left|\sigma''((\bw_{0, r} + \bw_{1, r}^T\bx)) - \sigma''((\bw_{0, r} + \bw_{2, r}^T\bx))\right|\\
&~~\le \left|f(\bx; \bW_1) - f(\bx; \bW_2)\right| + \left|(\bw_{1, r} - \bw_{2, r})^T\bx\right|\\
&~~\le \left|(\bw_{1, r} - \bw_{2, r})^T\bx\right| + \frac{1}{\sqrt{m}}\sum_{s=1}^m\left|\sigma((\bw_{0, s} + \bw_{1, s}^T\bx)) - \sigma((\bw_{0, s} + \bw_{2, s}^T\bx))\right|\\
&~~\le \left|(\bw_{1, r} - \bw_{2, r})^T\bx\right| + \frac{1}{\sqrt{m}}\sum_{s=1}^m\left|(\bw_{1, s} - \bw_{2, s})^T\bx\right|\\
\end{align*}
Therefore
\begin{align*}
&\frac{1}{\sqrt{m}}\sum_{r=1}^m\EE_n \left|(\tilde \bw_r^T\bx)^2\left(\ell'(y, f(\bx; \bW_1))\sigma''((\bw_{0, r} + \bw_{1, r}^T\bx)) - \ell'(y, f(\bx; \bW_2))\sigma''((\bw_{0, r} + \bw_{2, r}^T\bx))\right) \right|\\
&~~\le\frac{1}{\sqrt{m}}\EE_n\left|(\tilde \bw_r^T\bx)^2(\bw_{1, r} - \bw_{2, r})^T\bx\right| + \frac{1}{m}\sum_{r,s=1}^m\EE_n\left|(\tilde \bw_r^T\bx)^2(\bw_{1, s} - \bw_{2, s})^T\bx \right|\\
&~~\lesssim \frac{d^{3/2}}{\sqrt{m}}\sum_{r=1}^m\|\tilde \bw_r\|^2\|\bw_{1, r} - \bw_{2, r}\| + \frac{d^{3/2}}{m}\|\tilde \bW\|_F^2\sum_{s=1}^m\|\bw_{1, s} - \bw_{2, s}\|\\
&~~\lesssim \frac{d^{3/2}}{\sqrt{m}}\|\tilde\bW\|_F^2\|\bW_1 - \bW_2\|_F.
\end{align*}

To bound the second term, we have
\begin{align*}
&\left|\ell''(y, f(\bx; \bW_1))\langle \nabla_\bW f(\bx; \bW_1), \tilde{\bW} \rangle^2 - \ell''(y, f(\bx; \bW_2))\langle \nabla_\bW f(\bx; \bW_2), \tilde{\bW} \rangle^2 \right|\\
&~~\le \left|\langle \nabla_\bW f(\bx; \bW_1), \tilde{\bW} \rangle^2 -  \langle \nabla_\bW f(\bx; \bW_2), \tilde{\bW} \rangle^2\right| + \langle \nabla_\bW f(\bx; \bW_2), \tilde{\bW} \rangle^2\left|\ell''(y, f(\bx; \bW_1)) - \ell''(y, f(\bx; \bW_2))\right|\\
&~~\le \|\tilde \bW\|_F^2\|\nabla_\bW f(\bx; \bW_1) - \nabla_\bW f(\bx; \bW_2)\|\|\nabla_\bW f(\bx; \bW_1)+ \nabla_\bW f(\bx; \bW_2)\|\\
&~~~~ + \|\tilde \bW\|_F^2\|\nabla_\bW f(\bx; \bW_2)\|^2|f(\bx; \bW_1) - f(\bx; \bW_2)|.
\end{align*}
Since $\sigma'$ is bounded by 1, we can bound $\|\nabla_\bW f(\bx; \bW_2)\| \le \sqrt{d}$. Next, we have
\begin{align*}
\|\nabla_\bW f(\bx; \bW_1) - \nabla_\bW f(\bx; \bW_2)\|_F^2 &\le \frac{1}{m}\|\sigma'((\bw_{0, r}^ + \bw_{1, r})^T\bx)\bx -  \sigma'((\bw_{0, r}^ + \bw_{2, r})^T\bx)\bx\|^2\\
&\le \frac{d}{m}\sum_{r=1}^m\left|(\bw_{1, r} - \bw_{2, r})^T\bx\right|^2\\
&\le \frac{d^2}{m}\|\bW_1 - \bW_2\|_F^2.
\end{align*}
Finally,
\begin{align*}
|f(\bx; \bW_1) - f(\bx; \bW_2)| &\le \frac{1}{\sqrt{m}}\sum_{r=1}^m\left|(\bw_{1, r} - \bw_{2, r})^T\bx \right|\\
&\le \sqrt{d}\|\bW_1 - \bW_2\|_F.
\end{align*}
Altogether,
\begin{align*}
&\left|\ell''(y, f(\bx; \bW_1))\langle \nabla_\bW f(\bx; \bW_1), \tilde{\bW} \rangle^2 - \ell''(y, f(\bx; \bW_2))\langle \nabla_\bW f(\bx; \bW_2), \tilde{\bW} \rangle^2 \right|\\
&~~\le \|\tilde \bW\|_F^2d^{3/2}\|\bW_1 - \bW_2\|_F.
\end{align*}
Therefore
\[
	\left|\left(\nabla^2 \hat L (\bW_1) - \nabla^2 \hat L (\bW_2)\right)[\tilde \bW, \tilde \bW]\right| \le \|\tilde \bW\|_F^2d^{3/2}\|\bW_1 - \bW_2\|_F,
\]
so
\[
	\left\|\nabla^2 \hat L (\bW_1) - \nabla^2 \hat L (\bW_2)\right\|_{op} \le d^{3/2}\|\bW_1 - \bW_2\|_F.
\]
\end{proof}

\begin{lemma}[Regularized loss is Hessian-Lipschitz]\label{lemma: reg hessian lip}
The regularized loss $L_\lambda$ is $O(\lambda_4\Gamma^5)$-Hessian-Lipschitz inside the region $\{\bW \mid \|\bW\|_F \le \Gamma\}$
\end{lemma}
\begin{proof}
For $i = 1, 2, 3$, the regularizer $\cR_i$ is a convex quadratic, so $\nabla^2\cR_i(\bW_1) = \nabla^2\cR_i(\bW_2)$. As for the regularizer $\cR_4$, we have
\[
	\nabla^2 \cR_4(\bW)[\tilde \bW, \tilde \bW] = 32\left(\sum_{r=1}^m\|\bw_r\|^2\bw_r^T\tilde \bw_r\right)^2 + 8\|\bW\|_{2, 4}^4\left(\sum_{r=1}^m 2(\bw_r^T\tilde \bw)^2 + \|\bw_r\|^2\|\tilde\bw_r\|^2\right).
\]
Therefore, for $\|\tilde \bW\|_F = 1$,
\begin{align*}
&\left|\nabla^2 \cR_4(\bW_1)[\tilde \bW, \tilde \bW] - \nabla^2 \cR_4(\bW_2)[\tilde \bW, \tilde \bW]\right|\\
&~~\le 32\left(\sum_{r=1}^m(\|\bw_{1, r}\|^2\bw_{1, r}^T + \|\bw_{2, r}\|^2\bw_{2, r}^T)\tilde\bw_r\right)\left(\sum_{r=1}^m(\|\bw_{1, r}\|^2\bw_{1, r}^T - \|\bw_{2, r}\|^2\bw_{2, r}^T)\tilde\bw_r\right)\\
&~~~~+ 8\sum_{r=1}^m2\|\bW_1\|^4_{2, 4}(\bw_{1, r}^T\tilde \bw_r)^2 - 2\|\bW_2\|^4_{2, 4}(\bw_{2, r}^T\tilde \bw_r)^2 + (\|\bW_1\|_{2, 4}^4\|\bw_{1, r}\|^2 - \|\bW_2\|_{2, 4}^4\|\bw_{2, r}\|^2)\|\tilde \bw_r\|^2.
\end{align*}

We can bound the first term using
\begin{align*}
\left|\sum_{r=1}^m(\|\bw_{1, r}\|^2\bw_{1, r}^T + \|\bw_{2, r}\|^2\bw_{2, r}^T)\tilde\bw_r\right| &\le \sum_{r=1}^m (\|\bw_{1, r}\|^3 + \|\bw_{2, r}\|^3)\|\tilde \bw_r\|\\
&\le \left(\|\bW_1\|_{2, 6}^3 + \|\bW_2\|_{2, 6}^3\right)\|\tilde \bW\|_F\\
&= \|\bW_1\|_{2, 6}^3 + \|\bW_2\|_{2, 6}^3.
\end{align*}
and
\begin{align*}
	\left|\sum_{r=1}^m(\|\bw_{1, r}\|^2\bw_{1, r}^T - \|\bw_{2, r}\|^2\bw_{2, r}^T)\tilde\bw_r\right| &\le \sum_{r=1}^m\left\|\|\bw_{1, r}\|^2\bw_{1, r} - \|\bw_{2, r}\|^2\bw_{2, r} \right\|\|\tilde \bw_r\|\\
	&\le \sum_{r=1}^m\left\|\|\bw_{1, r}\|^2\bw_{1, r} - \|\bw_{2, r}\|^2\bw_{2, r} \right\|\\
	&\lesssim \sum_{r=1}^m \|\bw_{1, r} - \bw_{2, r}\|(\|\bw_{1, r}\|^2 + \|\bw_{2, r}\|^2)\\
	&\le \|\bW_1 - \bW_2\|_F(\|\bW_1\|_{2, 4}^2 + \|\bW_2\|_{2, 4}^2).
\end{align*}

For the second term, we bound
\begin{align*}
&\left|\sum_{r=1}^m\|\bW_1\|^4_{2, 4}(\bw_{1, r}^T\tilde \bw_r)^2 - \|\bW_2\|^4_{2, 4}(\bw_{2, r}^T\tilde \bw_r)^2\right|\\
&~~\le \sum_{r=1}^m \left\|\|\bW_1\|^4_{2, 4}\bw_{1, r}\bw_{1, r}^T - \|\bW_2\|^4_{2, 4}\bw_{2, r}\bw_{2, r}^T\right\|_{op}\\
&~~\le\sum_{r=1}^m\left(\|\bW_1\|^2_{2, 4}\|\bw_{1, r}\| + \|\bW_2\|^2_{2, 4}\|\bw_{2, r}\|\right)\left(\left\|\|\bW_1\|^2_{2, 4}\bw_{1, r} - \|\bW_2\|^2_{2, 4}\bw_{2, r}\right\|\right)\\
&~~\lesssim \sum_{r=1}^m\left(\|\bW_1\|^2_{2, 4}\|\bw_{1, r}\| + \|\bW_2\|^2_{2, 4}\|\bw_{2, r}\|\right)\left(\|\bW_1\|_{2, 4}^2 + \|\bW_2\|_{2, 4}^2\right)\left\|\bw_{1, r} - \bw_{2, r}\right\|\\
&~~\lesssim \left(\|\bW_1\|_{2, 4}^2 + \|\bW_2\|_{2, 4}^2\right)\left(\|\bW_1\|_{2, 4}^2\|\bW_1\|_F + \|\bW_2\|_{2, 4}^2\|\bW_2\|_F\right)\|\bW_1 - \bW_2\|_F.
\end{align*}

Finally, we bound
\begin{align*}
&\sum_{r=1}^m\left|\|\bW_1\|_{2, 4}^4\|\bw_{1, r}\|^2 - \|\bW_2\|_{2, 4}^4\|\bw_{2, r}\|^2\right|\|\tilde \bw_r\|^2\\
&~~\le\sum_{r=1}^m\left|\|\bW_1\|_{2, 4}^4\|\bw_{1, r}\|^2 - \|\bW_2\|_{2, 4}^4\|\bw_{2, r}\|^2\right|\\
&~~\le\sum_{r=1}^m\left|\|\bW_1\|_{2, 4}^4 - \|\bW_2\|_{2, 4}^4\right|\|\bw_{1, r}\|^2 +\|\bW_2\|_{2, 4}^4\left|\|\bw_{1, r}\|^2 - \|\bw_{2, r}\|^2\right|\\
&~~\le \|\bW_1\|_F^2\left|\|\bW_1\|_{2, 4}^4 - \|\bW_2\|_{2, 4}^4\right| + \|\bW_2\|_{2, 4}^4\sum_{r=1}^m(\|\bw_{1, r}\| + \|\bw_{2, r}\|)\|\bw_{1, r} - \bw_{2, r}\|\\
&~~\lesssim \|\bW_1\|_F^2\|\bW_1 - \bW_2\|_{2, 4}(\|\bW_1\|_{2, 4}^3 + \|\bW_2\|_{2, 4}^3) + \|\bW_2\|_{2, 4}^4(\|\bW_1\|_F + \|\bW_2\|_F)\|\bW_1 - \bW_2\|_F\\
&~~\le \left(\|\bW_1\|_F^2(\|\bW_1\|_{2, 4}^3 + \|\bW_2\|_{2, 4}^3) + \|\bW_2\|_{2, 4}^4(\|\bW_1\|_F + \|\bW_2\|_F)\right)\|\bW_1 - \bW_2\|_F.
\end{align*}

Altogether, when $\|\bW_1\|_F, \|\bW_2\|_F \le \Gamma$, and using $\|\bW\|_{2, 2k} \le \|\bW\|_F$ for $k \ge 1$, we get that
\begin{align*}
\left|\nabla^2 \cR_4(\bW_1)[\tilde \bW, \tilde \bW] - \nabla^2 \cR_4(\bW_2)[\tilde \bW, \tilde \bW]\right| &\lesssim \Gamma^5\|\bW_1 - \bW_F\|.
\end{align*}
Therefore $L_\lambda$ is $O(\Gamma^5)$-Hessian-Lipschitz.
\end{proof}

\begin{lemma}[Regularized loss is smooth]\label{lemma: reg smooth}
The regularized loss $L_\lambda$ is $O(\lambda_4\Gamma^6 + m^{1/2})$-smooth.
\end{lemma}
\begin{proof}
Recall that
\begin{align*}
\nabla^2 \hat{L}(\bW)[\tilde{\bW}, \tilde{\bW}] =~& \EE_n\left[\ell'(y, f(\bx; \bW))\frac{1}{\sqrt{m}}\sum_{r=1}^m a_r\sigma''((\bw_{0, r} + \bw_r)^T\bx)(\tilde{\bw}_r^T\bx)^2\right]\\
+~& \EE_n\left[\ell''(y, f(\bx; \bW))\langle \nabla_\bW f(\bx; \bW), \tilde{\bW} \rangle^2\right],
\end{align*}
and thus, for $\|\tilde \bW\|_F = 1$, 
\begin{align*}
\left| \nabla^2 \hat{L}(\bW)[\tilde{\bW}, \tilde{\bW}] \right| &\le \frac{1}{\sqrt{m}}\sum_{r=1}^m\EE_n[(\tilde \bw_r^T\bx)^2] + \EE_n \|\nabla_\bW f(\bx; \bW)\|^2\|\tilde \bW\|^2\\
&\lesssim \frac{d}{\sqrt{m}}\|\tilde \bW\|_F^2 + d\|\tilde \bW\|^2\\
&\lesssim d.
\end{align*}
Therefore $\hat L$ is $O(d)$-smooth.

Next, consider the regularizers. $\cR_1, \cR_2, \cR_3$ are all convex quadratics, and since $\|\varphi(\bx)\| \le \sqrt{d}$, we can upper bound the smoothness of $\cR_1, \cR_2, \cR_3$ by $d$. Therefore the smoothness of $\lambda_1\cR_1 + \lambda_2\cR_2 + \lambda_3\cR_3$ is at most $m^{1/2}d^{-\frac{k-1}{2}}\varepsilon_{min}\cdot d \lesssim m^{1/2}$.

Finally, we consider $\cR_4$. For $\|\tilde \bW\|_F = 1$, we can bound
\begin{align*}
\nabla^2 \cR_4(\bW)[\tilde \bW, \tilde \bW] &= 32\left(\sum_{r=1}^m\|\bw_r\|^2\bw_r^T\tilde \bw_r\right)^2 + 8\|\bW\|_{2, 4}^4\left(\sum_{r=1}^m 2(\bw_r^T\tilde \bw)^2 + \|\bw_r\|^2\|\tilde\bw_r\|^2\right)\\
&\le 32\left(\sum_{r=1}^m\|\bw_r\|^3\|\tilde \bw_r\|\right)^2 + 24\|\bW\|_{2, 4}^4\sum_{r=1}^m\|\bw_r\|^2\|\tilde \bw_r\|^2\\
&\le 32\|\bW\|_{2, 6}^6\|\tilde \bW\|_F^2 + 24\|\bW\|_{2, 4}^4\|\bW\|_F^2\\
&\le 32\|\bW\|_{2, 6}^6 + 24\|\bW\|_{2, 4}^4\|\bW\|_F^2\\
&\le 56\|\bW\|_F^6.
\end{align*}
Therefore $L_\lambda$ is $O(\lambda_4\Gamma^6 + m^{1/2})$-smooth.
\end{proof}

\subsection{Proof of Theorem~\ref{thm: main optimization}}
We prove the following formal version of Theorem~\ref{thm: main optimization}.\\
\begin{theorem}\label{thm: optimization formal}
For $\nu, \gamma$, choose $\eta = \frac{c}{\lambda_4m^3}$ for sufficiently small constant $c$ and define $\tilde \varepsilon = \min(\nu, \gamma^2m^{-5/2}), \sigma = \tilde\Theta(\tilde \varepsilon)$. Then with probability $1-d^{-8}$, perturbed gradient descent reaches a $(\nu, \gamma)$-SOSP within $T = \Tilde O(m^3\tilde \varepsilon^{-2})$ timesteps.
\end{theorem}

\begin{proof}
We follow the same strategy as~\cite[Theorem 8]{jin2017}. We first show that perturbed gradient descent stays in a bounded region. Then, we can use the smoothness and Hessian-Lipschitz parameters in this region for the purpose of a convergence result. Throughout, we condition on the high probability event where the construction in Theorem~\ref{thm: main expressivity} holds.

Let $\Gamma = m^{1/2}$. I claim that perturbed gradient descent stays in the region $\{\|\bW\| \mid \|\bW\| \le \Gamma\}$. We prove the claim by induction. $\|\bW^0\| = 0$ so clearly the base case holds.

Assume that $\|\bW^t\| \le \Gamma$. The gradient descent update on $\bW$ is
\begin{align*}
\bW^{t+1} \leftarrow \bW^t - \eta \nabla \hat L(\bW^t) - \eta \nabla\left(\lambda_1 \cR_1(\bW^t) + \lambda_2 \cR_2(\bW^t) + \lambda_3 \cR_3(\bW^t)\right) - \eta\lambda_4 \nabla \cR_4(\bW^t).
\end{align*}
Observe that the learning rate is $\eta = \frac{c}{\lambda_4m^3} = \frac{c}{\lambda_4 \Gamma^6}$ for a sufficiently small constant $c$. We can bound $\|\nabla \hat L(\bW^t)\| \le \sqrt{d}$ and, as in the proof of the previous lemma,
\[
	\left\|\nabla\left(\lambda_1 \cR_1(\bW^t) + \lambda_2 \cR_2(\bW^t) + \lambda_3 \cR_3(\bW^t)\right) \right\| \le m^{1/2}\|\bW^t\|_F.
\]
Therefore
\begin{align*}
	\left\|\eta \nabla \hat L(\bW^t) + \eta \nabla\left(\lambda_1 \cR_1(\bW^t) + \lambda_2 \cR_2(\bW^t) + \lambda_3 \cR_3(\bW^t)\right) \right\| &\lesssim \lambda_4^{-1}\Gamma^{-6}\sqrt{d} + \lambda_4^{-1}\Gamma^{-5}m^{1/2}\\
	&\lesssim \lambda_4^{-1}\Gamma^{-5}m^{1/2}
\end{align*}

Finally, since the $\bw_r$ component of $\nabla \cR_4(\bW)$ is $8\|\bW\|_{2, 4}^4\|\bw_r\|^2\bw_r$, we have
\begin{align*}
\{\bW^t - \eta\lambda_4\nabla \cR_4(\bW^t)\}_r &= \bw^t_r\left(1 - 8\lambda_4\eta\|\bW^t\|_{2, 4}^4\|\bw^t_r\|^2\right)\\
&= \bw^t_r\left(1 - 8c\Gamma^{-6}\|\bW^t\|_{2, 4}^4\|\bw^t_r\|^2\right)
\end{align*} 
Note that
\begin{align*}
8c\Gamma^{-6}\|\bW^t\|_{2, 4}^4\|\bw^t_r\|^2 &\le 8c \le 1
\end{align*}
for $c$ small enough. We then have
\begin{align*}
\|\bW^t - \eta\lambda_4\nabla \cR_4(\bW^t)\|_F^2 &= \sum_r \|\bw^t_r\|^2\left(1 - 8c\Gamma^{-6}\|\bW^t\|_{2, 4}^4\|\bw^t_r\|^2\right)^2\\
&\le \sum_r \|\bw^t_r\|^2\left(1 - 8c\Gamma^{-6}\|\bW^t\|_{2, 4}^4\|\bw^t_r\|^2\right)\\
&= \|\bW^t\|_F^2 - 8c\Gamma^{-6}\|\bW^t\|_{2, 4}^8\\
&\le \|\bW^t\|_F^2 - 8c\Gamma^{-6}m^{-2}\|\bW^t\|_F^8.
\end{align*}

We split into two cases. If $\|\bW^t\|_F \le \Gamma/2$, then
\[
	\|\bW^t - \eta\lambda_4\nabla \cR_4(\bW^t)\|_F \le \Gamma/2
\]
and
\begin{align*}
	\left\|\eta \nabla \hat L(\bW^t) + \eta \nabla\left(\lambda_1 \cR_1(\bW^t) + \lambda_2 \cR_2(\bW^t) + \lambda_3 \cR_3(\bW^t)\right) \right\| &\lesssim  d^{2(k-1)}\varepsilon_{min}^{-1}\Gamma^{-5}m^{1/2}\\
	&\le \Gamma/4,
\end{align*}
so $\|\bW^{t+1}\|_F \le 3\Gamma/4$.

Otherwise, $\Gamma \ge \|\bW^t\|_F \ge \Gamma/2$, so
\begin{align*}
\|\bW^t - \eta\lambda_4\nabla \cR_4(\bW^t)\|_F &\le \sqrt{\Gamma^2 - \frac{c}{32}\Gamma^2m^{-2}}.\\
&\le \Gamma(1 - \frac{c}{64}m^{-2})
\end{align*}

Then
\begin{align*}
\|\bW^{t+1}\|_F &\le \Gamma(1 - \frac{c}{64}m^{-2}) + d^{2(k-1)}\varepsilon_{min}^{-1}\Gamma^{-5}m^{1/2}\\
&\le \Gamma\left(1 - \frac{c}{128}m^{-2}\right),
\end{align*}
since $\Gamma^6 = m^3 \gtrsim m^{5/2}d^{2(k-1)}\varepsilon_{min}^{-1}$.

Finally, the perturbation moves at most $\eta\|\bXi^t\|_F$, and since $\EE\|\bXi^t\|_F = \sigma$, with probability $1 - d^{-9}$ each of the $T$ perturbations is bounded by  $\Gamma m^{-3}\eta^{-1} \gg \sigma$
Therefore even after the perturbation we have $\|\bW^{t+1}\|_F \le \Gamma$, completing the induction step.

Lemmas~\ref{lemma: reg hessian lip}, \ref{lemma: reg smooth} tell us $L_\lambda$ is $O(\lambda_4m^{5/2})$ Hessian Lipschitz and $O(\lambda_4m^3)$ smooth throughout the entire gradient descent trajectory.

Our goal is to converge to a $(\nu, \gamma)$-SOSP; this is equivalent to converging to a $\tilde \varepsilon$-SOSP as defined in \citep{jin2017, jin2019}, with $\tilde \varepsilon := \min(\nu, \frac{\gamma^2}{\lambda_4m^{5/2}}) \ge \min(\nu, \gamma^{2}m^{-5/2})$. Therefore by \citep{jin2017, jin2019}, with probability $1 - d^{-9}$ perturbed gradient descent on the regularized loss with learning rate $\eta = \frac{c}{\lambda_4m^3}$ and perturbation radius $\sigma = \Tilde \Theta(\tilde \varepsilon)$ will encounter an $\tilde\varepsilon$-SOSP in at most $T = \Tilde O\left(\lambda_4m^3\tilde \varepsilon^{-2}\right) \le \Tilde O\left(m^3\tilde \varepsilon^{-2}\right)$ timesteps. Union bounding over the high probability events, this occurs with probability $1 - 3d^{-9} \ge 1 - d^{-8}$.
\end{proof}

\section{Generalization Proofs}\label{app: generalization proofs}
\subsection{Proof of Theorem~\ref{thm: main generalization}}
Recall the definition of the empirical Rademacher complexity:\\
\begin{definition}
Let $\cF$ be a class of functions from $\RR^d$ to $\RR$. Given a dataset $\cD = \{\bx_1, \dots, \bx_n\}$, the \emph{empirical Rademacher complexity} of $\cF$ is defined as
\begin{equation}
\cR_\cD(\cF) := \EE_{\sigma \in \{\pm1\}^n}\left[\sup_{f \in \cF}\frac{1}{n}\sum_{i=1}^n \sigma_if(\bx_i)\right].
\end{equation}
\end{definition}

We next show that the Rademacher complexities of the linear term and the quadratic term can be bounded.\\

\begin{lemma}[Rademacher complexity of linear term]\label{lemma: linear rademacher}
Let $\cW \subset \RR^{m \times d}$, and define the function class $\cF^L_k(\cW) := \{\bx \mapsto f^L(\bx; \bP_{\le k}\bW) : \bW \in \cW\}$. Then, with probability $1 - d^{-9}$ over the draw of $\cD$,
\begin{equation}
	\cR_\cD(\cF^L_k(\cW)) \lesssim \sqrt{\frac{d^k}{n}}\cdot\sup_{\bW \in \cW}\|f_L(\bx; \bP_{\le k}\bW)\|_{L^2}^2.
\end{equation}
\end{lemma}

\begin{lemma}[Rademacher complexity of quadratic term]\label{lemma: quad rademacher}
Let $\cW \subset \RR^{m \times d}$, and define the function class $\cF^Q(\cW) := \{\bx \mapsto f^Q(\bx; \bW) : \bW \in \cW\}$. Then, with probability $1 - d^{-9}$ over the draw of $\cD$,
\begin{equation}
	\cR_\cD(\cF^Q(\cW)) \lesssim \sqrt{\frac{d}{mn}}\cdot \sup_{\bW \in \cW}\|\bW\|_F^2.
\end{equation}
\end{lemma}

Lemma~\ref{lemma: linear rademacher} is presented in Appendix~\ref{app: linear rademacher proof}; Lemma~\ref{lemma: quad rademacher} follows directly from ~\cite[Lemma 5, Theorem 6]{bai2020}.

Equipped with these Rademacher complexity lemmas, we can now prove the main generalization result.
\begin{proof}[Proof of Theorem~\ref{thm: main generalization}]
By Lipschitzness of the loss and Lemma~\ref{lemma: function values}, we can bound
\begin{align*}
	L(\bW) &= \EE_\mu[\ell(y, f(\bx; \bW))]\\
	&\le \EE_\mu[\ell(y, f_Q(\bx; \bW) + f_L(\bx; \bW))] + \EE_\mu|f(\bx; \bW) - f_Q(\bx; \bW) - f_L(\bx; \bW)|\\
	&\le \EE_\mu[\ell(y, f_Q(\bx; \bW) + f_L(\bx; \bW))] + m^{-\frac12}\sum_{r=1}^m\EE_\mu|\bw_r^T\bx|^3\\
	&\le \EE_\mu[\ell(y, f_Q(\bx; \bW) + f_L(\bx; \bW))] + Cm^{-\frac14}\|\bW\|_{2, 4}^3.
\end{align*}

Again by Lipschitzness, we can bound
\begin{align*}
\EE_\mu[\ell(y, f_Q(\bx; \bW) + f_L(\bx; \bW))] &\le \EE_\mu[\ell(y, f_Q(\bx; \bW) + f_L(\bx; \bP_{\le k}\bW))] + \EE_\mu |f_L(x; \bP_{>k}\bW)|\\
&\le \EE_\mu[\ell(y, f_Q(\bx; \bW) + f_L(\bx; \bP_{\le k}\bW))] + \|f_L(\bx; \bP_{> k}\bW)\|_{L^2}.
\end{align*}

Similarly, we can lower bound
\begin{align*}
\hat L(\bW) \ge \EE_n[\ell(y, f_Q(\bx; \hat\bW) + f_L(\bx; \bP_{\le k}\hat\bW))] - Cm^{-\frac14}\|\bW\|_{2, 4}^3 - \left(\EE_n[(f_L(\bx; \bP_{>k}))^2]\right)^{\frac12}.
\end{align*}

Since $L_\lambda(\hat \bW) \le C\varepsilon_{min}$, the value of each regularizer satisfies $\cR_i(\hat \bW) \le C\lambda_i^{-1}\varepsilon_{min}$. By our choice of $(\lambda_1, \lambda_2, \lambda_3, \lambda_4)$, we have
\begin{align*}
\cR_1(\hat \bW) &= \|f_L(\cdot; \bP_{>k}\bW)\|_{L^2}^2 \lesssim m^{-\frac12}d^{\frac{k-1}{2}},\\
\cR_2(\hat \bW) &= \|f_L(\cdot; \bP_{\le k}\bW)\|_{L^2}^2 \lesssim 1,\\
\cR_3(\hat \bW) &= \EE_n\left[(f_L(\bx; \bP_{>k}\bW))^2\right] \lesssim m^{-\frac12}d^{\frac{k-1}{2}},\\
\cR_4(\hat \bW) &= \|\bW\|_{2, 4}^8 \lesssim d^{2(k-1)}.
\end{align*}
Therefore 
\[
	L(\hat \bW) \le \EE_\mu[\ell(y, f_Q(\bx; \bW) + f_L(\bx; \bP_{\le k}\bW))] + C\cdot m^{-\frac14}d^{\frac{3(k-1)}{4}} + m^{-\frac14}d^{\frac{k-1}{4}},
\]

and similarly
\[
	\hat L(\hat \bW) \ge \EE_n[\ell(y, f_Q(\bx; \bW) + f_L(\bx; \bP_{\le k}\bW))] - C\cdot m^{-\frac14}d^{\frac{3(k-1)}{4}} - m^{-\frac14}d^{\frac{k-1}{4}}.
\]
Since $\hat L(\hat \bW) \le C\varepsilon_{min}$, we thus have that
\[
	\EE_n[\ell(y, f_Q(\bx; \bW) + f_L(\bx; \bP_{\le k}\bW))] \lesssim \varepsilon_{min} + m^{-\frac14}d^{\frac{3(k-1)}{4}}.
\]

Next, define the set
\[
	\cW := \{\bW \in \RR^{m \times d} : \|\bW\|_{2, 4} \le Cd^{\frac{k-1}{4}}, \|f_L(\cdot; \bP_{\le k}\bW)\|_{L^2}^2 \le C\}.
\]
By construction, we have $\hat \bW \in \cW$. Furthermore, define the function class \[\cL := \{(\bx, y) \rightarrow \ell(y, f_Q(\bx; \bW) + f_L(\bx; \bP_{\le k}\bW)) : \bW \in \cW\}.\] 
By the standard empirical Rademacher complexity bound, with probability $1 - d^{-9}$ over the draw of $\mathcal{D}$,
\[
	\EE_\mu[\ell(y, f_Q(\bx; \hat\bW) + f_L(\bx; \bP_{\le k}\hat\bW))] - \EE_n[\ell(y, f_Q(\bx; \hat\bW) + f_L(\bx; \bP_{\le k}\hat\bW))]] \le 2\cR_\mathcal{D}(\cL),
\]
By the Rademacher contraction Lemma~\cite{wainwright}, since $\ell$ is 1-Lipschitz we can bound (conditioning on Lemmas~\ref{lemma: linear rademacher}, \ref{lemma: quad rademacher})
\begin{align*}
\cR_\cD(\cL) &\le \cR_\cD(\cF^L_k(\cW) + \cF^Q(\cW)) + n^{-\frac12}\\
&\le \cR_\cD(\cF^L_k(\cW)) + \cR_\cD(\cF^Q(\cW)) + n^{-\frac12}\\
&\lesssim \sqrt{\frac{d^k}{n}}\cdot\sup_{\bW \in \cW}\|f_L(\bx; \bP_{\le k}\bW)\|_{L^2}^2 + \sqrt{\frac{d}{mn}}\cdot \sup_{\bW \in \cW}\|\bW\|_F^2 + n^{-\frac12}\\
&\lesssim \sqrt{\frac{d^k}{n}} + \sqrt{\frac{d}{n}}\cdot \sup_{\bW \in \cW}\|\bW\|_{2, 4}^2 + n^{-\frac12}\\
&\lesssim \sqrt{\frac{d^k}{n}} + \sqrt{\frac{d}{n}}d^{\frac{k-1}{2}} + n^{-\frac12}\\
&\lesssim \sqrt{\frac{d^k}{n}}.
\end{align*}
Union bounding over the high probability events, with probability $1 - 3d^{-9} \ge 1 - d^{-8}$, we have 
\begin{align*}
L(\hat \bW) &\le \EE_\mu[\ell(y, f_Q(\bx; \hat\bW) + f_L(\bx; \bP_{\le k}\hat\bW))] + C\cdot m^{-\frac14}d^{\frac{3(k-1)}{4}}\\
&\le \EE_n[\ell(y, f_Q(\bx; \hat\bW) + f_L(\bx; \bP_{\le k}\hat\bW))]] + 2\cR_\cD(\cL) + C\cdot m^{-\frac14}d^{\frac{3(k-1)}{4}}\\
&\lesssim \varepsilon_{min} + m^{-\frac14}d^{\frac{3(k-1)}{4}} + \sqrt{\frac{d^k}{n}}\\
&\lesssim \varepsilon_{min} + \sqrt{\frac{d^k}{n}}.
\end{align*}
since $m \gtrsim d^{3(k-1)}\varepsilon_{min}^{-4}$.
\end{proof}

\subsection{Proof of Lemma~\ref{lemma: linear rademacher}}\label{app: linear rademacher proof}
\begin{proof}
Recall that we can write $f_L(\bx; \bW) = \varphi(\bx)^T\text{vec}(\bW)$, where $\varphi(\bx)$ is the NTK featurization map. Also, recall
\[
	\bSigma_{\le n_k} := \mathbb{E}_\mu\left[\varphi(\bx)^T\bP_{\le k}\varphi(\bx)\right] = \sum_{i=1}^{n_k}\lambda_i\bv_i\bv_i^T.
\]
We can then bound
\begin{align*}
	\cR_\cD(\cF^L_k(\cW)) &= \EE_\sigma\left[\sup_{\bW \in \mathcal{W}}\frac{1}{n}\sum_{i=1}^n\sigma_if_L(\bx; \bP_{\le k}\text{vec}(\bW))\right]\\
	&= \frac{1}{n}\EE_\sigma\left[\sup_{\bW \in \mathcal{W}}\sum_{i=1}^n\sigma_i\varphi(\bx)^T\bP_{\le k}\text{vec}(\bW)\right]\\
	&= \frac{1}{n}\EE_\sigma\left[\sup_{\bW \in \mathcal{W}}\sum_{i=1}^n\sigma_i\varphi(\bx)^T\bP_{\le k}(\bSigma_{\le n_k}^\dagger)^{1/2} \bSigma_{\le n_k}^{1/2} \text{vec}(\bW)\right]\\
	&\le \frac{1}{n}\sup_{\bW \in \mathcal{W}}\|\text{vec}(\bW)\|_{\bSigma_{\le n_k}} \EE_\sigma \left\|\sum_{i=1}^n \sigma_i\bP_{\le k}\varphi(\bx)\right\|_{\bSigma_{\le n_k}^\dagger},
\end{align*}
where the last step follows by Cauchy-Schwarz. By definition,
\[
	\|\text{vec}(\bW)\|_{\bSigma_{\le n_k}} = \text{vec}(\bW)^T\bSigma_{\le n_k}\text{vec}(\bW) = \|f_L(\bx; \bP_{\le k}\bW)\|_{L^2}^2.
\]
Also, we can bound
\begin{align*}
\EE_\sigma \left\|\sum_{i=1}^n \sigma_i\bP_{\le k}\varphi(\bx)\right\|_{\bSigma_{\le n_k}^\dagger} &\le \left(\EE_\sigma \left\|\sum_{i=1}^n \sigma_i\bP_{\le k}\varphi(\bx)\right\|^2_{\bSigma_{\le n_k}^\dagger}\right)^{1/2}\\
&= (n\EE_n[\varphi(\bx)^T\bP_{\le k}\bSigma_{\le n_k}^\dagger \bP_{\le k}\varphi(\bx)])^{1/2}\\
&= \left(n\text{Tr}\left(\EE_n\left[(\bSigma_{\le n_k}^\dagger)^{1/2}\bP_{\le k}\varphi(\bx)\varphi(\bx)^T\bP_{\le k}(\bSigma_{\le n_k}^\dagger)^{1/2}\right]\right)\right)^{1/2}
\end{align*}
By Lemma~\ref{lemma: NTK features concentrate}, with probability $1 - d^{-9}$ we have
\begin{align*}
\left\|\EE_n\left[(\bSigma_{\le n_k}^\dagger)^{1/2}\bP_{\le k}\varphi(\bx)\varphi(\bx)^T\bP_{\le k}(\bSigma_{\le n_k}^\dagger)^{1/2}\right] - \bP_{\le k} \right\|_{op} \le \frac12.
\end{align*}
Furthermore, $\EE_n\left[(\bSigma_{\le n_k}^\dagger)^{1/2}\bP_{\le k}\varphi(\bx)\varphi(\bx)^T\bP_{\le k}(\bSigma_{\le n_k}^\dagger)^{1/2}\right]$ and $\bP_{\le k}$ have the same span, which is dimension $n_k = \Theta(d^k)$. Therefore
\begin{align*}
\text{Tr}\left(\EE_n\left[(\bSigma_{\le n_k}^\dagger)^{1/2}\bP_{\le k}\varphi(\bx)\varphi(\bx)^T\bP_{\le k}(\bSigma_{\le n_k}^\dagger)^{1/2}\right]\right) \le \Theta(d^k).
\end{align*}
Altogether, we get the bound
\begin{align*}
\EE_\sigma \left\|\sum_{i=1}^n \sigma_i\bP_{\le k}\varphi(\bx)\right\|_{\bSigma_{\le n_k}^\dagger} &\lesssim \sqrt{nd^k},
\end{align*}
so

\begin{align*}
	\cR_n(\cF^L_k(\cW)) \lesssim \sqrt{\frac{d^k}{n}}\cdot\sup_{\bW \in \cW}\|f_L(\bx; \bP_{\le k}\bW)\|_{L^2}^2.
\end{align*}
\end{proof}

\section{Proof of Theorem~\ref{thm: main thm}}\label{app: main theorem proof}
\begin{proof}
Choose $m \gtrsim n^4d^{\frac{26(k + 1)}{3}}\varepsilon^{-22/3}.$ Set the regularization parameters as $\lambda_1 = m^{1/2}d^{-\frac{k-1}{2}}\varepsilon_{min}, \lambda_2 = \varepsilon_{min}, \lambda_3 = m^{1/2}d^{-\frac{k-1}{2}}\varepsilon_{min}\eta, \lambda_4 = d^{-2(k-1)}\varepsilon_{min}$. Also, set $r = n_k$, $\eta = c\lambda_4^{-1}m^{-3}$ for small constant $c$, and $\sigma = \Tilde O(m^{-4})$.

For $\nu = m^{-1/2}, \gamma = m^{-3/4}$, $\tilde \varepsilon = \min(\nu, \gamma^2m^{-5/2}) = m^{-4}$. Therefore by Theorem~\ref{thm: optimization formal}, with probability $1 - d^{-8}$ we reach a $(\nu, \gamma)$-SOSP within $\mathscr{T} = \tilde O(m^3\varepsilon^{-2}) = \tilde O(m^{11})$ timesteps. Call this point $\hat \bW$.

By Corollary~\ref{cor: main landscape cor}, $L_\lambda(\hat \bW) \le C\varepsilon_{min}$. Finally, by Theorem~\ref{thm: main generalization}, with probability $1-d^{-8}$, $L(\hat \bW) \le C\varepsilon_{min}$. Setting $\varepsilon_{min} = \varepsilon/C$, we get that $L(\hat \bW) \le \epsilon$ with probability $1 - d^{-7}$, as desired.
\end{proof}

\section{Additional Experiments}\label{app:extra_experiments}
\subsection{Additional Simulations}

\begin{figure}[h!]
\centering
\includegraphics[width=\textwidth]{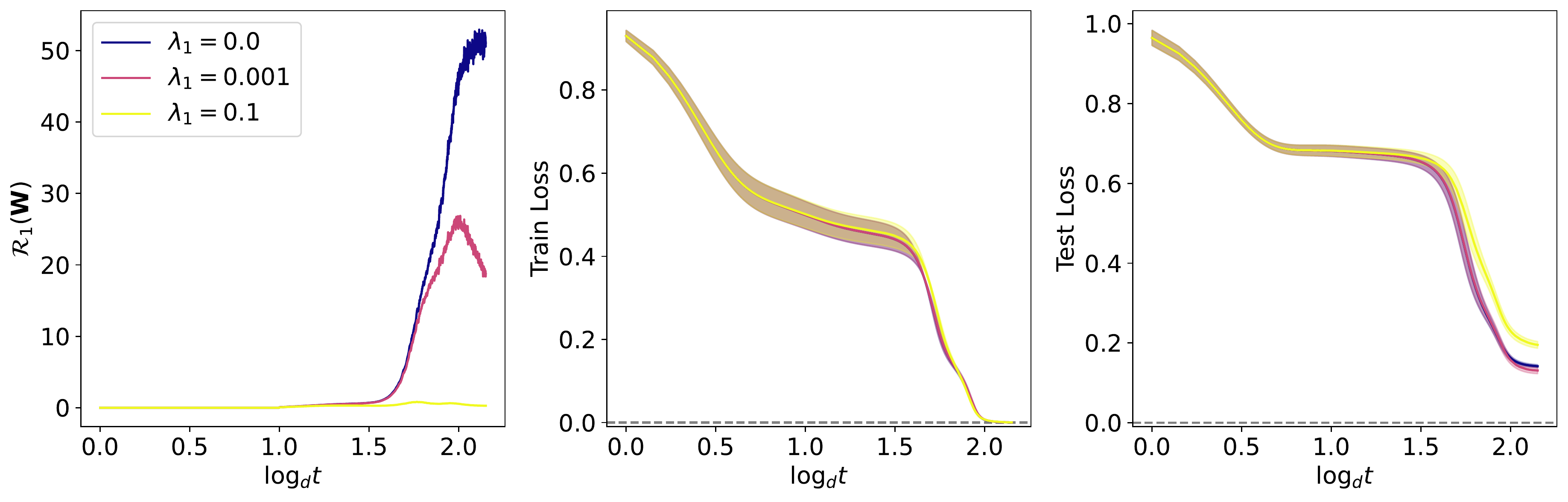}
\caption{We train $f_L + f_Q$ with varying $\lambda_1$, while keeping $\lambda_3$ fixed}
\label{fig:r1}
\end{figure}

In Figure~\ref{fig:r1}, we conduct the same experiment as in Figure~\ref{fig: main figure}, while additionally using the $\cR_1$ regularizer. We fix $\lambda_3 = 0.01$ while varying $\lambda_1$. Since we cannot compute $\cR_1$ exactly, we use an unbiased estimate at every timestep by sampling a new set of $\bx$'s and computing $\EE[f_L(\bx; \bP_{>k}\bW)]$ on this set. In the leftmost pane, we plot a moving average of our estimate of $\cR_1$.

First, we observe that even when $\lambda_1 = 0$, the regularizer $\cR_1$ is an order of magnitude smaller than $\cR_3$ was when we set $\lambda_3 = 0$ (50 versus 500). Furthermore, in the rightmost pane, we see that the test loss of the model is small regardless of which value of $\lambda_1$ was chosen. This provides additional evidence that $\cR_1$ is kept small throughout the training process.

\subsection{CIFAR10 Experiments}

To demonstrate the significance of our approach on ``realistic" datasets/models, we consider experiments with CNNs on CIFAR10.

\cite{bai2020taylorized} showed that, in practice, training the second-order Taylor expansion of the network tracks the true gradient descent dynamics far better than the network's linearization does. This is further demonstrated in Figure~\ref{fig:cifar10}. Here, we train a 4-layer CNN with width 512, $\mathrm{ReLU}$ activation, average pooling between each layer, and the standard PyTorch initialization, on the cats vs. horses CIFAR10 classification task. We train via SGD with batch size 128. For both the train loss and test loss, the dynamics from training $f_L + f_Q$ tracks the true network dynamics better than just the linearization $f_L$. Furthermore, $f_L + f_Q$ acheives a lower test loss than $f_L$ (the true network beats both Taylorizations).

\begin{figure}[h!]
\centering
\includegraphics[width=0.6\textwidth]{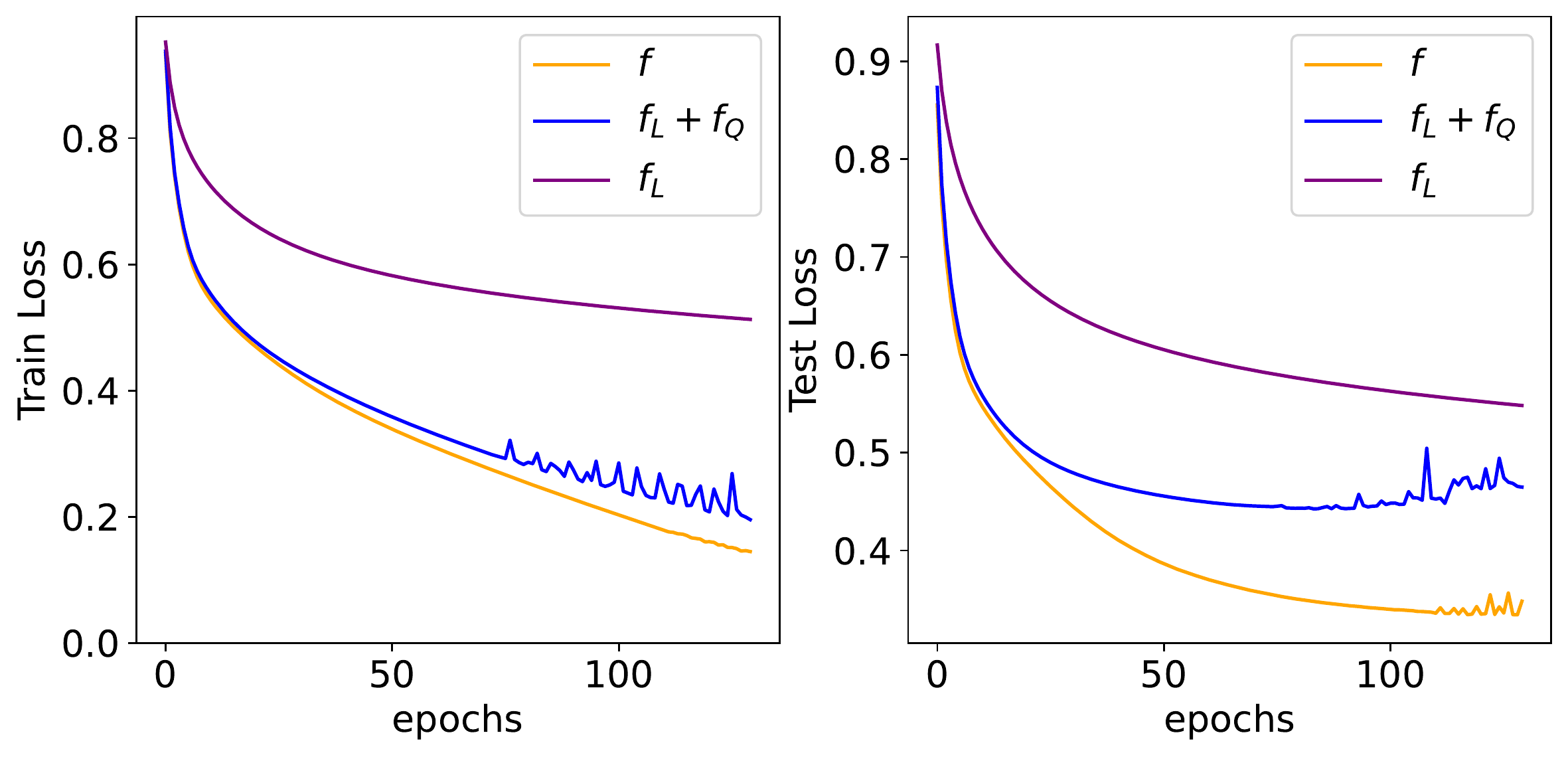}
\caption{$f_L + f_Q$ tracks the true network dynamics far better than just $f_L$.}
\label{fig:cifar10}
\end{figure}

In Table \ref{tab:table1}, we additionally measure the test loss of the linear and quadratic terms \emph{after 100 epochs of training the full model}. We observe that both the linear and quadratic term have a nontrivial loss ($<1$), and thus learned a nonzero component of the signal. This provides evidence that in real neural networks, both the linear and quadratic terms learn a nontrivial component of the signal. 

\begin{table}[h!]
  \begin{center}
    \begin{tabular}{c|c|c|c} 
    \hline
      Model & $f$ & $f_L + f_Q$ & $f_L$\\
      \hline
      Test Loss & 0.340 & 0.699 & 0.887\\
      \hline
      Test Accuracy & 90.8\% & 76.7\% & 60.2\%\\
    \end{tabular}
    \vspace{1cm}
    \caption{The test loss and accuracy of $f$, $f_L$, and $f_L + f_Q$ evaluated on the iterate obtained after 100 epochs of training using the model $f$.}
    \label{tab:table1}
  \end{center}
\end{table}

\subsection{Standard MLP Experiments}

In Figure~\ref{fig:real}, we demonstrated that ``standard'' neural networks can effectively learn low-degree dense and high-degree sparse polynomials. We trained a 2-layer neural network with standard PyTorch initialization and width 100 to learn the target function $f^*(\bx) = \bx^T\bA\bx + h_3(\beta^T\bx)$, where $\bA$ is a high-rank matrix chosen so that $\bx \rightarrow \bx^T\bA\bx$ has an $L^2$ norm of 1. Here, $h_3$ is the 3rd Hermite polynomial, and thus $h_3(\beta^T\bx)$ is a sparse cubic only depending on the random direction $\beta$ (the Hermite polynomial is chosen for this task so that it is orthogonal to the quadratic term, making it the ``hardest'' low-rank cubic to learn).

For varying values of dimension $d$ from 10 to 100 and number of samples $n$, we train our network via vanilla gradient descent with fixed learning rate $0.05$. The initialization, small width, fixed learning rate, and lack of regularization are designed to mimic a standard deep learning setup. For each value of $d$, we compute the minimum $n$ required such that the test loss is $< 0.1$ (note that the test loss of the zero predictor is $1.0$). Figure~\ref{fig:real} is a $\log-\log$ plot of $d$ versus this optimal $n$.


In Figure~\ref{fig:real}, we observe that the number of samples needed to obtain $0.1$ test loss roughly scales with $d^2$. We convincingly see that much fewer than $d^3$ samples (the red dashed line) are needed. The NTK, on the otherhand, requires $\Omega(d^3)$ samples to learn any cubic function. The minimax sample complexity to learn arbitrary quadratics is $\Theta(d^2)$, and therefore this experiment shows that standard neural networks learn ``dense qudratic plus sparse cubic" functions with optimal sample complexity. This provides further evidence that the low-degree plus sparse task is worthy of theoretical study.

\textbf{Experimental Details.} All experiments were run on an NVIDIA RTX A6000 GPU. We use the JAX framework~\cite{jax2018github} along with the Neural Tangents API~\cite{neuraltangents2020}. Code for all experiments can be found at \url{https://github.com/eshnich/escape_NTK}.

\end{document}